\documentclass[11pt,reqno]{proc-l}
\usepackage[letterpaper,margin=1in]{geometry} % 1in margins on all sides
\usepackage{mathptmx,tikz} 
\usetikzlibrary{arrows.meta,positioning,fit,backgrounds,calc,shapes.geometric,shadings,shadows,shadows.blur}
\usepackage{sansmath}

\newtheorem{theorem}{Theorem}[section]
\newtheorem{lemma}[theorem]{Lemma}
\newtheorem{corollary}[theorem]{Corollary}
\theoremstyle{definition}
\newtheorem{definition}[theorem]{Definition}

\theoremstyle{remark}
\newtheorem{remark}[theorem]{Remark}
\numberwithin{equation}{section}
\allowdisplaybreaks[1]

\usepackage{amsmath,amsthm,amssymb,amsxtra,mathtools,bm}
\usepackage{array}
\usepackage{graphicx}
\usepackage{algorithm}
\usepackage{algpseudocode}
\usepackage{physics}
\usepackage{MnSymbol,wasysym}
\usepackage{tikzsymbols}
\usepackage{array,tabularx,multirow}
\usepackage{makecell}
\usepackage{array}
\usepackage{paralist}
\usepackage{parskip}
\usepackage{relsize}
\usepackage{booktabs}

\usepackage{array}
\newcolumntype{P}[1]{>{\centering\arraybackslash}p{#1}}

\usepackage{verbatim}
\usepackage{enumerate}
\usepackage{parskip}
\usepackage{subfigure}

\usepackage[many]{tcolorbox}
\usetikzlibrary{decorations}

\usepackage{color,soul}
\definecolor{clemson-orange}{RGB}{234,106,32}
\definecolor{highlight-orange}{RGB}{255,150,150}
\definecolor{chicago-maroon}{RGB}{128,0,0}
\definecolor{cincinnati-red}{RGB}{190,0,0}
\definecolor{soft-cyan}{RGB}{68,85,90}
\definecolor{firebrick}{RGB}{178,34,34}
\definecolor{crimson}{RGB}{220,20,60}
\definecolor{cerrulean}{rgb}{0.165,0.322,0.745}
\definecolor{jaam}{rgb}{0.45,0.0,0.45}

\usepackage{hyperref}
\hypersetup{
  colorlinks   = true,    % Colours links instead of ugly boxes
  urlcolor     = jaam,    % Colour for external hyperlinks
  linkcolor    = firebrick,    % Colour of internal links
  citecolor    = blue      % Colour of citations
}

\usepackage{parskip}

\newtheorem{assumption}{\bf Assumption}
\newif\ifsolutions \solutionstrue

\def\final{0}
\newcommand{\reviewer}[3]{
  \expandafter\newcommand\csname #1\endcsname[1]{
    \ifthenelse{\equal{\final}{1}} {
      \textcolor{#3}{}
    } {
      \textcolor{#3}{\begin{center} \textbf{#2} ##1 \end{center}}
    }
  }
}

\reviewer{anirbit}{}{jaam}

\newcommand{\Ldata}{\mathcal{L}}

\DeclareMathOperator{\vecmap}{vec}
\newcommand{\mHess}{\mathbf{G}}
\newcommand{\mzero}{\bm{0}}

\def\1{\bm{1}}

\newcommand{\E}{\mathbb{E}}

\newcommand{\Var}{\mathrm{Var}}

\def\vzero{{\bm{0}}}

\def\va{{\bm{a}}}
\def\vb{{\bm{b}}}

\def\ve{{\bm{e}}}

\def\vg{{\bm{g}}}

\def\vp{{\bm{p}}}

\def\vu{{\bm{u}}}
\def\vv{{\bm{v}}}

\def\vx{{\bm{x}}}

\def\vz{{\bm{z}}}
\def\vT{{\bm{T}}}

\def\mA{{\bm{A}}}
\def\mB{{\bm{B}}}

\def\mE{{\bm{E}}}

\def\mH{{\bm{H}}}
\def\mI{{\bm{I}}}

\def\mK{{\bm{K}}}

\def\mM{{\bm{M}}}

\def\mP{{\bm{P}}}
\def\mQ{{\bm{Q}}}
\def\mR{{\bm{R}}}
\def\mS{{\bm{S}}}
\def\mT{{\bm{T}}}
\def\mU{{\bm{U}}}
\def\mV{{\bm{V}}}
\def\mW{{\bm{W}}}
\def\mX{{\bm{X}}}
\def\mY{{\bm{Y}}}

\DeclareMathAlphabet{\mathsfit}{\encodingdefault}{\sfdefault}{m}{sl}
\SetMathAlphabet{\mathsfit}{bold}{\encodingdefault}{\sfdefault}{bx}{n}

\def\gL{{\mathcal{L}}}

\newcommand{\bb}{\mathbb}

\newcommand{\N}{{\bb N}}

\newcommand{\R}{\bb R}

\newcommand{\Poincare}{Poincar\'e\xspace}

\begin{document}
%\title[short text for running head]{}
\title[Provable Stochastic Training of Multi-Headed Attention and for L{\relsize{-2}O}RA]{Convergent Stochastic Training of Multi-Headed Attention \\ and Understanding L{\relsize{-2}O}RA}


%    author one information

\author{Zhengkai Sun$^\dagger$}
\address{Department of Physics. The University of Manchester}
\curraddr{}
\email{zhengkai.sun@student.manchester.ac.uk}
\thanks{}

\author{Dibyakanti Kumar$^\dagger$}
\address{Department of Computer Science. The University of Manchester}
\curraddr{}
\email{dibyakanti.kumar@manchester.ac.uk}
\thanks{$^\dagger$Equal Contribution}

\author{Alejandro F Frangi}
\address{Department of Computer Science. The University of Manchester}
\curraddr{}
\email{alejandro.frangi@manchester.ac.uk}
\thanks{}

\author{Anirbit Mukherjee$^\star$}
\address{Department of Computer Science. The University of Manchester}
\curraddr{}
\email{anirbit.mukherjee@manchester.ac.uk}
\thanks{$^\star$ Corresponding Author}

\author{Mingfei Sun}
\address{Department of Computer Science. The University of Manchester}
\curraddr{}
\email{mingfei.sun@manchester.ac.uk}
% \thanks{}
%    \subjclass is required.
%\subjclass[2010]{Primary }

\date{}

\dedicatory{}

%    "Communicated by" -- provide editor's name; required.
%\commby{}

\begin{abstract}
Transformers have revolutionized machine learning and deploying attention layers in the model is increasingly standard across a myriad of applications. Further, for large models, it is common to implement Low Rank Adaptation (LoRA), whereby a factorized parameterization of them is trained, to achieve a surprisingly beneficial accuracy-size trade-off.  In this work, via a unified framework we rigorously establish trainability of such models under stochastic methods. We prove that for a class of mild regularizations, the empirical regression loss on a attention layer and LoRA on a shallow neural net, both induce \Poincare inequality for the corresponding Gibbs' measure. Crucially, we show that the \Poincare constant is free of the data dimension for LoRA and for multi-head attention it is free of the head dimensions, Then it follows, via invoking recent results, that a certain SDE, which mimics the SGD, minimizes the corresponding losses. In both the cases, our first-of-its-kind results of trainability on attention and nets, do not use any assumptions on the data or the model size. % of the architecture.
\end{abstract}

\maketitle
% \tableofcontents

\section{Introduction}

The remarkable empirical success of attention mechanisms have fundamentally reshaped modern machine learning, most prominently through the transformer architecture -- which is the backbone of Large Language Models (LLMs) \cite{radford2018improving}. By allowing models to dynamically weight interactions between ``tokens''/data fragments, attention layers enable modeling of complicated distributions. Central to the attention mechanism are the query and key matrices, $\mW_Q$ and $\mW_K$ which occur in the model as the product $\mW_Q\mW_K$ and one such pair are trained in each ``attention head'' of which there are many in each of the many attention layers in any commonly used transformer. Despite their central role in the practice of modern AI, our theoretical understanding remains limited of how these matrices evolve during successful training.

While modern attention mechanisms are now the standard, they originally emerged to address the bottleneck of fixed-length representations in early sequence modeling. Early sequence-to-sequence models were limited by fixed-length representations. \cite{bahdanau2016nmtjointly} introduced attention to dynamically aggregate encoder states, allowing the decoder to focus on relevant input positions. Building on this, \cite{luong2015effective} proposed alternative attention variants, including global attention over all positions and local attention over a subset. 

These developments ultimately led to the transformer architecture \cite{vaswani2017attention}, which removes recurrence entirely and instead models sequence interactions purely through stacked self-attention layers, where given an input $\mX \in \R^{t \times d}$ each layer computes $\text{RowSoftMax}_\beta\left(\frac{\mX \mW_Q \mW_K^\top \mX^\top}{\sqrt{d}}\right)\mX \mW_V$, where $\mW_Q, \mW_K$ and $\mW_V$ are learned weight matrices where the RowSoftMax operator is defined as $[\text{RowSoftMax}_\beta(\mM)]_{ij} \coloneqq \frac{e^{\beta M_{ij}}}{\sum_{k=1}^t e^{\beta M_{ik}}}$.

The transformer architecture \cite{vaswani2017attention} has become the dominant paradigm for sequence modeling, replacing recurrence with stacked self-attention and feed-forward layers. Numerous variants have since been proposed, namely Sparse Transformer \cite{child2019sparse}, Longformer \cite{beltagy2020longformer}, Linformer \cite{wang2020linformer}, Transformer-XL \cite{dai2019transformerxl}, ALBERT \cite{lan2020albert}, and Vision Transformer \cite{dosovitskiy2021vit}, among many others. A foundational reason explaining these successes is given in works like \cite{yun2020are} that have show that transformers are universal approximators of sequence-to-sequence matrix functions.

{\bf The Rising Importance of Doing Regression on Transformers} The success of transformers has led to their widespread adoption in scientific machine learning, where many tasks can be naturally formulated as regression problems over high-dimensional discretizations of function spaces. In particular, applications in fluid dynamics and weather prediction --- ranging from operator learning for PDEs to data-driven forecasting --- often reduce to learning mappings between input and output fields by doing regression using a attention-based architecture, FourCastNet \cite{pathak2022fourcastnet}, GraphCast \cite{lam2023graphcast}, Pangu-Weather \cite{bi2023pangu}, Poseidon \cite{herde2024poseidon} and GenCFD \cite{molinaro2024gencfd}. This perspective motivates studying attention not only as a representation mechanism, but as a regression operator whose properties govern generalization and efficiency in continuous domains.

% \note{Location for Q1?}

\begin{center}
    \textbf{Q1: }\emph{With no assumptions on data or architecture, can a attention layer be trained by a stochastic algorithm?}
\end{center}

On the other hand, the rise of large-scale pretraining has motivated parameter-efficient fine-tuning methods such as Low-Rank Adaptation (LoRA), which constrains updates to lie in low-dimensional subspaces while preserving the pretrained backbone. More precisely, the idea is to freeze the pre-trained weight matrix $\mW_{\rm pre} \in \mathbb{R}^{d_{\text{out}} \times d_{\text{in}}}$ and only train a small update. This update is represented as the product of two small matrices, $\mA$ and $\mB$. The modified weights $\mW'$ are parameterized as, $\mW' = \mW_{\rm pre} + \frac{\alpha}{r} \mB \mA$
where $\mA \in \mathbb{R}^{r \times d_{\text{in}}}$ and $\mB \in \mathbb{R}^{d_{\text{out}} \times r}$ are the trainable factors and $\alpha$ being a scaling constant. By training only $\mA$ and $\mB$, we significantly reduce the number of parameters to update. 

As showed in \cite{shuttleworth2025lora}, LoRA works best when applied to all weight matrices and small-to-medium datasets, and its optimal learning rate is largely independent of rank due to the $\frac{1}{r}$ scaling. Overall, LoRA performs similar to full-finetuning in the “low-regret regime” making it a parameter-efficient alternative for post-training adaptation. However, the product $\mB \mA$ introduces a specific scaling problem : we can multiply $\mA$ by a constant and divide $\mB$ by the same constant without changing the final output. This redundancy makes the training landscape ``flat" in certain directions, which is seemingly a natural obstacle for gradient based algorithms to succeed. And yet LoRA has proven strikingly effective in practice, dramatically reducing memory and computational overhead while retaining performance. Thus we posit that the  optimization dynamics of neural models under LoRA have remained unclear from a theoretical standpoint.

\begin{center}
    \textbf{Q2: }\emph{With no assumptions on data or width, can a implementation of LoRA on a net be trained by a stochastic algorithm?}
\end{center}

In  this work, {\em firstly} we make progress towards uncovering hitherto unknown mathematical properties about the attention map and thus uncover a first-of-its-kind provably convergent training mechanism for the query and key matrices. {\em Secondly,} this work also initiates a theoretical study of training dynamics for standard neural networks under LoRA parameterization --- which evidently shares a mathematical similarity of training a factorized weight parameterization as while training the key and query matrices in the attention, as introduced above.

Whether for training the key and query matrices of an attention head or for training standard nets under LoRA, we consider training through stochastic differential equations (SDEs), capturing the continuous-time limit of stochastic gradient methods commonly used in practice. {\em In both cases --- for any number of parameters and for any data --- we prove convergence of a risk function for the models under a SDE flow.}

% \note{need here text motivating the SGD training figure - emulate Weijie for it}

In Figure~\ref{fig:stats1}, we present motivational experimental evidence, that when training  our multi-head attention regressor for the task of time evolving solutions of the Darcy flow, via SGD, scaling down the step-length at scheduled epochs (dotted vertical lines) consistently improves the training loss across all three regularization regimes (unregularized, Frobenius, log-amplified Frobenius) --- at the cost of additional run-time, exactly the trade-off characterized by our main theorem..

% \note{the paragraph above is not good! so much of classical deep-learning reference is not needed and reference to Weijei is also not needed - we are just trying to say a very simple thing that there is experimental evidence that for regression tasks on attention, scaling down the step-length proportionately improves the training loss at the cost of run-time, exactly the characteristics as what the final theorem we show has - which we foreshadow here -- this should be barely 2-3 lines!}

\begin{figure}[h!]
    \centering
    \includegraphics[width=0.8\linewidth]{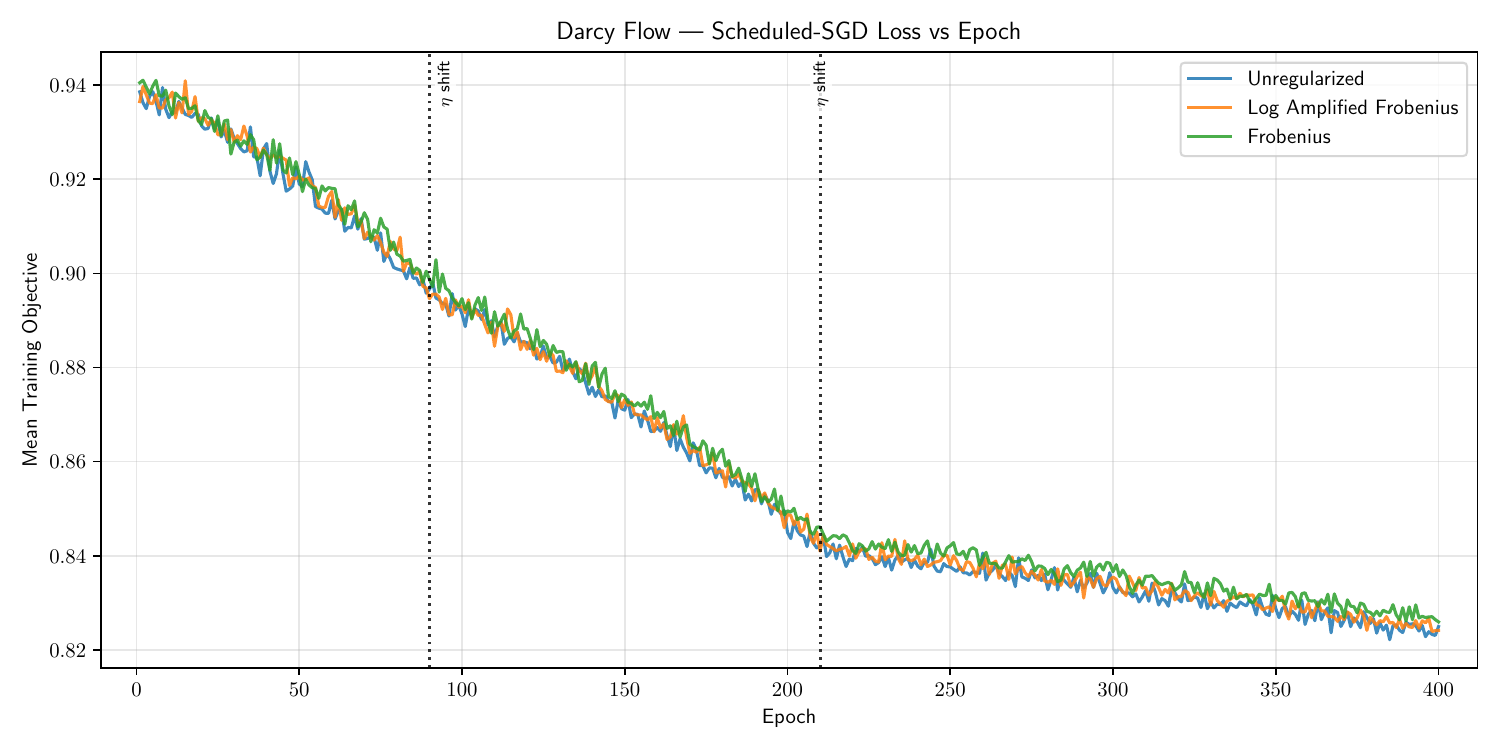}
    \caption{The figure shows how the  regression training error of a multi-head attention model (with 4 heads) facing the 2-D Darcy Flow dataset progresses when training algorithm is SGD with a mini-batch size of 32. The learning rate is scheduled as: $s=0.006$ for the first 90 epochs, $s=0.003$ for the subsequent 120 epochs, and $s=0.0005$ for the final 190 epochs. We see that the training error continues to decrease after each learning-rate reduction. At the same time, smaller learning rates lead to slower optimization dynamics, requiring more epochs for SGD to approach a plateau. In Section \ref{sec:exp} we return to this setup for further study.}
    \label{fig:stats1}
\end{figure}

\subsection{Summary of Results}

In this work, we establish that for both attention and depth-2 neural networks under LoRA, a mildly regularized regression loss (say $\tilde{V}$) is a Villani function (which will be precisely defined in the next section) --- which in turn implies that the corresponding Gibbs' measure ($\mu \sim e^{-\gamma \tilde{V}}$) satisfies the Poincare inequality i.e. for any $h \in H^1(\mu_s)$, $\Var_{\mu}(h) \leq C_{PI} \E_{\mu}[\norm{\grad h}^2]$. Then invoking recent results on isoperimetry based Stochastic Differential Equation (SDE) convergence \cite{shi2023villani} we can establish convergence in both settings for the following continuous-time stochastic gradient dynamics for the weights $\vT$, given by the SDE,  
\begin{align}
    \dd{\vT_t} = -\nabla \tilde{V}(\vT_t)\dd{t} + \sqrt{s} \dd{B_t}, \label{eq:sde}
\end{align}
where $\tilde{V}(\vT)$ denotes the regularized loss function, $s > 0$ is a temperature parameter and $(B_t)_{t \ge 0}$ is the Brownian motion.\footnote{The subscript $t$ denotes continuous time.} We provide an informal restatement of our main results below.

% \note{need to define Poincare constant} 

% \begin{theorem}[{\bf Informal Statement of Provable Learning for Attention-Based Regression}]\label{thm:informalattn}
% Consider a single attention layer with key and query matrices $\mW_K$ and $\mW_Q$, being trained using the $\ell_2$-loss function with either a logarithmically amplified $2-$norm regularization or a super-quadratic polynomial regularization. Then, for any arbitrarily low regularization, for any data and size of architecture, the loss function satisfies the Villani condition.

% As a consequence, for any $\varepsilon > 0$, there exists an appropriate step size $s = k \varepsilon$ for some $k > 0$ and a corresponding $\lambda_\varepsilon > 0$, inversely proportional to the \Poincare constant of the loss function, such that the SDE for $\mT = (\mW_Q, \mW_K)$ converges, in expectation, in ${\gO}\left (\frac{1}{\lambda_\varepsilon} \log \frac{1}{\varepsilon} \right)$ to within $\varepsilon$ of the global minimum of the training loss.
% \end{theorem}

\begin{theorem}[{\bf Informal Statement of Convergence for Attention-Based Regression}]
Consider a multi-head attention model with key and query matrices $\mW^{(j)}_K$ and $\mW^{(j)}_Q$ for $j \in [H]$ where $H$ is the number of heads, being trained using the $\ell_2$-loss with a certain class of radial regularizers. Then, for any number of heads, tokens, and query-key factor dimension, the loss function satisfies the Villani conditions, and thus the corresponding Gibbs measure $\mu_s$ satisfies the \Poincare inequality for some constant $C_{PI}$. And further, for a log-amplified-Frobenius norm regularization, the constant $C_{PI}$ has an explicit upperbound which is free of the head dimensions. 

%, but scales with the number of heads and the sequence length.
% Crucially, we show that the \Poincare constant is dimension-free for LoRA and for multi-head attention it is free of the embedding and head dimensions, but scales with the number of heads and the sequence length

\qquad Consequentially, for any $\varepsilon>0$, there exists a sufficiently small step size $s = O(\epsilon)$ such that the associated SDE for $\mT$ being the vector of weights for all the key and the query matrices, converges to within $\varepsilon$ of the global minimum of the regularized attention loss, in expectation, in\\ $O\left( \frac{1}{\epsilon} e^{\frac{1}{\epsilon}} \cdot \max\left( \log(\frac{2}{\epsilon}), \log(\frac{C(\epsilon)}{\epsilon}), \log(\frac{\norm{p_0 - \mu_\epsilon}_{L^2((\mu_\epsilon)^{-1})}}{\epsilon}) \right)\right)$ , where $p_0$ is the initial distribution of $\mT$ and $C(\epsilon)$ is some function of $\epsilon$.
\end{theorem}

% \note{no need to define T explicitly in the informal theorems}

The above informal restatement combines Theorem~\ref{thm:mh-general-profile-villani} with the relevant part of Theorem~\ref{thm:bpc-attention-poincare-bound}.

% \begin{theorem}[{\bf Informal Statement of Provable Learning for Depth-2 Neural Net Based Regression under LoRA}]\label{thm:informalnn}
% Consider a depth-2 neural network with weight matrix $\mW$ factorized as $\mU\mV$, being trained using the $\ell_2$-loss function with either a logarithmically amplified $2-$norm regularization or a super-quadratic polynomial regularization. Then, for any arbitrarily low regularization, for any data and size of architecture, the loss function satisfies the Villani condition.

% As a consequence, for any $\varepsilon > 0$, there exists an appropriate step size $s$ and a corresponding $\lambda_\varepsilon > 0$and a corresponding $\lambda_\varepsilon > 0$, inversely proportional to the \Poincare constant of the loss function, such that the SDE for $\mT = (\mU, \mV)$ converges, in expectation, in ${\gO}\left (\frac{1}{\lambda_\varepsilon} \log \frac{1}{\varepsilon} \right)$ to within $\varepsilon$ of the global minimum of the training loss.
% \end{theorem}

\begin{theorem}[{\bf Informal Statement of Convergence for Depth-2 Neural Net Based Regression under LoRA}]
Consider a depth-2 neural network trained under the LoRA parameterization and $\ell_2$-loss with a certain class of radial regularizers. Then, for any fixed base matrix, scaling parameter, and an arbitrarily small regularization, the loss function satisfies the Villani conditions, and thus the corresponding Gibbs measure $\mu_s$ satisfies the \Poincare inequality for some constant $C_{PI}$. And further, for a log-amplified-Frobenius norm regularization, the constant $C_{PI}$ has an explicit upperbound which is free of the data dimension.

\qquad Consequentially, for any $\varepsilon>0$, there exists a sufficiently small step size $s = O(\epsilon)$ such that the associated SDE for $\mT=(\mU,\mV)$ converges to within $\varepsilon$ of the global minimum of the regularized training loss in expectation, in $O\left( \frac{1}{\epsilon} e^{\frac{1}{\epsilon}} \cdot \max\left( \log(\frac{2}{\epsilon}), \log(\frac{C(\epsilon)}{\epsilon}), \log(\frac{\norm{p_0 - \mu_\epsilon}_{L^2((\mu_\epsilon)^{-1})}}{\epsilon}) \right)\right)$, where $p_0$ is the initial distribution of $\mT$ and $C(\epsilon)$ is some function of $\epsilon$.
\end{theorem}

The above informal restatement combines Theorems~\ref{thm:lora-general-profile-villani} with the relevant part of Theorem~\ref{cor:bpc-lora-poincare-bound}.

\subsection{Literature Review}

A theoretical analysis of transformers in the infinite-width limit by deriving their Neural Network Gaussian Process (NNGP) and Neural Tangent Kernel (NTK) equivalents was provided in \cite{hron2020inf}. They showed that self-attention layers admit well-defined kernel limit and in this regime, gradient descent training of a Transformer is equivalent to kernel regression with the corresponding NTK.

A global convergence framework for transformers by analyzing training dynamics in the mean-field limit was established in \cite{gao2024global}. By treating model width and depth as approaching infinity, the authors demonstrate that discrete gradient descent converges to a Wasserstein gradient flow on the distribution of parameters. Albeit the use of an infinite-width limit it strictly requires the inclusion of a weight decay parameter $\lambda>0$, for the convergence proof to work.

A formal proof that transformers can converge to the functional behavior of near-optimal Reinforcement Learning (RL) algorithms through the optimization of a log-likelihood objective was given by \cite{lin2023transformers}. By viewing the attention mechanism as an iterative optimizer, they demonstrate that supervised pre-training on offline trajectories allows the model to implement algorithms like LinUCB and Thompson Sampling directly. However, these convergence guarantees rely on non-standard architectures, most notably the use of ReLU-based attention to facilitate exact linear algebraic operations. Furthermore, the authors assume model realizability, implying that the Transformer’s capacity must be sufficient to encapsulate the expert's decision-making logic.

A multi-layer transformer trained on $n$-gram data, where each token depends on the preceding $n$ tokens, is analyzed in \cite{chen2024inductionheads}, and it is shown that gradient flow converges to a model exhibiting induction head behavior. Specifically, induction here refers to the phenomenon whereby, if a token at position $i$ matches a previous occurrence at position $j$, the model attends to the token following position $j$ to predict the next token at position $i$. Their results provide a rigorous characterization of how attention layers, feed-forward networks, and normalization interact to learn features from context. This work advances prior studies that focused on linear or single-layer models by handling richer architectures and more realistic data distributions.

%\cite{sitan2025provably} moves beyond the traditional assumption where it is assumed that every coordinate of $X$ is sampled from a Gaussian.

Fundamental algorithmic limits of Multi-Head Attention (MHA) restricted to a discrete Boolean input distribution $X \in \{\pm 1\}^{k \times d}$ was established in \cite{sitan2025provably}. Theorem 1.2 in \cite{sitan2025provably} establishes that, under a non-degeneracy condition on the attention and projection matrices and realizability assumption of the samples, there exists an algorithm that for $m-$headed attention estimates the parameters in $(kd)^{O(m^3)}$ time, using $(kd)^{\Theta(m)}$ samples, and achieves predictions that are $(kd)^{-\Omega(m)}$ close to the true values in expectation. This result identifies the number of heads $m$ as the dominant factor in computational scaling.

\subsubsection{Review of Existing Attempts at Provable Training of LoRA on Neural Networks --- with Weight Regularization}

LoRA was first introduced by \cite{hu2022lora} where it was asserted that the weight updates for task-specific adaptation in attention based models reside in a manifold of low intrinsic dimension \cite{aghajanyan2021intrinsic}. By reparameterizing the update matrix $\Delta \mW$ as the product of two low-rank matrices $\mA$ and $\mB$, \cite{hu2022lora} demonstrated that optimization can converge to high-performance solutions with significantly fewer trainable parameters. Crucially, their initialization strategy --- setting one matrix to zero --- ensures a stable starting point at the pre-trained state, effectively bridging the gap between training efficiency and the convergence stability typically observed in full fine-tuning.

A rigorous analysis of LoRA on neural networks in the generic non-linear regime was given by \cite{kim2025lora}, establishing a formal dichotomy between global convergence and parameter divergence. By characterizing the optimization as governed by a global Restricted Strong Convexity, they prove that the combination of zero-initialization and weight decay induces an implicit bias toward a low-rank global minimum. Crucially, they demonstrate that while the non-linear landscape may harbor spurious local minima, these points are spectrally isolated in high-rank regions and do not intersect with the stable optimization trajectory.

In the Neural Tangent Kernel (NTK) framework LoRA was analyzed by \cite{jang2024lora}, where the neural network's optimization can be treated as a linearized system. In this regime, the training dynamics are governed by a quadratic objective subject to a low-rank structural constraint. Alongside a non-standard regularizer --- derived from the Rademacher complexity of the low-rank bottleneck --- the authors prove that the non-convex BA reparameterization does not introduce spurious local minima. \cite{jang2024lora} further demonstrates, if the rank is above a certain threshold, gradient-based methods like stochastic gradient descent (SGD) converge to a low-rank global minimizer.

{\bf Limitations of LoRA} It was demonstrated by \cite{shuttleworth2025lora} that the perceived equivalence between LoRA and full fine-tuning is an ``illusion" maintained by surface-level metrics. Despite its parameter efficiency, LoRA can lead to catastrophic forgetting, thereby degrading performance in settings that require continual learning.

% \note{https://arxiv.org/abs/2410.21228 --- counter-argument to LoRA} 

% \note{
% {\bf LATER!}
% \paragraph{\bf Review \cite{zeng2024expressive}}  studied the expressive power of LoRA, asking if the $BA$ update is strong enough to mimic any desired target model.}

\subsubsection{Organization}

Section~\ref{sec:mathsetup} introduces the analytic framework, including the underlying conditions, neural architecture, loss functions, and the assumptions required for the subsequent sections. Section~\ref{sec:mainresults} presents the main results of the paper, namely Theorem~\ref{thm:mh-general-profile-villani} and \ref{thm:lora-general-profile-villani} where we prove that multi-head attention and depth-2 neural nets with LoRA regression losses satisfy the Villani conditions for a class of regularizers. Next, we state in Theorems~\ref{thm:bpc-attention-poincare-bound} and \ref{cor:bpc-lora-poincare-bound} the SDE convergence of certain log-amplified-Frobenius norm regularized losses for both cases --- in terms of the dimension independent bounds that we obtain on the corresponding \Poincare constants, from our Lemma~\ref{lem:bpc-general-poincare-bound}. The proof of Theorem~\ref{thm:mh-general-profile-villani}, along with the requisite auxiliary lemmas, is provided in Section~\ref{sec:proofattnreg}. Section~\ref{sec:exp} presents experimental results on solving the 2D Darcy flow problem using the regularized loss functions introduced. 

Appendix~\ref{app:conv_proof} presents the proof of Theorems~\ref{thm:bpc-attention-poincare-bound} \& \ref{cor:bpc-lora-poincare-bound} and Lemma~\ref{lem:bpc-general-poincare-bound} is proved in Appendix~\ref{sec:BPC}. Appendix~\ref{app:multihead-general-profile} has the deferred proofs from Section~\ref{sec:proofattnreg}.
The proof Theorem~\ref{thm:lora-general-profile-villani} is contained in Appendix~\ref{app:lora-villani}.
% \note{lots of broken links above}

%with the corresponding supporting lemmas established in Appendix~\ref{app:subsec:loralemm}
\section{Mathematical Setup}\label{sec:mathsetup}

In this section, we define the analytic conditions, architectures, and loss functions, alongside the core assumptions that underpin our subsequent analysis.

The Villani condition was introduced in \cite{villani2009hypocoercivity} to guarantee that, when a function satisfies this condition, the associated Gibbs measure satisfies the \Poincare inequality. We recall that a distribution $\pi$ is said to satisfy the \Poincare inequality for some constant $C_{PI}$, if for all smooth functions $h : \R^d \to \R$, $\Var_\pi(h) \leq C_{PI} \E_\pi[\norm{\grad h}^2]$. \cite{shi2023villani} leverage the \Poincare inequality induced by the Villani condition to establish convergence results for certain stochastic differential equations (SDEs). We now proceed to formally define the corresponding analytic conditions.

\begin{definition}[{\bf Villani Condition}]\label{def:villani}
A function $f$ is said to satisfy the Villani condition if it satisfies the following conditions,
\begin{inparaenum}
    \item $f \in C^2$,
    \item $\lim_{\|\vx\|\to+\infty} f(\vx) = +\infty$,
    \item ${\displaystyle \int_{\R^d}} e^{-\frac{2 f(\vx)}{s}}\dd{\vx} < \infty ~\forall s >0$, and
    \item $\forall s>0$,   $\frac{ \|\nabla f(\vx)\|^2}{s} - \Delta f(\vx) \rightarrow \infty$ as $\|\vx\| \to \infty$.
\end{inparaenum}

\end{definition}

In recent works \cite{gopalani2024globallog, gopalani2025global}, it was shown that the Villani condition holds for depth-2 neural networks of arbitrary width and for both squared and logistic losses. Building on this, \cite{kumar2025lmclearn} proved that Langevin Monte Carlo consequently achieves population risk minimization. {In contrast to \cite{gopalani2024globallog, gopalani2025global, kumar2025lmclearn}, our convergence results here hold for any value of the regularization parameter.}

Next we formally define the attention model and its regression loss that we choose to train. 

\begin{definition}[{\bf Multi-Head Key-Query Attention Model}]\label{def:mh-kq-attn}
    Let $H \in \N$ denote the number of attention heads. For each head $j \in \{1,\dots,H\}$, let $\mW_Q^{(j)}, \mW_K^{(j)} \in \R^{d \times r}$.
    Then, for the training data $\{(\mX_i,\mY_i) \in \R^{t \times d} \times \R^{t \times d} \}_{i=1}^n$, the attention model is defined by,
    \begin{align}
    \mM_i^{(j)}(\mW_Q, \mW_K)
    &\coloneqq
    \frac{1}{\sqrt d}
    \mX_i\mW_Q^{(j)}(\mW_K^{(j)})^\top\mX_i^\top,
    \label{eq:mh-score-matrix} \in \R^{t \times t}\\
    \mS_i^{(j)}(\mW_Q, \mW_K)
    &\coloneqq
    \operatorname{RowSoftMax}_\beta\!\left(\mM_i^{(j)}\right) \in \R^{t \times t},
    \label{eq:mh-softmax-matrix}\\
    \R^{t \times d} \ni \mX_i \mapsto \widehat{\mY}_i(\mW_Q, \mW_K)
    &\coloneqq
    \left(\sum_{j=1}^H\mS_i^{(j)}\right)\mX_i\mW_v \in \R^{t \times d}.
    \label{eq:mh-output}
    \end{align}
    where $\mW_v \in \R^{t \times t}$ is the fixed value matrix with $\norm{\mW_v}_F \leq B_w$.
\end{definition}

\begin{definition}[{\bf Multi-Head Attention Loss with General Head-Wise Radial Regularization}]\label{def:mh-general-profile-potential}
Let $\widehat{\mY}_i$ be as defined in Definition~\ref{def:mh-kq-attn}. Then, for the training data $\{(\mX_i,\mY_i) \in \R^{t \times d} \times \R^{t \times d} \}_{i=1}^n$, the non-regularized loss $\widehat R_{A,H}$ is defined as,
\begin{align}
\ell_i(\mW_Q, \mW_K)
&\coloneqq
\frac12\lVert\mY_i-\widehat{\mY}_i(\mW_Q, \mW_K)\rVert_F^2,
\qquad
\widehat R_{A,H}(\mW_Q, \mW_K)
\coloneqq
\frac1n\sum_{i=1}^n\ell_i(\mW_Q, \mW_K).
\label{eq:mh-data-loss}
\end{align}
For a radial potential $\phi:[0,\infty)\rightarrow\mathbb R$ and $\lambda>0$, we define
\begin{equation}
\widetilde V_{H,\mathrm{ATT}}(\mW_Q, \mW_K)
\coloneqq
\widehat R_{A,H}(\mW_Q, \mW_K)
+
R_{H}(\mW_Q, \mW_K),
\quad
R_{H}(\mW_Q, \mW_K)
\coloneqq
\frac{\lambda}{2H}\sum_{j=1}^H\phi\!\left(\Bigl(\norm{\mW_Q^{(j)}}_F^2 + \norm{\mW_K^{(j)}}_F^2\Bigr)^{1/2}\right),
\label{eq:mh-general-potential}
\end{equation}
where $\phi$ is an arbitrary radial regularization potential --- specific choices of the form $\phi(r)=r^2h(r)$, including the logarithmic and polynomial regularizers are considered below.
\end{definition}

For further intuition, we provide Figures~\ref{fig:head} and \ref{fig:mhead} in Appendix~\ref{sec:diag}, illustrating the single-head and multi-head attention mechanisms of above, respectively.

% \note{refer to Appendix \ref{sec:diag} here}

A notable application of this framework is the Vision Transformer (ViT) \cite{dosovitskiy2021vit}. In ViT, the input image is first divided into non-overlapping patches, commonly of size \(16 \times 16\) pixels. Each patch is flattened and projected into a \(d\)-dimensional embedding, producing a sequence of \(t\) token embeddings that serve as input to the attention layers described above. This allows the model to capture long-range dependencies across the image while leveraging the same attention mechanism as in general Transformer architectures. The row-wise softmax scaling parameter \(\beta\) is typically set to \(1\).

For the case of shallow neural networks, we assume a training in a space of weight matrices with a rank bound i.e we assume the trainable weight to be factorizable as, $\mW = \mU\mV^\top$, with $\mU \in \R^{p \times r}$ and $\mV \in \R^{d \times r}$. This factorization implements low-rank adaptation (LoRA) approach for depth-2 nets --- and we recall that this has been previously studied in \cite{jang2024lora, kim2025lora} as a theoretical sandbox for the LoRA technique. As in the attention-based setting, regularization of the factor matrices is introduced to ensure well-behaved potentials in the weight space. Similar to the attention training setup above, we define two types of factor-regularized loss functions as follows,
%: a non-polynomial logarithmic regularization and a polynomial regularization with exponent $2+\epsilon$. The following definitions formalize these regularized potentials and their associated mean square losses.

% \begin{definition}[{\bf Rank-Restricted Mean Square Loss on Shallow Nets with Non-Polynomial Factor-Regularization}] \label{def:v_log}
% We define the potential $\tilde{V}(\vT)$ in the factor space $\R^D$ for $\vT = (\mU, \mV)$ and $D = (p+d)r$ as,  
% \begin{equation}
% \tilde{V}(\vT) \coloneqq \Ldata(\vT) + \frac{\lambda}{2} (\|\mU\|_F^2 + \|\mV\|_F^2)\log \left ( 1 + \|\mU\|_F^2 + \|\mV\|_F^2 \right )
% \end{equation}
% where $\gL(\mT) \coloneqq \frac{1}{n} \sum_{i=1}^n \ell_i(\mU\mV^\top)$ and $\ell_i(\mW) \coloneqq \frac{1}{2}(y_i - \va^\top \sigma(\mW \vx_i))^2$.
% \end{definition}

% \begin{definition}[{\bf Rank-Restricted Mean Square Loss on Shallow Nets with Polynomial Factor-Regularization}] \label{def:v_2+eps}
% We define the potential $\Vcorr(\vT)$ in the factor space $\R^D$ for $\vT = (\mU, \mV)$ and $D = (p+d)r$ as,  
% \begin{equation}
% \Vcorr(\vT) = \Ldata(\vT) + \frac{\lambda}{2} (\|\mU\|_F^{2+\epsilon} + \|\mV\|_F^{2+\epsilon})
% \end{equation}
% where $\gL(\mT) \coloneqq \frac{1}{n} \sum_{i=1}^n \ell_i(\mU\mV^\top)$ and $\ell_i(\mW) \coloneqq \frac{1}{2}(y_i -\va^\top \sigma( \mW \vx_i))^2$.
% \end{definition}

We consider the following affine LoRA parameterization of the trainable weight matrix in the shallow-net setting. 

% The base matrix $\mW_0\in\R^{p\times d}$ and the scaling parameter $\alpha\in\R$ are fixed, whereas the factor matrices $\mU\in\R^{p\times r}$ and $\mV\in\R^{d\times r}$ are trainable.

\begin{definition}[{\bf Shallow-Net LoRA Parameterization}] \label{def:shallow-lora-param}
    The base matrix $\mW_0\in\R^{p\times d}$ and the scaling parameter $\alpha\in\R$ are fixed. Corresponding to trainable factor matrices $\mU\in\R^{p\times r}$ and $\mV\in\R^{d\times r}$ we define
    the LoRA-adapted weight matrix as, 
    \begin{equation} 
        \mW(\mU, \mV) \coloneqq \mW_0+\frac{\alpha}{r}\mU\mV^\top. \label{eq:shallow-lora-weight} 
    \end{equation} 
    For the training data $\{(\vx_i,y_i)\in \R^d \times \R \}_{i=1}^n$, the corresponding shallow-network prediction is 
    \begin{equation} 
        \R^d \ni \vx_i \mapsto \widehat y_i(\mU, \mV) \coloneqq \va^\top\sigma\!\left(\mW(\mU, \mV)\vx_i\right) \in \R, \label{eq:shallow-lora-prediction} 
    \end{equation}
    where $\norm{\va} \le B_a$ and $\norm{\mW_0}_F\le B_0$.
\end{definition}

\begin{definition}[{\bf Mean Square Loss on Shallow Nets Under LoRA with General Radial Factor-Regularization}] \label{def:lora-general-radial-potential} 
Let $\widehat y_i(\mU, \mV)$ be as defined in Definition~\ref{def:shallow-lora-param}. Then the individual and empirical data losses are defined by, 
\begin{equation} 
\ell_i(\mU, \mV) \coloneqq \frac12\left(y_i-\widehat y_i(\mU, \mV)\right)^2, \qquad \Ldata(\mU, \mV) \coloneqq \frac1n\sum_{i=1}^n\ell_i(\mU, \mV). \label{eq:lora-general-data-loss} 
\end{equation} 
For a radial potential $\phi:[0,\infty)\to\R$ and $\lambda>0$, we define 
\begin{equation} 
\widetilde V_{\mathrm{LoRA}}(\mU, \mV) \coloneqq \Ldata(\mU, \mV)+R_{\mathrm{LoRA}}(\mU, \mV), \qquad R_{\mathrm{LoRA}}(\mU, \mV) \coloneqq \frac{\lambda}{2}\phi \left (\left( \norm{\mU}_F^2 + \norm{\mV}_F^2 \right)^{1/2} \right ), \label{eq:lora-general-radial-potential} 
\end{equation}
where $\phi$ is an arbitrary radial regularization potential --- specific choices of the form $\phi(r) = r^2 h(r)$, including the logarithmic and polynomial regularizers are considered below.
\end{definition}

To establish convergence results for the factorized, regularized losses defined above, we impose a few standard assumptions on the network and define certain properties of the training data. 

% These assumptions ensure that the gradients and higher-order derivatives remain controlled and that the training data lies in a bounded domain.

% \note{give a definition of the training data and the wight bounded. (They are not assumption)}

\begin{definition}[{\bf Training Data Bounds}] \label{as2}
The training data is bounded as follows:
\begin{enumerate}
    \item For Attention : Each training example $(\mX_i, \mY_i) \in \mathbb{R}^{t \times d} \times \mathbb{R}^{t \times d}$ satisfies $\|\mX_i\|_F\leq B_x,\norm{\mY_i}_F\leq B_y, i=1,\ldots,n$. 
    \item For Neural Network : Each training example $(\boldsymbol{x}_i, y_i) \in \mathbb{R}^d \times \mathbb{R}$ satisfies $\|\boldsymbol{x}_i\|_2\leq B_x,|y_i|\leq B_y, i=1,\ldots,n$. 
\end{enumerate}
Depending on the architecture, the bounds $B_x$ and $B_y$ are interpreted according to the corresponding definitions given above.
\end{definition}

\begin{assumption}[{\bf Activation Function Bounds}] \label{as1}
We assume that $\sigma$, $\sigma'$ and $\sigma''$ are bounded as $|\sigma(u)|\le B_\sigma$, $|\sigma'(u)|\le L_\sigma$ and $|\sigma''(u)|\le L_\sigma'$ for all $u\in\R$, respectively. 
\end{assumption}

\begin{remark}
We note that the following standard activations satisfy Assumption \ref{as1}:  \begin{inparaenum}
    \item Sigmoid,
    \item Tanh,
    \item $\arctan$,
    \item $\sin$ and $\cos$.
\end{inparaenum} --- the explicit constants $B_\sigma, L_\sigma, L_\sigma'$ given in Table~\ref{tab:actlora} in Appendix~\ref{app:subsec:actlora}. And these bounds hold for all $u \in \R$ so the constants $B_\sigma, L_\sigma, L_\sigma'$ can be taken uniformly over $u \in \R$.
\end{remark}

\begin{assumption}[{\bf Villani Admissible Radial Profile}]
\label{ass:lora-admissible-profile}\label{ass:mh-admissible-profile}
For the $\phi$ defined in Definition~\ref{def:lora-general-radial-potential}. Assume:
\begin{enumerate}
\item[(P1)] \hspace{6.00em} $\phi\in C^2([0,\infty)), \qquad \phi'(0)=0.$

\item[(P2)] \hspace{6.00em} $\lim_{a\to\infty}\frac{\phi'(a)}{a}=+\infty.$

\item[(P3)] \hspace{6.00em} $\lim_{a\to\infty} \frac{[\phi''(a)]_+}{\phi'(a)^2}=0, \qquad [b]_+\coloneqq\max\{b,0\}.$
\end{enumerate}
\end{assumption}

\begin{remark}
\label{rem:mh-admissible-examples}
Assumption~\ref{ass:mh-admissible-profile} is satisfied by several simple
regularization profiles of direct interest. In particular, the logarithmic profile $\phi_{\log}(a)=a^2\log(1+a^2)$, and, for every $\varepsilon>0$, the polynomial profile $\phi_{\varepsilon}(a)=a^{2+\varepsilon}$, are both admissible. These two examples are verified explicitly in Corollaries~\ref{cor:mh-log-profile-standalone} and~\ref{cor:mh-joint-polynomial}, respectively.
\end{remark}

Now we have assembled all the components to state the global convergence of the SDE \eqref{eq:sde} for the regularized loss function of depth-2 nets under LoRA constraints and the attention-based model. The SDE is modeled over the parameter space $\vT$, which denotes the trainable parameters of the corresponding neural network or attention model. The corresponding invariant Gibbs measure $\mu_s(\dd{\vT})$ is defined as
\begin{equation} \label{eq:gibbs}
\mu_s(d\vT) \coloneqq \frac{1}{Z_s} \exp\left(-\frac{2}{s}\tilde{V}(\vT)\right) d\vT,
\end{equation}
where $\tilde{V}(\vT)$ represents any of the factor-regularized potentials defined above, $s > 0$ acts as the temperature parameter (proportional to the learning rate), and $Z_s$ is the normalization constant.

\section{Main Results}\label{sec:mainresults}

Given the formal setup in the previous section, we first state our key result showing that the regression loss function associated with the softmax-attention layer, defined in Definition~\ref{def:mh-general-profile-potential}, satisfy the Villani condition.

\begin{theorem}[{\bf Attention-Based Regression Loss is a Villani Function}]
\label{thm:mh-general-profile-villani}
Suppose the multi-head attention training data satisfy Definition~\ref{as2}, the value matrix $\mW_v$ is bounded as in Definition~\ref{def:mh-kq-attn}, and suppose $\phi$ satisfies Assumption~\ref{ass:mh-admissible-profile}. Then, for every fixed $H\in\mathbb N$ and every $\lambda>0$, the potential $\widetilde V_{ H,\mathrm{ATT}}$ in
Definition~\ref{def:mh-general-profile-potential} satisfies the Villani condition in
Definition~\ref{def:villani}.
\end{theorem}

We give the proof of the above in Section~\ref{sec:proofattnreg}.

We next present our second key result, establishing that the regression loss function for a depth-2 neural network with LoRA, as defined in Definitions~\ref{def:lora-general-radial-potential}, satisfy the Villani condition.

\begin{theorem}[{\bf Depth-2 Neural Net Based Regression Loss under LoRA is a Villani Function}]
\label{thm:lora-general-profile-villani}
Suppose the shallow-net training data satisfy Definition~\ref{as2}, the base matrix $\mW_0$ and scaling $\alpha$ are fixed with $\mW_0$ and $\va$ being bounded as in Definition~\ref{def:shallow-lora-param}, and suppose $\phi$ satisfies Assumption~\ref{ass:lora-admissible-profile}. Then, for every $\lambda>0$, the potential $\widetilde V_{\mathrm{LoRA}}$ in
Definition~\ref{def:lora-general-radial-potential} satisfies the Villani condition in
Definition~\ref{def:villani}.
\end{theorem}

We give the proof of the above in Appendix~\ref{app:lora-villani}.
% \note{need here the location of the proof of above}

% \note{$\log^\star$ is a valid $\phi$.}

% \note{why is the admissibility being proven twice? That proof is about $\phi$ and has nothing to do with LoRA or attention, right?}

By Theorems~\ref{thm:mh-general-profile-villani} and \ref{thm:lora-general-profile-villani}, all considered loss functions with their associated neural architectures satisfy the Villani condition. Consequently, we may invoke Theorem~1 of \cite{shi2023villani} to obtain the convergence result for the SDE~\eqref{eq:sde}, as stated in Theorem~\ref{theorem3} of Appendix~\ref{app:conv_proof}.

% \note{Towards proving our two results ... Now we focus on ... with an explicit upperbound on the \Poincare constant to give explicit rates ... }

The convergence rate of the SDE~\eqref{eq:sde} is governed by the \Poincare constant $C_{PI}$ of the corresponding Gibbs measure. Towards obtaining more explicit convergence guarantees for our models of interest, it is essential to derive quantitative bounds on $C_{PI}$. To this end, Lemma~\ref{lem:bpc-general-poincare-bound} establishes a general upper bound on the \Poincare constant for a broad class of logarithmically regularized loss functions. Consequently, Theorems~\ref{thm:bpc-attention-poincare-bound} and ~\ref{cor:bpc-lora-poincare-bound} yield convergence rates for the SDE~\eqref{eq:sde} with nearly dimension free explicit bounds on $C_{PI}$ for these models.

%We then show that the logarithmically regularized losses arising from multi-head attention and depth-$2$ neural networks under LoRA constraints satisfy the assumptions of this lemma
% \note{start with motivating why the above theorems are needed and whats the difference between them!}

\begin{definition}[{\bf \Poincare Inequality}] \label{def:poincare}
    For a probability measure $\pi$ on $\R^D$, define its optimal
    \Poincare constant by
    \begin{equation}
    \label{eq:bpc-optimal-poincare-constant}
    C_{\mathrm{PI}}^\star(\pi) \coloneqq \inf\left\{ C>0: \Var_\pi(h) \leq C\int_{\R^D}\|\nabla h(\vT)\|^2\,\pi(d\vT) \text{ for every }h\in H^1(\pi) \right\}.
    \end{equation}
    
    Following the normalization used in \cite{shi2023villani}, the corresponding \Poincare inequality is
    \begin{equation}
    \label{eq:bpc-spectral-gap-poincare}
    \operatorname{Var}_{\mu_s}(h) \leq C_{\mathrm{PI}}^\star(\mu_s) \int_{\mathbb{R}^D} \|\nabla h(\vT)\|^2\,\mu_s(d\vT), \qquad h\in H^1(\mu_s).
    \end{equation}
\end{definition}

\begin{lemma}[{\bf Upper bound on the \Poincare Constant for The Logarithmic Profile}]
\label{lem:bpc-general-poincare-bound}
Let $D\geq 2$, $\lambda>0$, and $s>0$. Let
$\widehat R:\R^D\rightarrow\R$ be a bounded measurable function with $\operatorname{osc}(\widehat R) \coloneqq \sup_{\vT\in\R^D}\widehat R(\vT) - \inf_{\vT\in\R^D}\widehat R(\vT) \leq M_{\widehat R} < \infty$. Consider the corresponding potential,
% \note{single headed or multi headed attention here?(also in appendix}
\[ 
\tilde V_{\widehat R}(\vT) = \widehat R(\vT) + \frac{\lambda}{2} \|\vT\|^2\log\bigl(1+\|\vT\|^2\bigr), \qquad \vT\in\R^D, 
\]
and its Gibbs measure, $\mu_s = \frac{1}{Z_s} \exp\left( -\frac{2}{s}\tilde V_{\widehat R}(\vT) \right)$. Then $\mu_s$ is well-defined and satisfies the \Poincare inequality \eqref{eq:bpc-spectral-gap-poincare} for every $h\in H^1(\mu_s)$, where
\begin{equation}
C_{\mathrm{PI}}^\star(\mu_s)
\le
\exp\!\left(\frac{2M_{\widehat R}}{s}\right)
\left(
1+\frac{s}{\lambda\log 2}
\right)
<\infty.
\label{eq:G5}
\end{equation}

In particular, for every fixed $s>0$ and $\lambda>0$, the \Poincare constant of $\mu_s$ is finite.
\end{lemma}

We give the proof of the above in Appendix~\ref{sec:BPC}.

\begin{theorem}[{\bf Convergence of SDE for Attention-Based Regression}]
\label{thm:bpc-attention-poincare-bound}
Recalling the regularized Q-K multi-head loss $\tilde V_{H,\mathrm{ATT}}(\cdot)$ given in Definition~\ref{def:mh-general-profile-potential} for $\phi(r) = r^2 \log(1+r^2)$ and data as given in Definition~\ref{as2}, we define $\vT_H = \left( \mW_Q^1,\mW_K^1,\ldots,\mW_Q^H,\mW_K^H \right) \in\R^D, D=2Hdr$, where $\mW_Q^j, \mW_K^j$ is as in Definition~\ref{def:mh-kq-attn}. 
Consider the corresponding Gibbs measure $\mu_s(d\vT_H) = \frac{1}{Z_s} \exp\left( -\frac{2}{s}\tilde V_{H,\mathrm{ATT}}(\vT_H) \right)d\vT_H, s>0$. Recall the bounds $\|\mX_i\|_F\leq B_x, \|\mY_i\|_F\leq B_y, \|\mW_v\|_F\leq B_w,$ as in Definitions~\ref{as2} and~\ref{def:mh-kq-attn}. If $D=2Hdr\geq2$, then $\mu_s$ satisfies the \Poincare inequality (Defintion~\ref{def:poincare}), where
\begin{equation}
\label{eq:bpc-attention-poincare-bound}
C_{\mathrm{PI}}^\star(\mu_s) \leq \exp\left( \frac{\bigl(H\sqrt{t}B_xB_w+B_y\bigr)^2}{s} \right) \left( 1+\frac{Hs}{\lambda\log 2} \right)<\infty.
\end{equation}
Further, suppose the initial probability density is $p_0 \in L^2((\mu_s)^{-1})$ of the SDE \eqref{eq:sde} where $\mu_s$ is the corresponding Gibbs measure \eqref{eq:gibbs}. Then there exists constant $S > 0$, such that if we further choose the learning rate $s \leq \min\left\{\frac{\epsilon}{2A}, S\right\},$ and the time $t$ satisfies, $t \geq \frac{1}{\lambda_s} \log \left( \frac{2 D(s, p_0)}{\epsilon} \right),$ where $D(s, p_0) = C(s) \cdot \|p_0 - \mu_s\|_{L^2((\mu_s)^{-1})}$, $\lambda_s = \frac{s}{2 C_{PI}^\star(\mu_s)}$ and $C(s)$ is a positive constant, then,
\begin{align}
    \mathbb{E}[\tilde V_{H,\mathrm{ATT}}(\vT_H)] - \tilde V_{H,\mathrm{ATT}}^{\star} \leq \epsilon,
\end{align}
where  $\tilde V_{H,\mathrm{ATT}}^{\star} = \inf~\tilde V_{H,\mathrm{ATT}}(\vT)$ is the global minimum of the respective loss.
\end{theorem}

\begin{theorem}[{\bf Convergence of SDE for Depth-2 Neural Net Based Regression under LoRA}]
\label{cor:bpc-lora-poincare-bound}
Recalling the regularized LoRA loss on a depth-2 net  $\widetilde V_{\mathrm{LoRA}}(\cdot)$ as given in Definition~\ref{def:lora-general-radial-potential} for $\phi(r) = r^2 \log(1+r^2)$, data as given in Definition~\ref{as2} and suppose the activation $\sigma$ satisfies Assumption~\ref{as1}, we define $\vT=(\mU,\mV)\in\R^D, D=(p+d)r,$ where $\mU, \mV$ is as given in Definition~\ref{def:shallow-lora-param} and the Gibbs measure of $\widetilde V_{\mathrm{LoRA}}(\vT)$ is $\mu_s = \frac{1}{Z_s} \exp\left( -\frac{2}{s}\widetilde V_{\mathrm{LoRA}}(\vT) \right)$. Recall the bounds $\|y_i\|_F\leq B_y, \norm{\va} \leq B_a, \abs{\sigma}\leq B_\sigma,$ as in Definitions~\ref{as2}, \ref{def:shallow-lora-param} and Assumption~\ref{as1}. If $D=(p+d)r\geq2$, then $\mu_s$ satisfies the \Poincare inequality (Definition~\ref{def:poincare}), where
\begin{equation}
\label{eq:bpc-lora-poincare-bound}
C_{\mathrm{PI}}^\star(\mu_s) \leq \exp\left( \frac{(B_y+B_aB_\sigma\sqrt p)^2}{s} \right) \left( 1+\frac{s}{\lambda\log 2} \right) <\infty.
\end{equation}
% where $B_y+B_aB_\sigma\sqrt p$ is defined in Definition~\ref{def:shallow-lora-param} and Definition~\ref{as2}.
% \note{maybe the constant should be $(B_y+B_aB_\sigma\sqrt p)$ instead of $B_0$ (needs check) }
Further, suppose the initial probability density is $p_0 \in L^2((\mu_s)^{-1})$ of the SDE \eqref{eq:sde} where $\mu_s$ is the corresponding Gibbs measure \eqref{eq:gibbs}. Then there exists constant $S > 0$, such that if we further choose the learning rate $s \leq \min\left\{\frac{\epsilon}{2A}, S\right\},$ and the time $t$ satisfies, $t \geq \frac{1}{\lambda_s} \log \left( \frac{2 D(s, p_0)}{\epsilon} \right),$ where $D(s, p_0) = C(s) \cdot \|p_0 - \mu_s\|_{L^2((\mu_s)^{-1})}$, $\lambda_s = \frac{s}{2 C_{PI}^\star(\mu_s)}$ and $C(s)$ is a positive constant, then,
\begin{align}
    \mathbb{E}[\widetilde V_{\mathrm{LoRA}}(\vT)] - \widetilde V_{\mathrm{LoRA}}^\star \leq \epsilon,
\end{align}
where $\widetilde V_{\mathrm{LoRA}}^\star = \inf~\widetilde V_{\mathrm{LoRA}}(\vT)$ is the global minimum of the respective loss.
\end{theorem}

The proof of the above two theorems are given in Section \ref{sec:poincspec}.

We note that since $\lambda$ and $\epsilon$ can be set to be arbitrarily small positive numbers for the above convergence, it follows that such a mild regularizer would have negligible effect at small/finite weight values w.r.t unregularized loss and that the regularization only appreciably affects the shape of the loss at infinity. 

\begin{remark}
    Theorems~\ref{thm:bpc-attention-poincare-bound} and \ref{cor:bpc-lora-poincare-bound} requires choosing $s = O(\varepsilon)$. Consequently, the \Poincare constant is of the order $O\left((1+\varepsilon)\exp(\frac{1}{\varepsilon})\right)$ in both the cases. Thus, the loss converges in expectation at a rate of $O\left( \frac{1+\epsilon}{\epsilon} e^{\frac{1}{\epsilon}} \cdot  \log(\frac{2 C(\epsilon) \norm{p_0 - \mu_\epsilon}_{L^2((\mu_\epsilon)^{-1})}}{\epsilon})\right)$ = $O\left( \frac{1}{\epsilon} e^{\frac{1}{\epsilon}} \cdot \max\left( \log(\frac{2}{\epsilon}), \log(\frac{C(\epsilon)}{\epsilon}), \log(\frac{\norm{p_0 - \mu_\epsilon}_{L^2((\mu_\epsilon)^{-1})}}{\epsilon}) \right)\right)$.
\end{remark}

% \note{track the percolation of $\epsilon$ to the run-time, as much as possible}

\begin{remark}[\bf Necessity of Factor Regularization]
We note that the factorized loss function 
$
\gL(\mT) = \frac{1}{n} \sum_{i=1}^n \ell_i(\mU\mV^\top),
$
where $\mT = (\mU,\mV)$, exhibits a scaling invariance under the transformation, 
$
g_\mA(\mU, \mV) = (\mU\mA, \mV\mA^{-T}), \quad \mA \in \mathbb{R}^{r \times r}, \ \det(\mA) \neq 0,
$
since, 
\[
\gL(\mU\mA, \mV\mA^{-T}) = \gL(\mU, \mV).
\]

% \note{should the $\mA$ parameterization be different?}

So the potential $\gL(\vT)$ is constant along the non-compact orbits generated by the general linear group.

As a consequence, the Gibbs' measure
$
\mu_\gamma \propto e^{-\gamma \gL(\vT)}
$
is non-normalizable. Specifically, for any fixed rank-$r$ matrix $\mW_0 \neq 0$, $\gL(\vT)$ is constant along the orbit 
\(\mathcal{O} = \{ (\mU\mA, \mV\mA^{-T}) \mid \det(\mA) \neq 0 \}\), which extends infinitely far from the origin. Hence, the partition function
\[
Z = \int_{\R^D} e^{-\gamma \gL(\vT)} d\vT
\]
diverges because it includes an integral of a non-zero constant density over an infinite-volume set. Consequently, the confining condition is violated, and the Poincaré Inequality cannot hold for the unregularized factorized loss.
\end{remark}

\section{Proof of Villani Conditions for Regression on Attention}\label{sec:proofattnreg}

% Throughout this section, 

Towards stating the proofs we note the following notations,
we consider the multi-head attention loss defined in Definition~\ref{def:mh-general-profile-potential}. We recall from Definition~\ref{def:mh-kq-attn}, $H$ is the number of heads and for each head $j\in\{1,\ldots,H\}$, $\mW_Q^{(j)}$ and $\mW_K^{(j)}$ denote the query and key matrices of the corresponding attention layer, respectively. Here, we will define the head-wise parameter vector as,
\begin{equation}
\vT_j \coloneqq
\begin{pmatrix}
\operatorname{vec}(\mW_Q^{(j)})\\[1mm]
\operatorname{vec}(\mW_K^{(j)})
\end{pmatrix}
\in\mathbb R^{D_0},
\qquad
D_0\coloneqq 2dr.
\label{eq:mh-head-vector}
\end{equation}
The full vector of trainable query and key weights is
\begin{equation}
\vT \coloneqq
\begin{pmatrix}
\vT_1\\
\vdots\\
\vT_H
\end{pmatrix}
\in\mathbb R^D,
\qquad
D\coloneqq HD_0=2Hdr.
\label{eq:mh-full-vector}
\end{equation}
The head-wise and total parameter norms satisfy
\begin{equation}
\lVert\vT_j\rVert = \Bigl( \lVert\mW_Q^{(j)}\rVert_F^2 + \lVert\mW_K^{(j)}\rVert_F^2 \Bigr)^{1/2}, \qquad \lVert\vT\rVert = \Bigl(\sum_{j=1}^H\lVert\vT_j\rVert^2\Bigr)^{1/2}.
\label{eq:mh-radii}
\end{equation}

\begin{lemma}[Row-Softmax Bounds]
\label{lem:mh-rowsoftmax-bounds}
For $\beta>0$ and $\mM\in\mathbb R^{t\times t}$, let
\begin{equation}
\mS \coloneqq \operatorname{RowSoftMax}_\beta(\mM), \qquad \mS_{pq} \coloneqq \frac{\exp(\beta\mM_{pq})} {\sum_{k=1}^t\exp(\beta\mM_{pk})}.
\label{eq:mh-softmax-entry-definition}
\end{equation}
Then
\begin{equation}
\lVert\mS\rVert_F\le\sqrt t.
\label{eq:mh-softmax-output-bound}
\end{equation}
Moreover, for every perturbation $\dd\mM$, the induced first and second differentials
satisfy
\begin{align}
\lVert\dd\mS\rVert_F &\le \beta B_{s'}\lVert\dd\mM\rVert_F, \qquad B_{s'}\coloneqq2, \label{eq:mh-softmax-jacobian-bound}\\
\lVert\dd^2\mS\rVert_F &\le \beta^2B_{s''}\lVert\dd\mM\rVert_F^2, \qquad B_{s''}\coloneqq6t^2,
\label{eq:mh-softmax-hessian-bound}
\end{align}
where
$\dd\mS=D\operatorname{RowSoftMax}_\beta(\mM)[\dd\mM]$ and
$\dd^2\mS=D^2\operatorname{RowSoftMax}_\beta(\mM)[\dd\mM,\dd\mM]$.\footnote{
Let $\mathcal{M}:=\mathbb{R}^{t\times d}$ and let $S:\mathcal{M}\to\mathcal{M}$ be
twice Fr\'echet differentiable at $\mM\in\mathcal{M}$. The first derivative of
$S$ at $\mM$ is the \emph{linear} map $\dd S(\mM):\mathcal{M}\to\mathcal{M}$ whose
value in the direction $\mH\in\mathcal{M}$ is written with square brackets,
\begin{equation}
\dd S(\mM)[\mH]\;:=\;\left.\frac{d}{d\tau}\right|_{\tau=0}S(\mM+\tau\mH)\;\in\;\mathcal{M},
\label{eq:first-derivative-bracket}
\end{equation}
and the second derivative at $\mM$ is the \emph{symmetric bilinear} map
$\dd^2S(\mM):\mathcal{M}\times\mathcal{M}\to\mathcal{M}$,
\begin{equation}
\dd^2S(\mM)[\mH,\mK]\;:=\;
\left.\frac{\partial^2}{\partial\tau\,\partial\sigma}\right|_{\tau=\sigma=0}
S\!\left(\mM+\tau\mH+\sigma\mK\right)
\;=\;
\left.\frac{d}{d\sigma}\right|_{\sigma=0} \dd S(\mM+\sigma\mK)[\mH].
\label{eq:second-derivative-bracket}
\end{equation}
Additional details regarding this notation can be found in Remark~\ref{rem:notationlem1} in Appendix~\ref{app:multihead-general-profile}.
% \note{this referencing is wrong}
}

\end{lemma} 

\begin{lemma}[Gradient and Laplacian Bounds for the Multi-Head Attention Data Term]
\label{lem:mh-data-bounds}
Recall from Definition~\ref{as2}, that for every $i=1,\dots,n$, $\norm{\mX_i}_F\le B_x, \norm{\mY_i}_F\le B_y,$ and from Definition~\ref{def:mh-kq-attn} that $\norm{\mW_v}_F\le B_w$. For the empirical data term $\widehat R_{A,H}$ in Definition~\ref{def:mh-general-profile-potential},
\begin{equation}
\left\lVert\nabla_{\vT}\widehat R_{A,H}(\vT)\right\rVert
\le
C_{\nabla,H}\lVert\vT\rVert,
\label{eq:mh-data-gradient-bound}
\end{equation}
and
\begin{equation}
\left|\Delta_{\vT}\widehat R_{A,H}(\vT)\right|
\le
C_{\Delta,H}\lVert\vT\rVert^2,
\label{eq:mh-data-laplacian-bound}
\end{equation}
where
\begin{equation}
C_{\nabla,H}
\coloneqq
\left(H\sqrt t\,B_xB_w+B_y\right)
B_wB_x\beta B_{s'}
\frac{B_x^2}{\sqrt d},
\label{eq:mh-C-gradient}
\end{equation}
and
\begin{equation}
C_{\Delta,H}
\coloneqq
\frac{B_x^4}{d}
\left[
(\beta B_{s'}B_xB_w)^2
+
\left(H\sqrt t\,B_xB_w+B_y\right)
B_xB_w\beta^2B_{s''}
\right].
\label{eq:mh-C-laplacian}
\end{equation}
\end{lemma}

The above two lemmas are proved in Appendix~\ref{multiQK-Lemmas}.

\begin{proof}[Proof of Theorem~\ref{thm:mh-general-profile-villani}]
We verify the four requirements of Definition~\ref{def:villani} as follows. In Step~1, we establish the first requirement. In Step~2, we verify the second and third requirements. Finally, in Steps~3 to 6, we verify the fourth requirement.

\paragraph{\hspace{1.00em}\bf (Step 1: $C^2$-Regularity.)}
The first requirement in Definition~\ref{def:villani} is that the potential be twice
continuously differentiable. We recall from \eqref{eq:mh-general-potential} of Definition~\ref{def:mh-general-profile-potential},
\[
\widetilde V_{ H,\mathrm{ATT}}(\vT)
=
\widehat R_{A,H}(\vT)
+
R_{H}(\vT),
\]
where $R_{H}(\vT) = \frac{\lambda}{2H} \sum_{j=1}^H \phi\!\left(\lVert\vT_j\rVert\right)$. The data term will be shown below to be smooth. Thus, the only point requiring special
care is the regularizer. Although $\phi\in C^2([0,\infty))$ by (P1) of Assumption~\ref{ass:mh-admissible-profile}, the Euclidean norm
is not differentiable at the origin. We therefore verify the regularity of each head-wise
radial map separately at $\vT_j\neq\vzero$ and at $\vT_j=\vzero$.

For a fixed head $j$, consider the map
\begin{equation}
\vT_j
\longmapsto
\phi\!\left(\lVert\vT_j\rVert\right),
\qquad
\vT_j\in\mathbb R^{D_0}.
\label{eq:mh-block-radial-function}
\end{equation}
First suppose that $\vT_j\neq\vzero$. The Euclidean norm is smooth away from the origin,
with
\[
\nabla_{\vT_j}\lVert\vT_j\rVert
=
\frac{\vT_j}{\lVert\vT_j\rVert},
\qquad
\nabla_{\vT_j}^2\lVert\vT_j\rVert
=
\frac{1}{\lVert\vT_j\rVert}
\left(
\mI_{D_0}
-
\frac{\vT_j\vT_j^\top}{\lVert\vT_j\rVert^2}
\right).
\]
Hence, the chain rule gives
\begin{equation}
\nabla_{\vT_j}
\phi\!\left(\lVert\vT_j\rVert\right)
=
\frac{\phi'\!\left(\lVert\vT_j\rVert\right)}
{\lVert\vT_j\rVert}
\vT_j,
\qquad
\vT_j\neq\vzero,
\label{eq:mh-radial-gradient}
\end{equation}
and
\begin{align}
\nabla_{\vT_j}^2
\phi\!\left(\lVert\vT_j\rVert\right)
&=
\phi''\!\left(\lVert\vT_j\rVert\right)
\frac{\vT_j\vT_j^\top}{\lVert\vT_j\rVert^2}
+
\frac{\phi'\!\left(\lVert\vT_j\rVert\right)}
{\lVert\vT_j\rVert}
\left(
\mI_{D_0}
-
\frac{\vT_j\vT_j^\top}{\lVert\vT_j\rVert^2}
\right)
\notag\\
&=
\frac{\phi'\!\left(\lVert\vT_j\rVert\right)}
{\lVert\vT_j\rVert}
\mI_{D_0}
+
\left(
\phi''\!\left(\lVert\vT_j\rVert\right)
-
\frac{\phi'\!\left(\lVert\vT_j\rVert\right)}
{\lVert\vT_j\rVert}
\right)
\frac{\vT_j\vT_j^\top}{\lVert\vT_j\rVert^2}.
\label{eq:mh-radial-hessian}
\end{align}
Here $\frac{\vT_j\vT_j^\top}{\lVert\vT_j\rVert^2}$ is the rank-one orthogonal projector onto $\operatorname{span}\{\vT_j\}$, namely, the
radial direction. It has eigenvalue $1$ in the direction of $\vT_j$ and eigenvalue $0$
on its orthogonal complement. Therefore, $\left\lVert\frac{\vT_j\vT_j^\top}{\lVert\vT_j\rVert^2} \right\rVert_{\mathrm{op}} =1,$ where $\lVert\cdot\rVert_{\mathrm{op}}$ denotes the Euclidean operator norm.

It remains to examine the origin. By (P1) of Assumption~\ref{ass:mh-admissible-profile}, $\phi'(0)=0$, and the fundamental theorem of
calculus gives, for every $\vT_j\neq\vzero$,
\begin{equation}
\frac{\phi'\!\left(\lVert\vT_j\rVert\right)}
{\lVert\vT_j\rVert}
=
\int_0^1
\phi''\!\left(\theta\lVert\vT_j\rVert\right)\,d\theta
\longrightarrow
\phi''(0)
\qquad
\text{as }\vT_j\to\vzero.
\label{eq:mh-phi-prime-origin-limit}
\end{equation}
The limit follows from the continuity of $\phi''$ at the origin.

We first verify differentiability of the head-wise radial map at the origin. Since
$\phi'(0)=0$ and $\phi'$ is continuous,
\[
\frac{
\left|
\phi\!\left(\lVert\vT_j\rVert\right)-\phi(0)
\right|
}{
\lVert\vT_j\rVert
}
=
\left|
\int_0^1
\phi'\!\left(\theta\lVert\vT_j\rVert\right)\,d\theta
\right|
\longrightarrow
0
\qquad
\text{as }\vT_j\to\vzero.
\]
Thus,
\[
\left.
\nabla_{\vT_j}
\phi\!\left(\lVert\vT_j\rVert\right)
\right|_{\vT_j=\vzero}
=
\vzero.
\]
Moreover, using \eqref{eq:mh-radial-gradient} and
\eqref{eq:mh-phi-prime-origin-limit},
\[
\frac{
\left\lVert
\nabla_{\vT_j}\phi\!\left(\lVert\vT_j\rVert\right)
-
\phi''(0)\vT_j
\right\rVert
}{
\lVert\vT_j\rVert
}
=
\left|
\frac{\phi'\!\left(\lVert\vT_j\rVert\right)}
{\lVert\vT_j\rVert}
-
\phi''(0)
\right|
\longrightarrow
0
\qquad
\text{as }\vT_j\to\vzero.
\]
Therefore, the gradient is differentiable at the origin and
\[
\left.
\nabla_{\vT_j}^2
\phi\!\left(\lVert\vT_j\rVert\right)
\right|_{\vT_j=\vzero}
=
\phi''(0)\mI_{D_0}.
\]

Finally, we verify continuity of this Hessian at the origin. For this purpose, we
specifically let $\vT_j\to\vzero$. By continuity of the Euclidean norm,
\[
\vT_j\to\vzero
\quad\Longrightarrow\quad
\lVert\vT_j\rVert\to0.
\]

Thus, for $\vT_j\neq\vzero$, subtracting $\phi''(0)\mI_{D_0}$ from
\eqref{eq:mh-radial-hessian} gives,
\[
\nabla_{\vT_j}^2 \phi\!\left(\lVert\vT_j\rVert\right) - \phi''(0)\mI_{D_0} = \left( \frac{\phi'\!\left(\lVert\vT_j\rVert\right)} {\lVert\vT_j\rVert} - \phi''(0) \right) \mI_{D_0} + \left( \phi''\!\left(\lVert\vT_j\rVert\right) - \frac{\phi'\!\left(\lVert\vT_j\rVert\right)} {\lVert\vT_j\rVert} \right) \frac{\vT_j\vT_j^\top}{\lVert\vT_j\rVert^2}.
\]
Taking operator norms and applying the triangle inequality together with the
homogeneity of $\lVert\cdot\rVert_{\mathrm{op}}$,
\[
\left\lVert \nabla_{\vT_j}^2 \phi\!\left(\lVert\vT_j\rVert\right) - \phi''(0)\mI_{D_0} \right\rVert_{\mathrm{op}}
\le
\left| \frac{\phi'\!\left(\lVert\vT_j\rVert\right)} {\lVert\vT_j\rVert} - \phi''(0) \right| \left\lVert \mI_{D_0}\right\rVert_{\mathrm{op}}
+
\left| \phi''\!\left(\lVert\vT_j\rVert\right) - \frac{\phi'\!\left(\lVert\vT_j\rVert\right)} {\lVert\vT_j\rVert} \right| \left\lVert \frac{\vT_j\vT_j^\top}{\lVert\vT_j\rVert^2} \right\rVert_{\mathrm{op}}.
\]
Since $\left\lVert \mI_{D_0}\right\rVert_{\mathrm{op}}=1$ and, as noted after
\eqref{eq:mh-radial-hessian}, $\left\lVert \frac{\vT_j\vT_j^\top}{\lVert\vT_j\rVert^2}
\right\rVert_{\mathrm{op}}=1$, we conclude
\begin{align}
\left\lVert\nabla_{\vT_j}^2 \phi\!\left(\lVert\vT_j\rVert\right) - \phi''(0)\mI_{D_0} \right\rVert_{\mathrm{op}} 
\le 
\left| \frac{\phi'\!\left(\lVert\vT_j\rVert\right)} {\lVert\vT_j\rVert} - \phi''(0) \right| + \left| \phi''\!\left(\lVert\vT_j\rVert\right) - \frac{\phi'\!\left(\lVert\vT_j\rVert\right)} {\lVert\vT_j\rVert} \right|.
\label{eq:mh-hessian-origin-bound}
\end{align}

Both terms on the right-hand side tend to zero as $\vT_j\to\vzero$, by
\eqref{eq:mh-phi-prime-origin-limit} and the continuity of $\phi''$ at the origin.
Consequently,
\begin{equation}
\left(
\vT_j
\longmapsto
\phi\!\left(\lVert\vT_j\rVert\right)
\right)
\in C^2(\mathbb R^{D_0}),
\qquad
\left.
\nabla_{\vT_j}
\phi\!\left(\lVert\vT_j\rVert\right)
\right|_{\vT_j=\vzero}
=
\vzero,
\qquad
\left.
\nabla_{\vT_j}^2
\phi\!\left(\lVert\vT_j\rVert\right)
\right|_{\vT_j=\vzero}
=
\phi''(0)\mI_{D_0}.
\label{eq:mh-radial-C2}
\end{equation}
Applying \eqref{eq:mh-radial-C2} independently to all head blocks shows that $R_{ H}(\vT)\in C^2(\mathbb R^D).$

For every $i$ and $j$, the map
\begin{equation}
(\mW_Q^{(j)},\mW_K^{(j)})
\longmapsto
\frac{1}{\sqrt d}
\mX_i\mW_Q^{(j)}(\mW_K^{(j)})^\top\mX_i^\top
\label{eq:mh-polynomial-score-map}
\end{equation}
is polynomial in the matrix entries, while the row-softmax map is infinitely
differentiable. Since sums, matrix products, and the squared Frobenius norm preserve
smoothness, it follows that
\begin{equation}
\widehat R_{A,H}\in C^\infty(\mathbb R^D),
\qquad
\widetilde V_{ H,\mathrm{ATT}}\in C^2(\mathbb R^D).
\label{eq:mh-total-C2}
\end{equation}

\paragraph{\hspace{1.00em}\bf (Step 2: Confinement and Integrability.)} 
Fix $M>0$. By (P2) of Assumption~\ref{ass:mh-admissible-profile}, we have $\phi'(a)/a\to+\infty$ as $a\to\infty$. By the definition of divergence to $+\infty$, applied with the threshold $M$, there exists $R_M>0$ such that
\[
\frac{\phi'(a)}{a}\ge M
\qquad
\text{for every }a\ge R_M.
\]
Since $a>0$, multiplying both sides by $a$ preserves the inequality and yields,
\begin{equation}
\phi'(a)\ge Ma
\qquad
\text{for every }a\ge R_M.
\label{eq:mh-slope-lower-bound}
\end{equation}

Integrating \eqref{eq:mh-slope-lower-bound} gives
\begin{equation}
\phi(a)\ge \frac{M}{2}a^2-C_M
\qquad
\text{for every }a\ge0,
\label{eq:mh-global-phi-lower-bound}
\end{equation}
where one may take
\begin{equation}
C_M
\coloneqq
\max_{0\le a\le R_M}
\left(\frac{M}{2}a^2-\phi(a)\right)<\infty.
\label{eq:mh-CM-definition}
\end{equation}
Since $\widehat R_{A,H}\ge0$, Equations~\eqref{eq:mh-general-potential},
\eqref{eq:mh-radii}, and \eqref{eq:mh-global-phi-lower-bound} imply
\begin{align}
\widetilde V_{ H,\mathrm{ATT}}(\vT)
&\ge
\frac{\lambda}{2H}\sum_{j=1}^H\phi\!\left(\lVert\vT_j\rVert\right)
\ge
\frac{\lambda M}{4H}\sum_{j=1}^H\lVert\vT_j\rVert^2
-
\frac{\lambda}{2}C_M 
=
\frac{\lambda M}{4H}\lVert\vT\rVert^2
-
\frac{\lambda}{2}C_M.
\label{eq:mh-confining-lower-bound}
\end{align}
It follows that
\begin{equation}
\lim_{\lVert\vT\rVert\to\infty}
\widetilde V_{ H,\mathrm{ATT}}(\vT)=+\infty.
\label{eq:mh-confinement}
\end{equation}
Furthermore, for every $s>0$, Equation~\eqref{eq:mh-confining-lower-bound} gives a Gaussian
upper bound for
$\exp(-2\widetilde V_{ H,\mathrm{ATT}}/s)$ outside a compact ball. Hence
\begin{equation}
\int_{\mathbb R^D}
\exp\!\left(
-\frac{2}{s}\widetilde V_{ H,\mathrm{ATT}}(\vT)
\right)d\vT
<\infty.
\label{eq:mh-integrability}
\end{equation}

\paragraph{\hspace{1.00em}\bf (Step 3: Gradient and Laplacian of the Regularizer.)}
For $\vT_j\neq\vzero$, Equation~\eqref{eq:mh-radial-gradient} gives
\begin{equation}
\nabla_{\vT_j}R_{ H}(\vT)
=
\frac{\lambda}{2H}
\frac{\phi'\!\left(\lVert\vT_j\rVert\right)}{\lVert\vT_j\rVert}\vT_j.
\label{eq:mh-regularizer-block-gradient}
\end{equation}
At $\vT_j=\vzero$, this gradient equals $\vzero$ by
\eqref{eq:mh-radial-C2}. Consequently,
\begin{equation}
\left\lVert
\nabla_{\vT}R_{ H}(\vT)
\right\rVert^2
=
\frac{\lambda^2}{4H^2}
G(\vT),
\qquad
G(\vT)
\coloneqq
\sum_{j=1}^H\phi'\!\left(\lVert\vT_j\rVert\right)^2.
\label{eq:mh-G-definition}
\end{equation}
% \note{Are you trying to say below that "We define $\gL_\phi$ as the trace of the Hessian in \eqref{eq:mh-radial-hessian}"?}

% The continuous extension at the origin is
% \begin{equation}
% \mathcal L_\phi(0)
% \coloneqq
% D_0\phi''(0),
% \label{eq:mh-radial-laplacian-origin}
% \end{equation}
% which agrees with the limit in \eqref{eq:mh-phi-prime-origin-limit} and the continuity of
% $\phi''$.
We now compute the Laplacian explicitly. For $\vT_j\neq\vzero$, applying
the radial Hessian formula \eqref{eq:mh-radial-hessian} to the $j$-th block and taking
its trace gives
\begin{align}
\Delta_{\vT_j}\phi\!\left(\lVert\vT_j\rVert\right)
&=
\operatorname{tr}\!\left[
\frac{\phi'\!\left(\lVert\vT_j\rVert\right)}{\lVert\vT_j\rVert}
\mI_{D_0}
+
\left(
\phi''\!\left(\lVert\vT_j\rVert\right)
-
\frac{\phi'\!\left(\lVert\vT_j\rVert\right)}{\lVert\vT_j\rVert}
\right)
\frac{\vT_j\vT_j^\top}{\lVert\vT_j\rVert^2}
\right]
\notag\\
&=
D_0\frac{\phi'\!\left(\lVert\vT_j\rVert\right)}{\lVert\vT_j\rVert}
+
\left(
\phi''\!\left(\lVert\vT_j\rVert\right)
-
\frac{\phi'\!\left(\lVert\vT_j\rVert\right)}{\lVert\vT_j\rVert}
\right)
\notag\\
&=
\phi''\!\left(\lVert\vT_j\rVert\right)
+
(D_0-1)
\frac{\phi'\!\left(\lVert\vT_j\rVert\right)}{\lVert\vT_j\rVert}
=
\mathcal L_\phi\!\left(\lVert\vT_j\rVert\right),
\label{eq:mh-block-laplacian-nonzero}
\end{align}
where for $a>0$ define
\begin{equation}
\mathcal L_\phi(a)
\coloneqq
\phi''(a)+(D_0-1)\frac{\phi'(a)}{a}.
\label{eq:mh-radial-laplacian-profile}
\end{equation}
Here we used
\begin{equation}
\operatorname{tr}(\mI_{D_0})=D_0,
\qquad
\operatorname{tr}\!\left(
\frac{\vT_j\vT_j^\top}{\lVert\vT_j\rVert^2}
\right)
=
\frac{\operatorname{tr}(\vT_j\vT_j^\top)}{\lVert\vT_j\rVert^2}
=1.
\label{eq:mh-projector-trace}
\end{equation}
At $\vT_j=\vzero$, Equation~\eqref{eq:mh-radial-C2} gives
\begin{equation}
\left.
\Delta_{\vT_j}\phi\!\left(\lVert\vT_j\rVert\right)
\right|_{\vT_j=\vzero}
=
\operatorname{tr}\!\left(\phi''(0)\mI_{D_0}\right)
=
D_0\phi''(0)
=
\mathcal L_\phi(0).
\label{eq:mh-block-laplacian-origin}
\end{equation}
Thus,
\begin{equation}
\Delta_{\vT_j}\phi\!\left(\lVert\vT_j\rVert\right)
=
\mathcal L_\phi\!\left(\lVert\vT_j\rVert\right)
\qquad
\text{for every }\vT_j\in\mathbb R^{D_0}.
\label{eq:mh-block-laplacian}
\end{equation}
Finally, recall from \eqref{eq:mh-general-potential} that the $j$-th summand of
$R_{ H}$ depends only on the block $\vT_j$. Hence the full Laplacian is the
sum of the block Laplacians:
\begin{align}
\Delta_{\vT}R_{ H}(\vT)
&=
\frac{\lambda}{2H}
\sum_{j=1}^H
\Delta_{\vT_j}\phi\!\left(\lVert\vT_j\rVert\right) = \frac{\lambda}{2H} \sum_{j=1}^H\mathcal L_\phi\!\left(\lVert\vT_j\rVert\right). \label{eq:mh-regularizer-laplacian}
\end{align}
Set $\kappa(a)\coloneqq[\mathcal L_\phi(a)]_+, K(\vT)\coloneqq\sum_{j=1}^H\kappa\!\left(\lVert\vT_j\rVert\right).$ Then
\begin{equation}
\Delta_{\vT}R_{ H}(\vT)
\le
\frac{\lambda}{2H}K(\vT).
\label{eq:mh-regularizer-laplacian-upper}
\end{equation}

\medskip 
\paragraph{\hspace{1.00em}\bf (Step 4: Multi-Head Asymptotic Domination.)}
We first show that the squared regularizer gradient dominates the total quadratic scale:
\begin{equation}
\lim_{\lVert\vT\rVert\to\infty}
\frac{G(\vT)}{\lVert\vT\rVert^2}=+\infty.
\label{eq:mh-G-dominates-rho}
\end{equation}
Indeed, fix $M>0$, and choose $R_M$ as in
\eqref{eq:mh-slope-lower-bound}. Let $I_M(\vT) \coloneqq \{j\in\{1,\ldots,H\}:\lVert\vT_j\rVert\ge R_M\}.$ Then
\begin{align}
G(\vT) \ge \sum_{j\in I_M(\vT)}\phi'\!\left(\lVert\vT_j\rVert\right)^2 \ge\quad M^2\sum_{j\in I_M(\vT)}\lVert\vT_j\rVert^2 \ge\quad M^2\bigl(\lVert\vT\rVert^2-HR_M^2\bigr).
\label{eq:mh-G-lower-bound}
\end{align}
Dividing by $\lVert\vT\rVert^2$ and letting $\lVert\vT\rVert\to\infty$ gives
\begin{equation}
\liminf_{\lVert\vT\rVert\to\infty}
\frac{G(\vT)}{\lVert\vT\rVert^2}\ge M^2.
\label{eq:mh-G-liminf}
\end{equation}
Since $M>0$ is arbitrary, \eqref{eq:mh-G-dominates-rho} follows. In particular,
\begin{equation}
G(\vT)\to\infty
\qquad
\text{as }\lVert\vT\rVert\to\infty.
\label{eq:mh-G-infinity}
\end{equation}

We next show that the positive part of the regularizer Laplacian is negligible compared with
$G$:
\begin{equation}
\lim_{\lVert\vT\rVert\to\infty}\frac{K(\vT)}{G(\vT)}=0.
\label{eq:mh-K-over-G}
\end{equation}
By (P2) of Assumption~\ref{ass:mh-admissible-profile}, $\phi'(a)>0$ for all sufficiently large $a$. Hence, for all sufficiently large $a$,
\begin{align}
\frac{\kappa(a)}{\phi'(a)^2}
&\le
\frac{[\phi''(a)]_+}{\phi'(a)^2}
+
(D_0-1)\frac{1}{a\phi'(a)}.
\label{eq:mh-kappa-pointwise-bound}
\end{align}
The first term tends to zero by (P3) of Assumption~\ref{ass:mh-admissible-profile}. The second term tends to zero because
\begin{equation}
\frac{1}{a\phi'(a)}
=
\frac{1}{a^2}\frac{1}{\phi'(a)/a}
\longrightarrow0
\label{eq:mh-angular-laplacian-limit}
\end{equation}
by (P2) of Assumption~\ref{ass:mh-admissible-profile}. Thus
\begin{equation}
\lim_{a\to\infty}\frac{\kappa(a)}{\phi'(a)^2}=0.
\label{eq:mh-kappa-over-phi-prime}
\end{equation}
Fix $\delta>0$. By \eqref{eq:mh-kappa-over-phi-prime}, there exists $R_\delta>0$ such that
\begin{equation}
\kappa(a)\le\delta\phi'(a)^2
\qquad
\text{for every }a\ge R_\delta.
\label{eq:mh-kappa-epsilon-bound}
\end{equation}
Since $\kappa$ is continuous, the constant $C_\delta \coloneqq  \max_{0\le a\le R_\delta}\kappa(a)$ is finite. Therefore, $K(\vT) \le \delta G(\vT)+HC_\delta.$ Divide by $G(\vT)$, use \eqref{eq:mh-G-infinity}, and then let $\delta\downarrow0$ to obtain
\eqref{eq:mh-K-over-G}.

\medskip 
\paragraph{\hspace{1.00em}\bf (Step 5: Data-Term Bounds.)}
We recall the bounds \eqref{eq:mh-data-gradient-bound} and
\eqref{eq:mh-data-laplacian-bound} from Lemma~\ref{lem:mh-data-bounds},
\begin{equation}
\left\lVert\nabla_{\vT}\widehat R_{A,H}(\vT)\right\rVert
\le
C_{\nabla,H}\lVert\vT\rVert,
\label{eq:mh-data-gradient-bound-prf}
\end{equation}
and
\begin{equation}
\left|\Delta_{\vT}\widehat R_{A,H}(\vT)\right|
\le
C_{\Delta,H}\lVert\vT\rVert^2.
\label{eq:mh-data-laplacian-bound-prf}
\end{equation}
\medskip 
\paragraph{\hspace{1.00em}\bf (Step 6: Villani Limit.)}
For arbitrary vectors $\va,\vb$,
\begin{equation}
\lVert\va+\vb\rVert^2
\ge
\frac12\lVert\vb\rVert^2-\lVert\va\rVert^2.
\label{eq:mh-young-gradient-bound}
\end{equation}
Apply \eqref{eq:mh-young-gradient-bound} with
$\va=\nabla_{\vT}\widehat R_{A,H}$ and
$\vb=\nabla_{\vT}R_{ H}$. Using
\eqref{eq:mh-G-definition} and \eqref{eq:mh-data-gradient-bound-prf}, we get
\begin{equation}
\left\lVert
\nabla_{\vT}\widetilde V_{ H,\mathrm{ATT}}(\vT)
\right\rVert^2
\ge
\frac{\lambda^2}{8H^2}G(\vT)
-
C_{\nabla,H}^2\lVert\vT\rVert^2.
\label{eq:mh-full-gradient-lower}
\end{equation}
Similarly, Equations~\eqref{eq:mh-regularizer-laplacian-upper} and
\eqref{eq:mh-data-laplacian-bound-prf} imply
\begin{equation}
\Delta_{\vT}\widetilde V_{ H,\mathrm{ATT}}(\vT)
\le
C_{\Delta,H}\lVert\vT\rVert^2
+
\frac{\lambda}{2H}K(\vT).
\label{eq:mh-full-laplacian-upper}
\end{equation}
Combining \eqref{eq:mh-full-gradient-lower} and
\eqref{eq:mh-full-laplacian-upper} yields
\begin{align}
\frac1s
\left\lVert
\nabla_{\vT}\widetilde V_{ H,\mathrm{ATT}}(\vT)
\right\rVert^2
-
\Delta_{\vT}\widetilde V_{ H,\mathrm{ATT}}(\vT)
&\ge
\frac{\lambda^2}{8sH^2}G(\vT)
-
\left(\frac{C_{\nabla,H}^2}{s}+C_{\Delta,H}\right)\lVert\vT\rVert^2
-
\frac{\lambda}{2H}K(\vT)
\notag\\
&=
G(\vT)
\left[
\frac{\lambda^2}{8sH^2}
-
\left(\frac{C_{\nabla,H}^2}{s}+C_{\Delta,H}\right)
\frac{\lVert\vT\rVert^2}{G(\vT)}
-
\frac{\lambda}{2H}\frac{K(\vT)}{G(\vT)}
\right].
\label{eq:mh-villani-factorization}
\end{align}
By \eqref{eq:mh-G-dominates-rho} and \eqref{eq:mh-K-over-G}, in the required limit, the expression in brackets
converges to
\begin{equation}
\frac{\lambda^2}{8sH^2}>0.
\label{eq:mh-positive-bracket-limit}
\end{equation}
By \eqref{eq:mh-G-infinity}, the prefactor $G(\vT)$ tends to $+\infty$. Hence
\begin{equation}
\lim_{\lVert\vT\rVert\to\infty}
\left[
\frac1s
\left\lVert
\nabla_{\vT}\widetilde V_{ H,\mathrm{ATT}}(\vT)
\right\rVert^2
-
\Delta_{\vT}\widetilde V_{ H,\mathrm{ATT}}(\vT)
\right]
=+\infty.
\label{eq:mh-villani-limit}
\end{equation}
Since $s>0$ was arbitrary, all four requirements in Definition~\ref{def:villani} hold.
\end{proof}

\section{An Empirical Study of Regularized Learning of Key and Query Matrices}\label{sec:exp}

% \note{
% 
% \begin{align} 
% \mathbb{R}^{t \times d} \ni \mX \mapsto {\rm Attention}(\mX) \coloneqq &{\rm RowSoftMax}_\beta \left(\frac{\mX \mW_Q \mW_K^\top \mX^\top}{\sqrt{d}}\right) \mX \mW_v \in \mathbb{R}^{t \times d} \\
% &{\rm where,}  ~\mW_Q, \mW_K \in \R^{d \times r} ~\text{and} ~\mW_V \in \R^{d \times d} \nonumber\\
% &{\rm and ~for ~any } ~\mM \in \R^{t \times t}, ~[{\rm RowSoftMax}_\beta(\mM)]_{ij} = \frac{e^{\beta \mM_{ij}}}{\sum_{k=1}^t e^{\beta \mM_{ik}}} \nonumber
% \end{align}
% 
% }

Towards demonstrating an use of our regularized attention losses, we study the two-dimensional Darcy Flow PDE --- which is popularly used as a benchmark in scientific-ML,
\begin{equation}
  -\nabla \cdot \bigl(a(\mathbf{x})\,\nabla u(\mathbf{x})\bigr) = f(\vx),
  \quad \mathbf{x} \in (0,1)^2,
  \label{eq:darcy}
\end{equation}
where $a$ is a spatially varying permeability (diffusion) coefficient, 
$u$ is the unknown pressure field and $f$ is a source function. The regression task is to learn a mapping $a \mapsto u$ from discretised input
fields to discretised solution fields.
Following the benchmark introduced by \cite{li2021fourier}, fields are
discretised on a uniform $64 \times 64$ grid (bilinear-downsampled from the
native $421\times 421$ resolution).
We use $N_{\mathrm{train}} = 900$ samples for training and
$N_{\mathrm{test}} = 124$ held-out samples for evaluation.
Both input and output fields are independently standardised using training-set
mean and standard deviation.

\subsection*{Model Architecture}

We use a patch-based single-head attention regressor
defined as follows.
\begin{itemize}
    \item (Tokenisation)
    Each $64\times 64$ input field $a$ is divided into non-overlapping
    $4\times 4$ patches $\vp_i$, yielding $t = (64/4)^2 = 256$ patches/tokens.
    \item (Embedding)
    Patches are projected to a $d$-dimensional
    representation via a two-stage convolutional encode: 
    $\mathbf{x}_i = \text{ConvEncoder}(\mathbf{p}_i$), and then a flattening layer.
    $\mathbf{X}$ is then given by $[\mathbf{x}_0, \mathbf{x}_1, ..., \mathbf{x}_t]^\top\in\R^{t\times d}$. 
    \item (Positional encoding) 
    Learnable 2-D positional bias $\mP$ for conv-token grid is added to the patch embeddings: $\hat{\mX} = \mX + \mP$. 
    \item (Multi-head attention)
    We fix the number of attention heads to $H=4$. For each head $j\in\{1,\dots,4\}$, query and key projections are
    \begin{equation*}
      \mQ^{(j)} = \hat{\mX}\mW_Q^{(j)},\quad
      \mK^{(j)} = \hat{\mX}\mW_K^{(j)},
      \quad \mW_Q^{(j)},\mW_K^{(j)} \in \mathbb{R}^{d\times r},
    \end{equation*}
    and the corresponding head-wise attention map, with temperature fixed to $1$ ($\beta=1$), is
    \begin{equation*}
      \mA^{(j)} = {\rm RowSoftMax}_\beta\!\left(\frac{\mQ^{(j)} (\mK^{(j)})^\top}{\sqrt{d}}\right) \in \mathbb{R}^{t\times t}.
    \end{equation*}
    The four head-wise maps are aggregated by summation and passed through a single shared, fixed value
    transform $\mW_v \in \R^{t\times t}$ (with $\norm{\mW_v}_F \le B_w$), so that the model output is
    \begin{equation*}
      \widehat{\mY} = \left(\sum_{j=1}^4 \mA^{(j)}\right)\mX\mW_v \in \R^{t\times d},
    \end{equation*}
    consistent with Definition~\ref{def:mh-kq-attn}.
    % \item (Single-head attention)
    % Query, key, and value projections are
    % \begin{equation}
    %   \mQ = \hat{\mX}\mW_Q,\quad
    %   \mK = \hat{\mX}\mW_K,\quad
    %   \mV = \hat{\mX}\mW_V,
    %   \quad \mW_Q,\mW_K \in \mathbb{R}^{d\times r},\;
    %   \mW_V \in \mathbb{R}^{d\times d}.
    % \end{equation}
    % The attention map is computed with temperature fixed to $1$ ($\beta=1$),
    %   $\mA = {\rm RowSoftMax}_\beta\!\left(\frac{\mQ \mK^\top}{\sqrt{d}}\right) \in \mathbb{R}^{t\times t}$.
\end{itemize}

% \note{\Large text needs to be updated to account for H =4 and the changed plots!}

The representation $\displaystyle{\rm Attention}(\mX) = \left(\sum_{j=1}^4 \mA^{(j)}\right)\mX\mW_v \in \mathbb{R}^{t\times d}$ is projected back to patch space by a two-layer net and the resulting patches are rearranged into the predicted $64\times 64$ output field $\hat{u}$. All dimensions are set to $r = d = 64$, with $H=4$ heads.

\subsection*{Two-Phase Training Protocol}
\paragraph{\bf Phase~1 (full pretraining): }
All parameters, embedding layers, projections $\mW^{(j)}_Q$, $\mW^{(j)}_K$, $\mW^{(j)}_V$, for all $j \in \{1,\dots,4\}$ and output MLP, are jointly optimized with plain mean-squared error (MSE) for $500$ epochs using Adam~\cite{kingma2014adam} with learning rate $\eta = 10^{-3}$ and batch size $32$. This phase yields the \emph{task-optimal} values for all weight matrices. 
%$\mW_V^*$ and $\mW_{\mathrm{out}}^*$.

\paragraph{\bf Phase~2 (ablation on ways of regularized training of Query and Key Matrices):}
The goal is to isolate the effect of norm regularisation on the query and key
matrices. Apart from the key and the query weight all other parameters are set to their Phase~1 values.
Then $\mW^{(j)}_Q$ and  $\mW^{(j)}_K$ for all $j \in \{1,\dots,4\}$ are \emph{re-initialised} identically across all three runs
to their initial values used in Phase~1.
Three objectives are then compared over $100$ additional epochs,

\begin{enumerate}
  \item \textbf{Unregularised:}
  \begin{equation}
    \mathcal{L}_{0} = \hat{R}_*(\mT),
    \label{eq:loss-none}
  \end{equation}
  \item \textbf{Log Amplified Frobenius regularization:}
  \begin{equation}
    \mathcal{L}_{\log} = \hat{R}_*(\mT) + R_{H}(\mT), \;
    R_{H}(\mT) \coloneqq \frac{\lambda}{2 \cdot 4}\sum_{j=1}^4\phi\!\left(\Bigl(\norm{\mW_Q^{(j)}}_F^2 + \norm{\mW_K^{(j)}}_F^2\Bigr)^{1/2}\right), \;
    \phi(r) = r^2 \log(1+r^2).
    \label{eq:loss-log}
  \end{equation}
  \item \textbf{Frobenius norm:}
  \begin{equation}
    \mathcal{L}_{frob} = = \hat{R}_*(\mT) + R_{H}(\mT), \;
    R_{H}(\mT) \coloneqq \frac{\lambda}{2 \cdot 4}\sum_{j=1}^4\phi\!\left(\Bigl(\norm{\mW_Q^{(j)}}_F^2 + \norm{\mW_K^{(j)}}_F^2\Bigr)^{1/2}\right), \;
    \phi(r) = r^2.
    \label{eq:loss-power}
  \end{equation}
\end{enumerate}
% \note{the third loss above should change}
% \note{the definition of S should change everywhere}

where $\mT$ is as defined in \eqref{eq:mh-full-vector}  and $\hat{R}_*(\mT)$ denotes the MSE loss with embedding layers, $\mW_V$ and output MLP
frozen at their optimal values.

In all Phase~2 runs we use Adam with $\eta = 10^{-3}$, batch size $32$ (as in Phase-1), and hyper-parameters $\lambda = 10^{-5}$ for \textit{log} and $\lambda=10^{-4}$ for \textit{frobenius-norm}.

% \note{are these $\lambda$s and $\eta$s still the same?}

\begin{figure}
    \centering
    \includegraphics[width=1.0\linewidth]{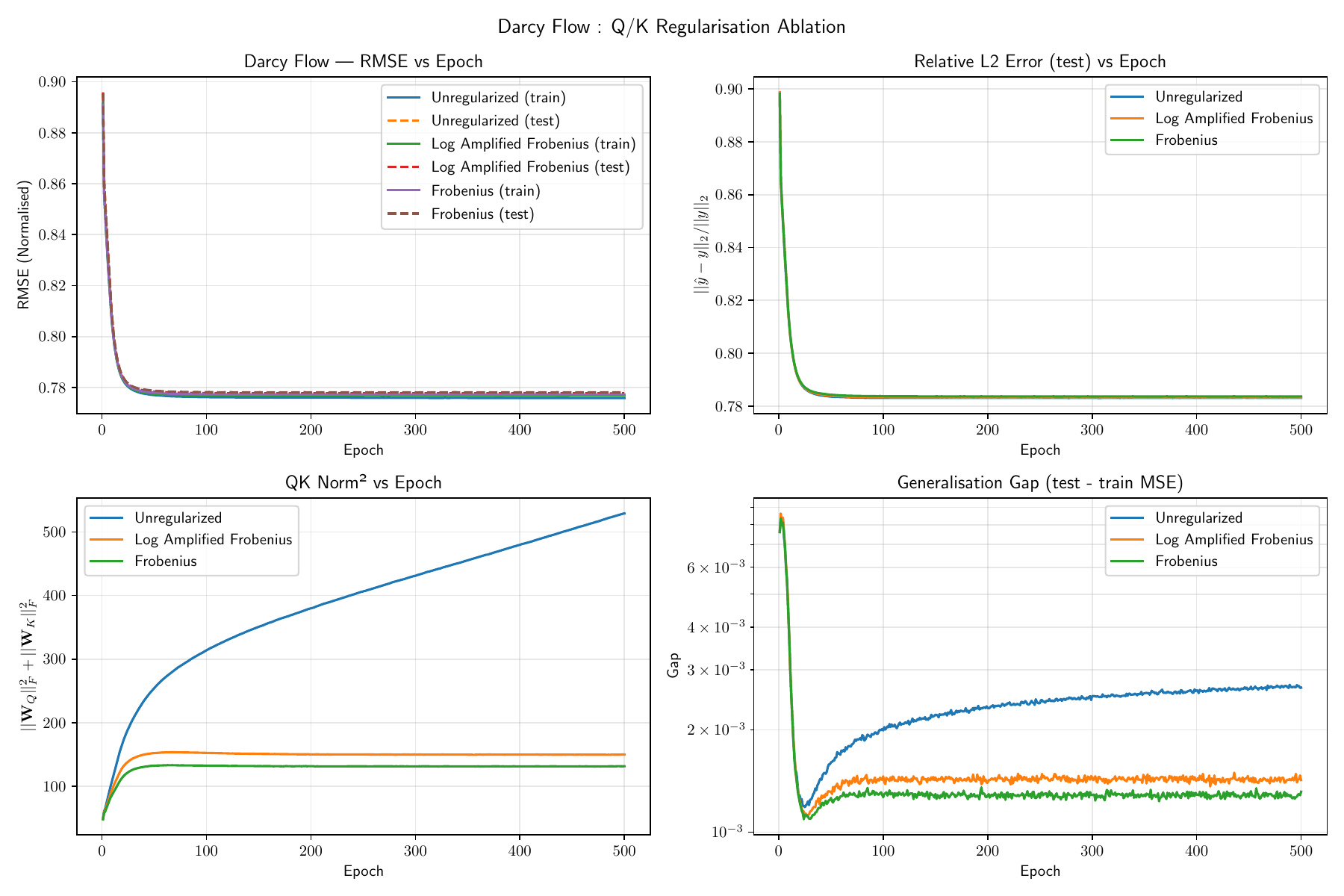}
    \caption{Per-epoch and per-run statistics for different loss functions}
    \label{fig:stats}
\end{figure}

% \begin{figure}
%     \centering
%     \includegraphics[width=0.7\linewidth]{Images/darcy_method_comparison_heatmaps.pdf}
%     \caption{Darcy flow predictions via different loss functions in 2nd Phase. }
%     \label{fig:placeholder}
% \end{figure}

\subsection*{Metrics}
We track the following quantities per epoch and per run.

\begin{itemize}
  \item \textbf{Train/Test RMSE:}
        $\sqrt{\mathbb{E}[\|\hat{u}-u\|^2]}$ in normalised target space.

  \item \textbf{Test Relative L2 Error:}
        %\begin{equation}
          $\varepsilon_{\mathrm{rel}}
          = \frac{\|\hat{u} - u\|_2}{\|u\|_2},$
        %\end{equation}
        the standard PDE-operator-learning benchmark metric.
        
  \item \textbf{Q/K Norm$^2$:}
        $\|\mW_Q\|_F^2 + \|\mW_K\|_F^2$.

  \item \textbf{Generalisation gap:}
        $\Delta = \text{Test MSE} - \text{Train MSE}$.

\end{itemize}

\subsection{Results analysis}
In the very small regularization regime $(\epsilon = 10^{-6}$ and $\lambda=10^{-4}$ or $10^{-5}$), 
all methods, no regularization, ``log'' regularization, and ``frobenius-norm'' regularisation, exhibit nearly identical convergence behaviour in both normalised test/train RMSE and test relative L2 error. 
As shown in Figure~\ref{fig:stats}, all configurations rapidly decay from their initial error and stabilise at approximately the same performance level ($0.78$), indicating that at this scale of regularization, 
there is no meaningful difference in predictive accuracy.

Despite the similarity in error metrics, 
substantial differences emerge in the internal dynamics of the model. 
In particular, the Frobenius norm ($||\mW_Q||^2_F + ||\mW_K||^2_F$) grows steadily in the absence of regularization, reaching significantly larger values over training. 
In contrast, both ``log" and ``frobenius-norm" regularization effectively constrain this growth, with power regularization enforcing the strongest suppression and log regularization yielding a slightly higher but still stable norm plateau. {\em Firstly,} these results demonstrate that even weak regularization of the kind studied here can meaningfully alter the scaling of the attention mechanism without impacting predictive performance.

In the generalization gap plot in Figure~\ref{fig:stats}, all methods initially exhibit a rapid reduction in gap but the unregularized model shows a gradual increase over time, suggesting mild overfitting. Both log and frobenius-norm regularization mitigate this effect, maintaining a consistently lower and more stable gap throughout training, with power again providing the strongest control. {\em Secondly,} this result indicates that despite comparable test error  improved generalization stability is induced by the regularizers being studied.

Overall, these findings suggest that the regularization studied in the presented theory primarily influences model stability. 
In low-$\lambda$ regimes, log regularisation provides a robust default by stabilising the key and query weight norms with minimal intervention, while power regularisation is preferable when stronger suppression of attention weight growth is desired.

\section{Conclusion}\label{sec:conclusion}

This work establishes that the SDE, that mimics the SGD, converges for both attention layers and depth-2 neural networks trained with LoRA, for arbitrary data and network sizes, even under arbitrarily low regularization, and we further provides non-asymptotic convergence rates.

A natural next question is whether the loss function on the attention layer with it's key-query-value matrices using LoRA, satisfy the Villani conditions in the corresponding space of factor matrices. More generally, it remains open to prove convergence guarantees for a full transformer layer in which all three matrices per head ($\mW_Q$, $\mW_K$, $\mW_V$) and the feedforward network weights are trained jointly, with one or more of these components potentially using LoRA.

An interesting direction for future work is to extend our guarantees to more advanced sequence models such as FlashAttention \cite{dao2022flashattention}, Performers \cite{choromanski2021rethinking}, Mamba \cite{gu2023mamba}. In contrast to vanilla attention, which computes softmax attention, these methods modify the computation in different ways. FlashAttention \cite{dao2022flashattention} computes the same softmax attention but using a memory-efficient tiled algorithm, which retains the same mathematical structure as vanilla attention and is thus likely amenable to a similar analysis. Performers \cite{choromanski2021performer} approximates the softmax kernel $\exp(\mQ\mK^\top)$ using positive random feature maps $\phi(\mQ)$ and $\phi(\mK)$, resulting in an approximate attention computation of the form
$\text{Attn}(\mQ,\mK,\mV) \approx \frac{\phi(\mQ)(\phi(\mK)^\top \mV)}{\phi(\mQ)(\phi(\mK)^\top \mathbf{1})},$ which reduces the computational complexity. Mamba \cite{gu2023mamba} replaces attention entirely with an input-dependent state-space model (SSM), yielding a linear-time recurrence. Extending convergence guarantees to these alternative attention mechanisms remains an open problem.


%\bibliographystyle{amsplain}
%Insert the bibliography data here.
% \clearpage 
\bibliographystyle{alpha}
\bibliography{reference}
\clearpage  

\appendix
% \input{Sections/new-LoRA(1)}

% \clearpage 
% \input{Sections/LoRA-lemmas}

% \clearpage 
\section{Proofs of SDE convergence}\label{app:conv_proof}

% \note{\large \bf This segment will have the proof of 2 final corollaries giving SDE convergence with explicit CPI bounds}

\begin{theorem}[{\bf Convergence of SDE for Depth-2 Neural Net Based Regression under LoRA and Attention-Based Regression}]\label{theorem3}
Suppose that Assumption~\ref{as1} holds for the activation $\sigma$. Let $\tilde{L}^{(k)}(\vT)$ denote any of the two regularized potentials defined in Definition~\ref{def:mh-general-profile-potential} and  Definition~\ref{def:lora-general-radial-potential}, where $k \in \{1, 2\}$ indexes the specific model and regularization choice. Suppose the initial probability density is $p_0 \in L^2((\mu_s^{(k)})^{-1})$ of the SDE \eqref{eq:sde} where $\mu_s^{(k)}$ is the corresponding Gibbs measure \eqref{eq:gibbs}. For each $\tilde{L}^{(k)}(\vT)$ satisfying the Villani conditions, there exists a positive $\lambda_s^{(k)} > 0$ and a constant $D^{(k)}(s, p_0)$ such that:
\begin{align}
    \mathbb{E}[\tilde{L}^{(k)}(\vT_t)] - \tilde{L}^{(k)\star} \leq \varepsilon^{(k)}(s) + D^{(k)}(s, p_0) e^{-\lambda_s^{(k)} t},
\end{align}
where $\tilde{L}^{(k)\star} = \inf \tilde{L}^{(k)}(\vT)$ is the global minimum of the respective loss, and $\varepsilon^{(k)}(s) = \mathbb{E}_{\mu_s^{(k)}}[\tilde{L}^{(k)}(\vT)] - \tilde{L}^{(k)\star}$.

Then there exist constants $A^{(k)}, S^{(k)} > 0$ such that $\mathbb{E}_{\mu_s^{(k)}}[\tilde{L}^{(k)}(\vT)] - \tilde{L}^{(k)\star} \leq A^{(k)}s$ for all $s \in (0, S^{(k)}]$. If we further choose the learning rate $s$ such that $s \leq \min\left\{\frac{\epsilon}{2A^{(k)}}, S^{(k)}\right\},$ and the time $t$ satisfies, $t \geq \frac{1}{\lambda_s^{(k)}} \log \left( \frac{2 D^{(k)}(s, p_0)}{\epsilon} \right),$ where $D^{(k)}(s, p_0) = C^{(k)}(s) \cdot \|p_0 - \mu_s^{(k)}\|_{L^2((\mu_s^{(k)})^{-1})}$ and $C^{(k)}(s)$ is a positive constant, then,
\begin{align}
    \mathbb{E}[\tilde{L}^{(k)}(\vT_t)] - \tilde{L}^{(k)\star} \leq \epsilon.
\end{align}
\end{theorem}

\begin{proof}[Proof of Theorem~\ref{theorem3}]
    Let $\tilde{L}^{(k)}(\vT)$ denote any of the four regularized potentials defined in Definition~\ref{def:mh-general-profile-potential}, and Definition~\ref{def:lora-general-radial-potential}, where $k \in \{1, 2\}$ indexes the specific model and regularization choice.

    For any $k \in \{1, 2\}$, perform the following decomposition of the excess risk:
    \begin{align}
        \mathbb{E}[\tilde{L}^{(k)}(\vT_t)] - \tilde{L}^{(k)\star} = \underbrace{\big(\mathbb{E}[\tilde{L}^{(k)}(X_s^{(k)}(\infty))] - \tilde{L}^{(k)\star}\big)}_{\varepsilon^{(k)}(s)} + \underbrace{\big(\mathbb{E}[\tilde{L}^{(k)}(\vT_t)] - \mathbb{E}[\tilde{L}^{(k)}(X_s^{(k)}(\infty))]\big)}_{ \leq D^{(k)}(s,p_0)e^{-\lambda_s^{(k)} t}}
    \end{align}
    By Proposition 5 of \cite{shi2023villani}, 
    \begin{align}
        \varepsilon^{(k)}(s) \leq A^{(k)}s
    \end{align}
    and Proposition 4 of \cite{shi2023villani},
    \begin{align}
        \mathbb{E}[\tilde{L}^{(k)}(\vT_t)] - \mathbb{E}[\tilde{L}^{(k)}(X_s^{(k)}(\infty))] \leq C^{(k)}(s)\|p_0 - \mu_s^{(k)}\|_{L^2((\mu_s^{(k)})^{-1})} e^{-\lambda_s^{(k)} t},
    \end{align}
    where  $p_0 \in L^2((\mu_s^{(k)})^{-1})$ is the initial probability density of the SDE \eqref{eq:sde} .
    
    Since $s \leq \min \left \{\frac{\epsilon}{2A^{(k)}}, S^{(k)}\right\},$
    and the time horizon $t \geq \frac{1}{\lambda_s^{(k)}} \log \left( \frac{2 D^{(k)}(s, p_0)}{\epsilon} \right),$ by Corollary 6 of \cite{shi2023villani} we have
    \begin{align}
        \mathbb{E}[\tilde{L}^{(k)}(\vT_t)] - \tilde{L}^{(k)\star} \leq \epsilon.
    \end{align}
\end{proof}

% \subsection{Proof of Theorem~\ref{} and \ref{}}

\subsection{Proofs of Theorems~\ref{thm:bpc-attention-poincare-bound} and ~\ref{cor:bpc-lora-poincare-bound}}\label{sec:poincspec}

Throughout this section we define
\begin{equation}
    \lambda_s \coloneqq \frac{s}{2 C_{PI}^\star(\mu_s)}, \label{eq:lmbds}
\end{equation}
where $\mu_s$ is the Gibbs measure and $C^\star_{PI}$ is the \Poincare constant of $\mu_s$.

\begin{proof}[Proof of Theorem~\ref{thm:bpc-attention-poincare-bound}]
By Definition~\ref{def:mh-general-profile-potential}, the individual squared loss satisfies
\[
\ell_i(\vT_H) = \frac12 \left\| \mY_i-\hat{\mY}_i(\vT_H) \right\|_F^2.
\]
By Equation~\ref{eq:mh-softmax-output-bound}, for each of the $h$ heads, the corresponding row-softmax matrix has
Frobenius norm at most $\sqrt t$.

Hence, by
Definitions~\ref{def:mh-kq-attn} and \ref{as2},
\[
\left\|
\hat{\mY}_i(\vT_H)
\right\|_F
\leq
h\sqrt t\,B_xB_w.
\]
Therefore, by the triangle inequality,
\[
\begin{aligned}
\ell_i(\vT_H)
&\leq
\frac12
\left(
\|\mY_i\|_F+
\|\hat{\mY}_i(\vT_H)\|_F
\right)^2
\\
&\leq
\frac12
\left(
B_y+H\sqrt t\,B_xB_w
\right)^2.
\end{aligned}
\]
Since each $\ell_i(\vT_H)$ is nonnegative,
\[
0
\leq
\hat R_A(\vT_H)
=
\frac1n\sum_{i=1}^n\ell_i(\vT_H)
\leq
\frac12
\left(
B_y+H\sqrt t\,B_xB_w
\right)^2.
\]
It follows that
\begin{equation}
\label{eq:bpc-attention-oscillation-bound}
\operatorname{osc}(\hat R_A)
\leq
\frac12
\left(
B_y+H\sqrt t\,B_xB_w
\right)^2.
\end{equation}
Applying Lemma~\ref{lem:bpc-general-poincare-bound} with
$\widehat R=\hat R_A$ and
\[
M_{\widehat R}
=
\frac12
\left(
B_y+H\sqrt t\,B_xB_w
\right)^2
\]
gives the following bound on the \Poincare constant,
\begin{equation}
C_{\mathrm{PI}}^\star(\mu_s) \leq \exp\left( \frac{\bigl(H\sqrt{t}B_xB_w+B_y\bigr)^2}{s} \right) \left(1+\frac{Hs}{\lambda\log 2} \right).
\end{equation}
as in Equation~\ref{eq:bpc-attention-poincare-bound}.

Recall the logarithmically regularized multi-head attention potential as in Definition~\ref{def:mh-general-profile-potential} with $\phi(r) = r^2 \log(1+r)$
\[
\tilde V_{H,\mathrm{ATT}}(\vT_H) = \hat R_A(\vT_H) + \frac{\lambda}{2H} \sum_{j=1}^H \left( \|\mW_Q^j\|_F^2+\|\mW_K^j\|_F^2 \right) \log\left( 1+\|\mW_Q^j\|_F^2+\|\mW_K^j\|_F^2 \right),
\]
where $\vT_H = \left( \mW_Q^1,\mW_K^1,\ldots,\mW_Q^H,\mW_K^H \right) \in\R^D, D=2Hdr$ and the corresponding Gibbs' measure is defined as $\mu_s(d\vT_H) = \frac{1}{Z_s} \exp\left( -\frac{2}{s}\tilde V_{H,\mathrm{ATT}}(\vT_H) \right)d\vT_H, s>0$. Theorem~\ref{thm:mh-general-profile-villani} establishes that $\tilde V_{H,\mathrm{ATT}}(\vT_H)$ satisfies the Villani conditions stated in Definition~\ref{def:villani}.

We decompose the excess risk as follows,
\begin{align}
    \mathbb{E}[\tilde V_{H,\mathrm{ATT}}(\vT_{H,t})] - \tilde V_{H,\mathrm{ATT}}^{\star} = \underbrace{\big(\mathbb{E}[\tilde V_{H,\mathrm{ATT}}(X_s(\infty))] - \tilde V_{H,\mathrm{ATT}}^{\star}\big)}_{\varepsilon(s)} + \underbrace{\big(\mathbb{E}[\tilde V_{H,\mathrm{ATT}}(\vT_H)] - \mathbb{E}[\tilde V_{H,\mathrm{ATT}}(X_s(\infty))]\big)}_{ \leq D(s,p_0)e^{-\lambda_s t}},
\end{align}
where  $\tilde V_{H,\mathrm{ATT}}^{\star} = \inf~\tilde V_{H,\mathrm{ATT}}(\vT)$ is the global minimum of the respective loss and $\lambda_s$ is as in \eqref{eq:lmbds}. By Proposition 5 of \cite{shi2023villani}, 
\begin{align}
    \varepsilon(s) \leq As
\end{align}
and Proposition 4 of \cite{shi2023villani},
\begin{align}
    \mathbb{E}[\tilde V_{H,\mathrm{ATT}}(\vT_{H,t})] - \mathbb{E}[\tilde V_{H,\mathrm{ATT}}(X_s(\infty))] \leq C(s)\|p_0 - \mu_s\|_{L^2((\mu_s)^{-1})} e^{-\lambda_s t},
\end{align}
where  $p_0 \in L^2((\mu_s)^{-1})$ is the initial probability density of the SDE \eqref{eq:sde} .

Since $s \leq \min \left \{\frac{\epsilon}{2A}, S\right\},$
and the time horizon $t \geq \frac{1}{\lambda_s} \log \left( \frac{2 D(s, p_0)}{\epsilon} \right),$ by Corollary 6 of \cite{shi2023villani} we have
\begin{align}
    \mathbb{E}[\tilde V_{H,\mathrm{ATT}}(\vT_{H,t})] - \tilde V_{H,\mathrm{ATT}}^{\star} \leq \epsilon.
\end{align}
\end{proof}

\begin{proof}[Proof of Theorem \ref{cor:bpc-lora-poincare-bound}]
By Definition~\ref{def:shallow-lora-param}, $\widehat y_i(\vT)=\va^\top\sigma(\vz_i(\vT))$ with
$\vz_i(\vT)\coloneqq\mW(\vT)\vx_i\in\R^p$, where $\vT = (\mU,\mV)$. Under Assumption~\ref{as1} the activation satisfies
$|\sigma(u)|\le B_\sigma$ for all $u\in\R$, so $\lVert\sigma(\vz_i(\vT))\rVert\le B_\sigma\sqrt p$
and, by Cauchy--Schwarz with $\lVert\va\rVert\le B_a$,
\begin{equation}
\bigl|\widehat y_i(\vT)\bigr|\le B_aB_\sigma\sqrt p
\qquad\text{for every }\vT\in\R^D .
\label{eq:bpc-lora-prediction-bound}
\end{equation}
Combining \eqref{eq:bpc-lora-prediction-bound} with $|y_i|\le B_y$ from Definition~\ref{as2} and setting
\begin{equation}
E\coloneqq B_y+B_aB_\sigma\sqrt p ,
\label{eq:bpc-lora-E}
\end{equation}
Definition~\ref{def:lora-general-radial-potential} gives, exactly as in
\eqref{eq:lora-data-value-bound} of Lemma~\ref{lem:lora-data-bounds},
\begin{equation}
0\;\le\;\Ldata(\vT)\;=\;\frac1n\sum_{i=1}^n\frac12\bigl(\widehat y_i(\vT)-y_i\bigr)^2
\;\le\;\frac12E^2 .
\label{eq:bpc-lora-data-bound}
\end{equation}
Consequently
\begin{equation}
\operatorname{osc}(\Ldata)\coloneqq\sup_{\R^D}\Ldata-\inf_{\R^D}\Ldata\;\le\;\frac12E^2 .
\label{eq:bpc-lora-oscillation-bound}
\end{equation}
Moreover,
\[
\|\vT\|^2 = \|\mU\|_F^2+\|\mV\|_F^2,
\]
so Equation~\ref{eq:lora-general-radial-potential} has exactly the form of
the potential in Lemma~\ref{lem:bpc-general-poincare-bound}. Applying
that lemma with $\widehat R=\Ldata$ and
$M_{\widehat R}=B_0^2/2$ gives the following bound on the \Poincare constant,
\begin{equation}
C_{\mathrm{PI}}^\star(\mu_s) \leq \exp\left( \frac{E^2}{s} \right) \left( 1+\frac{s}{\lambda\log 2} \right).
\end{equation}
as in Equation~\ref{eq:bpc-lora-poincare-bound}.
% \note{same $B_0$ issue here}
For any fixed $\mW_0\in\R^{p\times d}$, any fixed $\alpha\in\R$, $\lambda>0$, and $s>0$, recall the logarithmically regularized LoRA potential,
\[
\widetilde V_{\mathrm{LoRA}}(\vT) = \Ldata(\vT) + \frac{\lambda}{2} \|\vT\|^2\log\bigl(1+\|\vT\|^2\bigr), \qquad \vT=(\mU,\mV)\in\R^D, \qquad D=(p+d)r,
\]
where $\Ldata(\cdot)$ is as defined in Definition~\ref{def:lora-general-radial-potential}, the activation $\sigma$ satisfies Assumption~\ref{as1}, the training data satisfies Definition~\ref{as2},  and its Gibbs measure $\mu_s = \frac{1}{Z_s} \exp\left( -\frac{2}{s}\widetilde V_{\mathrm{LoRA}}(\vT) \right)$. Theorem~\ref{thm:lora-general-profile-villani} establishes that $\widetilde V_{\mathrm{LoRA}}(\vT)$ satisfies the Villani conditions stated in Definition~\ref{def:villani}.

We decompose the excess risk as follows,
\begin{align}
    \mathbb{E}[\widetilde V_{\mathrm{LoRA}}(\vT_t)] - \widetilde V_{\mathrm{LoRA}}^\star = \underbrace{\big(\mathbb{E}[\widetilde V_{\mathrm{LoRA}}(X_s(\infty))] - \widetilde V_{\mathrm{LoRA}}^\star\big)}_{\varepsilon(s)} + \underbrace{\big(\mathbb{E}[\widetilde V_{\mathrm{LoRA}}(\vT_t)] - \mathbb{E}[\widetilde V_{\mathrm{LoRA}}(X_s(\infty))]\big)}_{ \leq D(s,p_0)e^{-\lambda_s t}},
\end{align}
where  $\widetilde V_{\mathrm{LoRA}}^\star = \inf~\widetilde V_{\mathrm{LoRA}}(\vT)$ is the global minimum of the respective loss and $\lambda_s$ is as in \eqref{eq:lmbds}. By Proposition 5 of \cite{shi2023villani}, 
\begin{align}
    \varepsilon(s) \leq As
\end{align}
and Proposition 4 of \cite{shi2023villani},
\begin{align}
    \mathbb{E}[\widetilde V_{\mathrm{LoRA}}(\vT_t)] - \mathbb{E}[\widetilde V_{\mathrm{LoRA}}(X_s(\infty))] \leq C(s)\|p_0 - \mu_s\|_{L^2((\mu_s)^{-1})} e^{-\lambda_s t},
\end{align}
where  $p_0 \in L^2((\mu_s)^{-1})$ is the initial probability density of the SDE \eqref{eq:sde} .

Since $s \leq \min \left \{\frac{\epsilon}{2A}, S\right\},$
and the time horizon $t \geq \frac{1}{\lambda_s} \log \left( \frac{2 D(s, p_0)}{\epsilon} \right),$ by Corollary 6 of \cite{shi2023villani} we have
\begin{align}
    \mathbb{E}[\widetilde V_{\mathrm{LoRA}}(\vT_t)] - \widetilde V_{\mathrm{LoRA}}^\star \leq \epsilon.
\end{align}
\end{proof}

% \clearpage 
% \input{Sections/new-corollary-6}

% \clearpage 
% \input{Sections/dissipation}

\clearpage

\section{Bounding the Poincar\'e Constant for the Gibbs Measures of Multi-Head Attention and LoRA}\label{sec:BPC}

\begin{proof}[Proof of Lemma \ref{lem:bpc-general-poincare-bound}]
We divide the proof into three steps.

\medskip
\quad \textbf{Step 1: the Gibbs measure generated by the regularizer alone is
radial and log-concave.}

Consider the probability measure
\begin{equation}
\label{eq:bpc-reference-measure}
\nu_s(d\vT)
=
\frac{
\exp\left(
-\dfrac{\lambda}{s}
\|\vT\|^2\log\bigl(1+\|\vT\|^2\bigr)
\right)
}{
\displaystyle
\int_{\R^D}
\exp\left(
-\dfrac{\lambda}{s}
\|\vu\|^2\log\bigl(1+\|\vu\|^2\bigr)
\right)d\vu
}
\,d\vT.
\end{equation}
This is precisely the Gibbs measure generated by the regularization
term in $\tilde V_{\widehat R}$, because
\[
\frac{2}{s}
\left[
\frac{\lambda}{2}
\|\vT\|^2\log\bigl(1+\|\vT\|^2\bigr)
\right]
=
\frac{\lambda}{s}
\|\vT\|^2\log\bigl(1+\|\vT\|^2\bigr).
\]

We first verify that the normalization integral is finite. If
$\|\vT\|\geq 1$, then
\[
\log\bigl(1+\|\vT\|^2\bigr)\geq \log 2,
\]
and hence
\[
\exp\left(
-\frac{\lambda}{s}
\|\vT\|^2\log\bigl(1+\|\vT\|^2\bigr)
\right)
\leq
\exp\left(
-\frac{\lambda\log 2}{s}\|\vT\|^2
\right).
\]
The right-hand side is integrable over $\R^D$. On the unit ball,
the integrand is bounded by $1$. Therefore,
\[
\int_{\R^D}
\exp\left(
-\frac{\lambda}{s}
\|\vT\|^2\log\bigl(1+\|\vT\|^2\bigr)
\right)d\vT
<\infty,
\]
so $\nu_s$ is a well-defined probability measure.

Next, we verify log-concavity. Direct differentiation gives
\begin{equation}
\label{eq:bpc-regularizer-gradient}
\nabla_{\vT}
\left[
\frac{\lambda}{s}
\|\vT\|^2\log\bigl(1+\|\vT\|^2\bigr)
\right]
=
\frac{2\lambda}{s}\vT
\left[
\log\bigl(1+\|\vT\|^2\bigr)
+
\frac{\|\vT\|^2}{1+\|\vT\|^2}
\right].
\end{equation}
Differentiating once more gives
\begin{equation}
\label{eq:bpc-regularizer-hessian}
\begin{aligned}
\nabla_{\vT}^2
\left[
\frac{\lambda}{s}
\|\vT\|^2\log\bigl(1+\|\vT\|^2\bigr)
\right]
={}&
\frac{2\lambda}{s}
\left[
\log\bigl(1+\|\vT\|^2\bigr)
+
\frac{\|\vT\|^2}{1+\|\vT\|^2}
\right]I_D + \frac{4\lambda}{s} \frac{\|\vT\|^2+2} {\bigl(1+\|\vT\|^2\bigr)^2}
\vT\vT^\top.
\end{aligned}
\end{equation}
Both terms on the right-hand side of
\eqref{eq:bpc-regularizer-hessian} are positive semidefinite. Indeed, $\log\bigl(1+\|\vT\|^2\bigr) + \frac{\|\vT\|^2}{1+\|\vT\|^2} \geq 0$ and $\frac{\|\vT\|^2+2}{\bigl(1+\|\vT\|^2\bigr)^2}>0.$ Consequently,
\[
\nabla_{\vT}^2
\left[
\frac{\lambda}{s}
\|\vT\|^2\log\bigl(1+\|\vT\|^2\bigr)
\right]
\succeq 0.
\]
Thus $\nu_s$ is log-concave. Moreover, its density depends on $\vT$ only
through $\|\vT\|$, so it is spherically symmetric.

Since $D\geq 2$, the spectral-gap estimate for spherically
symmetric log-concave probability measures is given by Theorem~1.2 of~\cite{bonnefont2016spectral}:
\begin{equation}
\label{eq:bpc-bjm-poincare-bound}
C_{\mathrm{PI}}^\star(\nu_s)
\leq
\frac{
\displaystyle
\int_{\R^D}\|\vT\|^2\,\nu_s(d\vT)
}{
D-1
}.
\end{equation}
It remains to bound the second moment under $\nu_s$.

\medskip
\quad \textbf{Step 2: explicit bound on the second moment of $\nu_s$.}

Because
\[
\exp\left(
-\frac{\lambda}{s}
\|\vT\|^2\log\bigl(1+\|\vT\|^2\bigr)
\right)
\]
decays faster than a Gaussian at infinity, the boundary term in the
following integration by parts vanishes. Therefore,
\begin{equation}
\label{eq:bpc-integration-by-parts}
\begin{aligned}
0
={}&
\int_{\R^D}
\nabla_{\vT}\cdot
\left[
\vT
\exp\left(
-\frac{\lambda}{s}
\|\vT\|^2\log\bigl(1+\|\vT\|^2\bigr)
\right)
\right]d\vT.
\end{aligned}
\end{equation}
Expanding the divergence in \eqref{eq:bpc-integration-by-parts} gives
\begin{equation}
\label{eq:bpc-expanded-divergence}
\begin{aligned}
0
={}&
\int_{\R^D}
\Bigg[
D
-
\vT^\top
\nabla_{\vT}
\left(
\frac{\lambda}{s}
\|\vT\|^2\log\bigl(1+\|\vT\|^2\bigr)
\right)
\Bigg]
\times
\exp\left(
-\frac{\lambda}{s}
\|\vT\|^2\log\bigl(1+\|\vT\|^2\bigr)
\right)d\vT.
\end{aligned}
\end{equation}
Using \eqref{eq:bpc-regularizer-gradient},
\begin{equation}
\label{eq:bpc-radial-derivative}
\begin{aligned}
\vT^\top
\nabla_{\vT}
\left[
\frac{\lambda}{s}
\|\vT\|^2\log\bigl(1+\|\vT\|^2\bigr)
\right]
=
\frac{2\lambda}{s}
\|\vT\|^2
\left[
\log\bigl(1+\|\vT\|^2\bigr)
+
\frac{\|\vT\|^2}{1+\|\vT\|^2}
\right].
\end{aligned}
\end{equation}
Substituting \eqref{eq:bpc-radial-derivative} into
\eqref{eq:bpc-expanded-divergence} and dividing by the
normalization constant of $\nu_s$, we obtain the exact identity
\begin{equation}
\label{eq:bpc-virial-identity}
D
=
\frac{2\lambda}{s}
\int_{\R^D}
\|\vT\|^2
\left[
\log\bigl(1+\|\vT\|^2\bigr)
+
\frac{\|\vT\|^2}{1+\|\vT\|^2}
\right]
\nu_s(d\vT).
\end{equation}
Since
\[
\frac{\|\vT\|^2}{1+\|\vT\|^2}\geq 0,
\]
it follows from \eqref{eq:bpc-virial-identity} that
\begin{equation}
\label{eq:bpc-logarithmic-moment-bound}
\int_{\R^D}
\|\vT\|^2
\log\bigl(1+\|\vT\|^2\bigr)
\nu_s(d\vT)
\leq
\frac{Ds}{2\lambda}.
\end{equation}

We now split the second moment into the regions $\|\vT\|<1$ and
$\|\vT\|\geq 1$:
\begin{equation}
\label{eq:bpc-second-moment-decomposition}
\begin{aligned}
\int_{\R^D}\|\vT\|^2\,\nu_s(d\vT)
={}&
\int_{\{\|\vT\|<1\}}\|\vT\|^2\,\nu_s(d\vT)
+
\int_{\{\|\vT\|\geq 1\}}\|\vT\|^2\,\nu_s(d\vT).
\end{aligned}
\end{equation}
For $\{\|\vT\|<1\}$, one has $\|\vT\|^2\leq 1$, and therefore
\begin{equation}
\label{eq:bpc-unit-ball-moment-bound}
\int_{\{\|\vT\|<1\}}\|\vT\|^2\,\nu_s(d\vT)
\leq 1.
\end{equation}
For $\{\|\vT\|\geq 1\}$,
\[
\log\bigl(1+\|\vT\|^2\bigr)\geq \log 2,
\]
and hence
\[
\|\vT\|^2
\leq
\frac{
\|\vT\|^2\log\bigl(1+\|\vT\|^2\bigr)
}{
\log 2
}.
\]
Consequently,
\begin{equation}
\label{eq:bpc-outside-ball-moment-bound}
\begin{aligned}
\int_{\{\|\vT\|\geq 1\}}\|\vT\|^2\,\nu_s(d\vT)
&\leq
\frac{1}{\log 2}
\int_{\R^D}
\|\vT\|^2\log\bigl(1+\|\vT\|^2\bigr)
\nu_s(d\vT)
\leq
\frac{Ds}{2\lambda\log 2},
\end{aligned}
\end{equation}
where the last inequality follows from
\eqref{eq:bpc-logarithmic-moment-bound}. 
Equations~\eqref{eq:bpc-second-moment-decomposition},
\eqref{eq:bpc-unit-ball-moment-bound}, and
\eqref{eq:bpc-outside-ball-moment-bound} give
\begin{equation}
\label{eq:bpc-second-moment-bound}
\begin{aligned}
\int_{\R^D}\|\vT\|^2\,\nu_s(d\vT)
=&
\int_{\{\|\vT\|<1\}}
\|\vT\|^2\,\nu_s(d\vT)
+
\int_{\{\|\vT\|\geq1\}}
\|\vT\|^2\,\nu_s(d\vT)
\leq
1+\frac{Ds}{2\lambda\log 2}.
\end{aligned}
\end{equation}
Therefore, Equations~\eqref{eq:bpc-bjm-poincare-bound} and
\eqref{eq:bpc-second-moment-bound} give
\begin{equation}
\label{eq:bpc-reference-poincare-bound}
C_{\mathrm{PI}}^\star(\nu_s)
\leq
\frac{
1+\dfrac{Ds}{2\lambda\log 2}
}{
D-1
}.
\end{equation}

Since $D\ge 2$,
\begin{align}
\frac{1+\frac{Ds}{2\lambda\log 2}}{D-1}
&=
\frac{1}{D-1}
+
\frac{D}{D-1}\frac{s}{2\lambda\log 2}
\le
1+\frac{s}{\lambda\log 2}.
\end{align}
Therefore,
\begin{equation}
C_{\mathrm{PI}}^\star(\nu_s)
\le
\left(
1+\frac{s}{\lambda\log 2}
\right).
\end{equation}

\quad \textbf{Step 3: Transfer of the Poincar\'e inequality from $\nu_s$ to
$\mu_s$.}

We recall
\begin{equation}
\label{eq:bpc-general-oscillation-assumption}
\operatorname{osc}(\widehat R) \coloneqq \sup_{\vT\in\R^D}\widehat R(\vT) - \inf_{\vT\in\R^D}\widehat R(\vT) \leq M_{\widehat R} < \infty.
\end{equation}
From \eqref{eq:bpc-general-oscillation-assumption} we get,
\begin{equation}
\label{eq:bpc-data-oscillation-bound}
\begin{aligned}
\frac{2}{s}
\left(
\sup_{\vT\in\R^D}\widehat R(\vT)
-
\inf_{\vT\in\R^D}\widehat R(\vT)
\right)
&\leq
\frac{2M_{\widehat R}}{s}.
\end{aligned}
\end{equation}
Since $\widehat R$ is bounded,
$\exp(-2\widehat R/s)$ is bounded above and below by positive
constants. Together with the finiteness of the normalization integral
of $\nu_s$, this also shows that $0<Z_s<\infty$.

By the definitions of $\mu_s$, $\nu_s$, and
$\tilde V_{\widehat R}$,
\begin{equation}
\label{eq:bpc-radon-nikodym-derivative}
\frac{d\mu_s}{d\nu_s}(\vT)
=
\frac{
\exp\left(-\dfrac{2}{s}\widehat R(\vT)\right)
}{
\displaystyle
\int_{\R^D}
\exp\left(-\dfrac{2}{s}\widehat R(\vu)\right)
\nu_s(d\vu)
}.
\end{equation}
Note that the denominator in \eqref{eq:bpc-radon-nikodym-derivative} is a normalization integral. The normalization factor in \eqref{eq:bpc-radon-nikodym-derivative} is independent of $\vT$ and therefore cancels in the ratio. Hence,
\begin{equation}
\label{eq:bpc-density-ratio}
\begin{aligned}
\frac{
\displaystyle
\sup_{\vT\in\R^D}
\frac{d\mu_s}{d\nu_s}(\vT)
}{
\displaystyle
\inf_{\vT\in\R^D}
\frac{d\mu_s}{d\nu_s}(\vT)
}
&=
\frac{
\displaystyle
\exp\left(
-\frac{2}{s}
\inf_{\vT\in\R^D}\widehat R(\vT)
\right)
}{
\displaystyle
\exp\left(
-\frac{2}{s}
\sup_{\vT\in\R^D}\widehat R(\vT)
\right)
}
=
\exp\left[
\frac{2}{s}
\left(
\sup_{\vT\in\R^D}\widehat R(\vT)
-
\inf_{\vT\in\R^D}\widehat R(\vT)
\right)
\right]
\\
&\leq
\exp\left(
\frac{2M_{\widehat R}}{s}
\right),
\end{aligned}
\end{equation}
where the last inequality follows from
\eqref{eq:bpc-general-oscillation-assumption}.
Using the variational
characterization of the variance,
\begin{equation}
\label{eq:bpc-variance-variational-characterization}
\Var_{\mu_s}(h)
=
\inf_{c\in\R}
\int_{\R^D}(h(\vT)-c)^2\,\mu_s(d\vT).
\end{equation}
Hence, using
\[
\mu_s(d\vT)
=
\frac{d\mu_s}{d\nu_s}(\vT)\,\nu_s(d\vT),
\]
we obtain
\begin{equation}
\label{eq:bpc-variance-comparison}
\begin{aligned}
\Var_{\mu_s}(h)
&=
\inf_{c\in\R}
\int_{\R^D}
(h(\vT)-c)^2\,\mu_s(d\vT) = \inf_{c\in\R} \int_{\R^D} (h(\vT)-c)^2 \frac{d\mu_s}{d\nu_s}(\vT)\, \nu_s(d\vT)
\\
&\leq
\sup_{\vT\in\R^D}
\frac{d\mu_s}{d\nu_s}(\vT)
\inf_{c\in\R}
\int_{\R^D}
(h(\vT)-c)^2\,\nu_s(d\vT)
=
\sup_{\vT\in\R^D}
\frac{d\mu_s}{d\nu_s}(\vT)
\Var_{\nu_s}(h).
\end{aligned}
\end{equation}
Applying the Poincar\'e inequality for $\nu_s$ to
\eqref{eq:bpc-variance-comparison} gives
\begin{equation}
\label{eq:bpc-apply-reference-poincare}
\Var_{\mu_s}(h)
\leq
\sup_{\vT\in\R^D}
\frac{d\mu_s}{d\nu_s}(\vT)
C_{\mathrm{PI}}^\star(\nu_s)
\int_{\R^D}
\|\nabla h(\vT)\|^2\,\nu_s(d\vT).
\end{equation}
Moreover,
\begin{equation}
\label{eq:bpc-dirichlet-form-comparison}
\begin{aligned}
\int_{\R^D}
\|\nabla h(\vT)\|^2\,\nu_s(d\vT)
&=
\int_{\R^D}
\|\nabla h(\vT)\|^2
\left(
\frac{d\mu_s}{d\nu_s}(\vT)
\right)^{-1}
\mu_s(d\vT)
\leq
\left[
\inf_{\vT\in\R^D}
\frac{d\mu_s}{d\nu_s}(\vT)
\right]^{-1}
\int_{\R^D}
\|\nabla h(\vT)\|^2\,\mu_s(d\vT).
\end{aligned}
\end{equation}
Combining Equations~\eqref{eq:bpc-apply-reference-poincare} and
\eqref{eq:bpc-dirichlet-form-comparison} yields
\begin{equation}
\label{eq:bpc-perturbation-transfer}
\begin{aligned}
\Var_{\mu_s}(h)
\leq{}&
\frac{
\displaystyle
\sup_{\vT\in\R^D}\frac{d\mu_s}{d\nu_s}(\vT)
}{
\displaystyle
\inf_{\vT\in\R^D}\frac{d\mu_s}{d\nu_s}(\vT)
}
C_{\mathrm{PI}}^\star(\nu_s)
\int_{\R^D}
\|\nabla h(\vT)\|^2\,\mu_s(d\vT)
\\
={}&
\exp\left[
\frac{2}{s}
\left(
\sup_{\vT\in\R^D}\widehat R(\vT)
-
\inf_{\vT\in\R^D}\widehat R(\vT)
\right)
\right]
C_{\mathrm{PI}}^\star(\nu_s)
\times
\int_{\R^D}
\|\nabla h(\vT)\|^2\,\mu_s(d\vT).
\end{aligned}
\end{equation}
\eqref{eq:bpc-perturbation-transfer} shows that the transfer from
$\nu_s$ to $\mu_s$ contributes the density-ratio factor in
\eqref{eq:bpc-density-ratio}, while the remaining factor is the
Poincar\'e constant of the regularizer-only measure $\nu_s$. \eqref{eq:bpc-data-oscillation-bound} bounds the former, and
\eqref{eq:bpc-reference-poincare-bound} bounds the latter.
\begin{equation}
\label{eq:bpc-final-variance-bound}
\begin{aligned}
\Var_{\mu_s}(h)
\leq{}&
\exp\left(
\frac{2M_{\widehat R}}{s}
\right)
\left(
1+\frac{s}{\lambda\log 2}
\right)
\times
\int_{\R^D}
\|\nabla h(\vT)\|^2\,\mu_s(d\vT).
\end{aligned}
\end{equation}
By \eqref{eq:bpc-optimal-poincare-constant},
\eqref{eq:bpc-final-variance-bound} implies
\begin{equation}
\label{eq:bpc-proof-poincare-bound}
C_{\mathrm{PI}}^\star(\mu_s)
\leq
\exp\left(
\frac{2M_{\widehat R}}{s}
\right)
\left(
1+\frac{s}{\lambda\log 2}
\right)
<\infty.
\end{equation}

% Similarly, for the polynomial profile $\phi_\epsilon(a)=a^{2+\epsilon}$ with any $\epsilon>0$, we have
% \[
% C_{\mathrm{PI}}^\star(\mu_s)
% \leq
% \exp\!\left(\frac{2M_{\widehat R}}{s}\right)
% \frac{1}{D-1}
% \left(
% \frac{Ds}{\lambda(2+\epsilon)}
% \right)^{\frac{2}{2+\epsilon}}
% <\infty.
% \]

% \eqref{eq:bpc-proof-poincare-bound} is
% \eqref{eq:bpc-main-poincare-bound}, and therefore, by
% \eqref{eq:bpc-rate-definition},
% \begin{equation}
% \label{eq:bpc-proof-rate-bound}
% \lambda_s
% \geq
% \frac{s(D-1)}{2}
% \exp\left(
% -\frac{2M_{\widehat R}}{s}
% \right)
% \left(
% 1+\frac{Ds}{2\lambda\log 2}
% \right)^{-1}.
% \end{equation}
% \eqref{eq:bpc-proof-rate-bound} is
% \eqref{eq:bpc-main-rate-bound}. This proves the result.
\end{proof}

\clearpage

%\section{$r^2h(r)$ for multi-Q-K}

\section{General Head-Wise Radial Regularization for Multi-Head Key--Query Attention}
\label{app:multihead-general-profile}

\subsection{Proof of Intermediate Lemmas Needed for Theorem \ref{thm:mh-general-profile-villani}}\label{multiQK-Lemmas}

\newcommand{\dbracket}[1]{[\![#1]\!]}

\begin{remark}\label{rem:notationlem1}
In Lemma~\ref{lem:mh-rowsoftmax-bounds} we have used the notations $\dd S(\mM)[H]$ and $\dd^2 S(\mM)[\mH,\mK]$, as defined in equations~\ref{eq:second-derivative-bracket}. The slots inside $[\,\cdot\,,\cdot\,]$ carry \emph{directions} of perturbation, not arguments of $S$; the output of \eqref{eq:second-derivative-bracket} is again an element of $\mathcal{M}$. If $S$ is $C^2$ near $\mM$, Schwarz's theorem \cite[Theorem~9.41]{rudin1976principles} gives $\dd^2S(\mM)[\mH,\mK]=\dd^2S(\mM)[\mK,\mH]$, 

% \note{verify what is this Schawarz theorem and if it at all says this and reference it}

and the off-diagonal values are determined by the diagonal ones through polarization, $\dd^2S(\mM)[\mH,\mK]=\tfrac12\Bigl(\dd^2S(\mM)[\mH+\mK,\mH+\mK]-\dd^2S(\mM)[\mH,\mH]-\dd^2S(\mM)[\mK,\mK]\Bigr)$. In differential shorthand one writes $dS=\dd S(\mM)[d\mM]$ and $\dd^2S=\dd^2S(\mM)[d\mM,d\mM]$, i.e.\ the same increment is placed in both slots.

\qquad {\bf Second-Order Taylor Expansion.}
With this notation the expansion of $S$ about $\mM$ reads
\begin{equation}
S(\mM+\mH)\;=\;S(\mM)\;+\;\dd S(\mM)[\mH]\;+\;\tfrac{1}{2}\,\dd^2S(\mM)[\mH,\mH]
\;+\;\mR_2(\mM,\mH),
\qquad
\frac{\lVert \mR_2(\mM,\mH)\rVert}{\lVert \mH\rVert^{2}}\xrightarrow[\mH\to\mzero]{}0 .
\label{eq:taylor-bracket}
\end{equation}
Defining the induced operator norm
$\lVert \dd^2S(\mM)\rVert_{\mathrm{op}}:=\sup_{\lVert \mH\rVert=\lVert \mK\rVert=1}
\lVert \dd^2S(\mM)[\mH,\mK]\rVert$, the remainder obeys the quantitative bound
\begin{equation}
\bigl\lVert S(\mM+\mH)-S(\mM)-\dd S(\mM)[\mH]\bigr\rVert
\;\le\;
\tfrac{1}{2}\sup_{\theta\in[0,1]}\bigl\lVert \dd^2S(\mM+\theta\mH)\bigr\rVert_{\mathrm{op}}
\,\lVert \mH\rVert^{2}.
\label{eq:taylor-remainder-bound}
\end{equation}

\qquad {\bf Relation to the Six-Index Hessian Array.}
Write $S(\mM)=\bigl(S_{ij}(\mM)\bigr)_{i\in[t],j\in[d]}$ and define the array
\begin{equation}
\mathcal{S}_{ij,\,kl,\,pq}(\mM)\;:=\;
\frac{\partial^{2}S_{ij}(\mM)}{\partial \mM_{kl}\,\partial \mM_{pq}},
\qquad i,k,p\in[t],\; j,l,q\in[d],
\label{eq:six-index-array}
\end{equation}
which has $(td)^{3}$ entries and is symmetric under the exchange of the index
pairs $(kl)\leftrightarrow(pq)$. The bracket
\eqref{eq:second-derivative-bracket} is exactly the contraction of
\eqref{eq:six-index-array} against the two direction matrices in the two free
pairs of slots,
\begin{equation}
\bigl(\dd^{2}S(\mM)[\mH,\mK]\bigr)_{ij}
\;=\;
\sum_{k=1}^{t}\sum_{l=1}^{d}\sum_{p=1}^{t}\sum_{q=1}^{d}
\mathcal{S}_{ij,\,kl,\,pq}(\mM)\,\mH_{kl}\,\mK_{pq},
\label{eq:bracket-as-contraction}
\end{equation}
and in particular $\dd^{2}S_{ij}=\sum_{k,l,p,q}\mathcal{S}_{ij,kl,pq}(\mM)\,
(d\mM)_{kl}(d\mM)_{pq}$. Equivalently, after we vectorize using an operator $\vecmap :\mathcal{M}\to\mathbb{R}^{td}$,  each output coordinate carries a genuine Hessian matrix $\mHess_{ij}(\mM)\in\mathbb{R}^{td\times td}$ with entries $\bigl(\mHess_{ij}\bigr)_{(kl),(pq)}=\mathcal{S}_{ij,kl,pq}(\mM)$, so that \eqref{eq:bracket-as-contraction} becomes the quadratic form $$\bigl(\dd^{2}S(\mM)[\mH,\mK]\bigr)_{ij}=\vecmap(\mH)^{\top}\mHess_{ij}(\mM)\vecmap(\mK).$$
The bracket notation is preferred precisely because it avoids ever assembling the $td\times(td)^{2}$ object obtained by stacking the $\mHess_{ij}$.
\end{remark}

\begin{proof}[Proof of Lemma~\ref{lem:mh-rowsoftmax-bounds}]~\label{prooflem:mh-rowsoftmax-bounds}
For each row $p$, the entries $\mS_{pq}$ are nonnegative and satisfy
$\sum_{q=1}^t\mS_{pq}=1$. Therefore,
\begin{equation}
\lVert\mS\rVert_F^2
=
\sum_{p=1}^t\sum_{q=1}^t\mS_{pq}^2
\le
\sum_{p=1}^t
\left(\sum_{q=1}^t\mS_{pq}\right)^2
=t,
\label{eq:mh-softmax-output-proof}
\end{equation}
which proves \eqref{eq:mh-softmax-output-bound}.

Because the row-softmax acts independently on each row, derivatives with respect to
entries in different rows vanish. Within a fixed row $p$, direct differentiation of
\eqref{eq:mh-softmax-entry-definition} gives
\begin{equation}
\frac{\partial\mS_{pq}}{\partial\mM_{pk}}
=
\beta\mS_{pq}\bigl(\delta_{qk}-\mS_{pk}\bigr).
\label{eq:mh-softmax-first-partial}
\end{equation}
Let $\boldsymbol{s}_p=(\mS_{p1},\ldots,\mS_{pt})^\top$. The Jacobian acting on the
$p$-th row is
\begin{equation}
\mathcal J_p
=
\beta\left(
\operatorname{diag}(\boldsymbol{s}_p)
-
\boldsymbol{s}_p\boldsymbol{s}_p^\top
\right).
\label{eq:mh-softmax-row-jacobian}
\end{equation}
Since $\boldsymbol{s}_p$ is a probability vector,
\begin{equation}
\lVert\mathcal J_p\rVert_2
\le
\beta\left(
\lVert\operatorname{diag}(\boldsymbol{s}_p)\rVert_2
+
\lVert\boldsymbol{s}_p\boldsymbol{s}_p^\top\rVert_2
\right)
\le2\beta.
\label{eq:mh-softmax-row-jacobian-bound}
\end{equation}
Writing $\dd\mM_{p,:}$ and $\dd\mS_{p,:}$ for the corresponding row
perturbations and summing \eqref{eq:mh-softmax-row-jacobian-bound} over the rows yields
\begin{equation}
\lVert\dd\mS\rVert_F^2
=
\sum_{p=1}^t\lVert\dd\mS_{p,:}\rVert_2^2
\le
4\beta^2
\sum_{p=1}^t\lVert\dd\mM_{p,:}\rVert_2^2
=
4\beta^2\lVert\dd\mM\rVert_F^2.
\label{eq:mh-softmax-jacobian-proof}
\end{equation}
Taking square roots proves \eqref{eq:mh-softmax-jacobian-bound} with $B_{s'}=2$.

Differentiating \eqref{eq:mh-softmax-first-partial} once more within the same row gives
\begin{align}
\frac{\partial^2\mS_{pq}}
{\partial\mM_{pk}\,\partial\mM_{p\ell}}
&=
\beta^2\mS_{pq}
\bigl(
2\mS_{pk}\mS_{p\ell}
+\delta_{qk}\delta_{q\ell}
-\delta_{k\ell}\mS_{pk}
-\delta_{qk}\mS_{p\ell}
-\delta_{q\ell}\mS_{pk}
\bigr).
\label{eq:mh-softmax-second-partial}
\end{align}
Each probability lies in $[0,1]$ and each Kronecker delta lies in $\{0,1\}$, so every
entry in \eqref{eq:mh-softmax-second-partial} has absolute value at most $6\beta^2$.
All cross-row second derivatives vanish. Hence the full Hessian tensor has at most
$t^4$ nonzero entries and satisfies
\begin{equation}
\left\lVert
\dd^2\operatorname{RowSoftMax}_\beta(\mM)
\right\rVert_F
\le
6\beta^2t^2.
\label{eq:mh-softmax-hessian-tensor-bound}
\end{equation}
Finally, the Frobenius Cauchy--Schwarz inequality for the tensor contraction gives
\begin{align}
\lVert\dd^2\mS\rVert_F
&=
\left\lVert
\dd^2\operatorname{RowSoftMax}_\beta(\mM)
[\dd\mM,\dd\mM]
\right\rVert_F
\le
\left\lVert
\dd^2\operatorname{RowSoftMax}_\beta(\mM)
\right\rVert_F
\lVert\dd\mM\otimes\dd\mM\rVert_F
\notag\\
&\le
6\beta^2t^2\lVert\dd\mM\rVert_F^2.
\label{eq:mh-softmax-hessian-proof}
\end{align}
This proves \eqref{eq:mh-softmax-hessian-bound} with $B_{s''}=6t^2$.
\end{proof}

\begin{proof}[Proof of Lemma~\ref{lem:mh-data-bounds}]~\label{prooflem:mh-data-bounds}
For each sample $i$, define the residual
\begin{equation}
\mE_i
\coloneqq
\widehat{\mY}_i(\vT)-\mY_i.
\label{eq:mh-data-residual}
\end{equation}
By \eqref{eq:mh-output}, Lemma~\ref{lem:mh-rowsoftmax-bounds}, and the assumed norm
bounds,
\begin{align}
\lVert\mE_i\rVert_F
&\le
\left\lVert
\left(\sum_{j=1}^H\mS_i^{(j)}\right)
\mX_i\mW_v
\right\rVert_F
+
\lVert\mY_i\rVert_F
\le
\sum_{j=1}^H
\lVert\mS_i^{(j)}\rVert_F
\lVert\mX_i\rVert_F
\lVert\mW_v\rVert_F
+B_y
\le
H\sqrt t\,B_xB_w+B_y
\eqqcolon E_H.
\label{eq:mh-data-residual-bound}
\end{align}

We first prove the gradient bound. Since
$\ell_i=\frac12\lVert\mE_i\rVert_F^2$, varying only the $j$-th softmax matrix gives
\begin{align}
\dd_{\mS_i^{(j)}}\ell_i
&=
\left\langle
\mE_i,
(\dd\mS_i^{(j)})\mX_i\mW_v
\right\rangle_F
=
\left\langle
\mE_i\mW_v^\top\mX_i^\top,
\dd\mS_i^{(j)}
\right\rangle_F.
\label{eq:mh-data-upstream-differential}
\end{align}
Thus,
\begin{equation}
\nabla_{\mS_i^{(j)}}\ell_i
=
\mE_i\mW_v^\top\mX_i^\top,
\qquad
\left\lVert\nabla_{\mS_i^{(j)}}\ell_i\right\rVert_F
\le
E_HB_wB_x.
\label{eq:mh-data-upstream-gradient}
\end{equation}
Applying the Jacobian bound \eqref{eq:mh-softmax-jacobian-bound} to
\eqref{eq:mh-softmax-matrix} gives
\begin{equation}
\left\lVert\nabla_{\mM_i^{(j)}}\ell_i\right\rVert_F
\le
E_HB_wB_x\beta B_{s'}.
\label{eq:mh-data-intermediate-gradient}
\end{equation}

Differentiating the score matrix \eqref{eq:mh-score-matrix} and using the Frobenius
gradient identity yields
\begin{align}
\nabla_{\mW_Q^{(j)}}\ell_i
=
\frac1{\sqrt d}
\mX_i^\top
\left(\nabla_{\mM_i^{(j)}}\ell_i\right)
\mX_i\mW_K^{(j)},\quad
\nabla_{\mW_K^{(j)}}\ell_i
=
\frac1{\sqrt d}
\mX_i^\top
\left(\nabla_{\mM_i^{(j)}}\ell_i\right)^\top
\mX_i\mW_Q^{(j)}.
\label{eq:mh-data-parameter-gradients}
\end{align}
Combining \eqref{eq:mh-data-intermediate-gradient} and
\eqref{eq:mh-data-parameter-gradients} with $\lVert\mX_i\rVert_F\le B_x$ gives
\begin{align}
\left\lVert\nabla_{\mW_Q^{(j)}}\ell_i\right\rVert_F
&\le
C_{\nabla,H}\lVert\mW_K^{(j)}\rVert_F,
\notag\\
\left\lVert\nabla_{\mW_K^{(j)}}\ell_i\right\rVert_F
&\le
C_{\nabla,H}\lVert\mW_Q^{(j)}\rVert_F.
\label{eq:mh-data-parameter-gradient-bounds}
\end{align}
Therefore, by \eqref{eq:mh-radii},
\begin{align}
\lVert\nabla_{\vT}\ell_i(\vT)\rVert^2
&=
\sum_{j=1}^H
\left(
\left\lVert\nabla_{\mW_Q^{(j)}}\ell_i\right\rVert_F^2
+
\left\lVert\nabla_{\mW_K^{(j)}}\ell_i\right\rVert_F^2
\right)
\le
C_{\nabla,H}^2\lVert\vT\rVert^2.
\label{eq:mh-data-sample-gradient-bound}
\end{align}
Averaging over $i$ in Definition~\ref{def:mh-general-profile-potential} and applying the triangle inequality proves
\eqref{eq:mh-data-gradient-bound}.

It remains to prove the Laplacian bound. For any scalar coordinate $\theta$ of
$\mW_Q^{(j)}$ or $\mW_K^{(j)}$, differentiating the sample loss twice gives
\begin{equation}
\frac{\partial^2\ell_i}{\partial\theta^2}
=
\left\lVert
\frac{\partial\widehat{\mY}_i}{\partial\theta}
\right\rVert_F^2
+
\left\langle
\mE_i,
\frac{\partial^2\widehat{\mY}_i}{\partial\theta^2}
\right\rangle_F.
\label{eq:mh-data-loss-second-derivative}
\end{equation}
For $\theta=(\mW_Q^{(j)})_{ab}$,
\begin{equation}
\frac{\partial\mM_i^{(j)}}{\partial(\mW_Q^{(j)})_{ab}}
=
\frac1{\sqrt d}
\mX_i\ve_a\ve_b^\top
(\mW_K^{(j)})^\top\mX_i^\top.
\label{eq:mh-data-score-coordinate-derivative}
\end{equation}
Using the Frobenius norm of a rank-one matrix and summing over the coordinates gives
\begin{align}
&\sum_{a=1}^d\sum_{b=1}^r
\left\lVert
\frac{\partial\mM_i^{(j)}}{\partial(\mW_Q^{(j)})_{ab}}
\right\rVert_F^2
=
\frac1d
\sum_{a=1}^d\lVert\mX_i\ve_a\rVert^2
\sum_{b=1}^r\lVert\mX_i\mW_K^{(j)}\ve_b\rVert^2
=
\frac1d
\lVert\mX_i\rVert_F^2
\lVert\mX_i\mW_K^{(j)}\rVert_F^2
\le
\frac{B_x^4}{d}\lVert\mW_K^{(j)}\rVert_F^2.
\label{eq:mh-data-score-coordinate-sum}
\end{align}
The score matrix is linear in $\mW_Q^{(j)}$, so its second derivative with respect to
each coordinate is zero. Applying the two differential bounds from
Lemma~\ref{lem:mh-rowsoftmax-bounds} to \eqref{eq:mh-output} and using
\eqref{eq:mh-data-score-coordinate-sum} yields
\begin{align}
&\sum_{a=1}^d\sum_{b=1}^r
\left\lVert
\frac{\partial\widehat{\mY}_i}
{\partial(\mW_Q^{(j)})_{ab}}
\right\rVert_F^2
\le
(\beta B_{s'}B_xB_w)^2
\frac{B_x^4}{d}
\lVert\mW_K^{(j)}\rVert_F^2,
\label{eq:mh-data-output-first-derivative-sum}\\
&\sum_{a=1}^d\sum_{b=1}^r
\left\lVert
\frac{\partial^2\widehat{\mY}_i}
{\partial(\mW_Q^{(j)})_{ab}^2}
\right\rVert_F
\le
B_xB_w\beta^2B_{s''}
\frac{B_x^4}{d}
\lVert\mW_K^{(j)}\rVert_F^2.
\label{eq:mh-data-output-second-derivative-sum}
\end{align}
Equations~\eqref{eq:mh-data-residual-bound},
\eqref{eq:mh-data-loss-second-derivative},
\eqref{eq:mh-data-output-first-derivative-sum}, and
\eqref{eq:mh-data-output-second-derivative-sum} imply
\begin{equation}
\left|\Delta_{\mW_Q^{(j)}}\ell_i\right|
\le
C_{\Delta,H}\lVert\mW_K^{(j)}\rVert_F^2.
\label{eq:mh-data-WQ-laplacian-bound}
\end{equation}
The same calculation with $\mW_Q^{(j)}$ and $\mW_K^{(j)}$ exchanged gives
\begin{equation}
\left|\Delta_{\mW_K^{(j)}}\ell_i\right|
\le
C_{\Delta,H}\lVert\mW_Q^{(j)}\rVert_F^2.
\label{eq:mh-data-WK-laplacian-bound}
\end{equation}
Summing these two bounds over the heads and using \eqref{eq:mh-radii} gives
\begin{equation}
\left|\Delta_{\vT}\ell_i(\vT)\right|
\le
C_{\Delta,H}\lVert\vT\rVert^2.
\label{eq:mh-data-sample-laplacian-bound}
\end{equation}
Finally, averaging over $i$ in Definition~\ref{def:mh-general-profile-potential} and applying the triangle inequality
proves \eqref{eq:mh-data-laplacian-bound}.
\end{proof}

\subsection{Example Classes of Admissible Profiles}
\label{subsec:mh-generated-profiles}

\begin{corollary}[{\bf Logarithmic Multi-Head Regularizer}]
\label{cor:mh-log-profile-standalone}
Let $\phi_{\log}(a) = a^2\log(1+a^2)$ as in Remark~\ref{rem:mh-admissible-examples}. Then $\phi_{\log}$ satisfies Assumption~\ref{ass:mh-admissible-profile}, and the potential
\begin{align}
\widetilde V_{\log,H,\mathrm{ATT}}(\vT)
&=
\widehat R_{A,H}(\vT)
+
\frac{\lambda}{2H}\sum_{j=1}^H \lVert\vT_j\rVert^2\log\!\left(1+\lVert\vT_j\rVert^2\right)
\notag\\
&=
\widehat R_{A,H}(\vT)
+
\frac{\lambda}{2H} \sum_{j=1}^H \Bigl( \lVert\mW_Q^{(j)}\rVert_F^2 + \lVert\mW_K^{(j)}\rVert_F^2 \Bigr)
\log\!\Bigl(1 + \lVert\mW_Q^{(j)}\rVert_F^2 + \lVert\mW_K^{(j)}\rVert_F^2 \Bigr)
\label{eq:mh-log-recovers-standalone}
\end{align}
satisfies the Villani condition of Definition~\ref{def:villani}.
\end{corollary}

\begin{proof}
Let $\phi_{\log}(a)=a^2\log(1+a^2)$, and define
\begin{equation}
g_{\log}(a)
\coloneqq
\frac{\phi_{\log}'(a)}{a}
=
2\log(1+a^2)+\frac{2a^2}{1+a^2},
\qquad a>0,
\label{eq:mh-log-g-standalone}
\end{equation}
so that $\phi_{\log}(a)=\int_0^a v\,g_{\log}(v)\,dv$ and, correspondingly,
\begin{equation}
\phi_{\log}'(a)=a\,g_{\log}(a),
\qquad
\phi_{\log}''(a)=g_{\log}(a)+a\,g_{\log}'(a),
\qquad
g_{\log}'(a)
=
\frac{4a}{1+a^2}
+
\frac{4a}{(1+a^2)^2}.
\label{eq:mh-log-derivs-standalone}
\end{equation}
Direct computation gives $\phi_{\log}'(a)=2a\log(1+a^2)+\frac{2a^3}{1+a^2}$, which is continuous on $[0,\infty)$ with $\phi_{\log}'(0)=0$, and $\phi_{\log}''$ (obtained from \eqref{eq:mh-log-derivs-standalone}) is likewise continuous on $[0,\infty)$. Hence $\phi_{\log}\in C^2([0,\infty))$ with $\phi_{\log}'(0)=0$, which is (P1) of Assumption~\ref{ass:mh-admissible-profile}.

Since $2\log(1+a^2)\to\infty$ while $\frac{2a^2}{1+a^2}\to2$ remains bounded,
\begin{equation}
\frac{\phi_{\log}'(a)}{a}=g_{\log}(a)\longrightarrow+\infty
\qquad\text{as }a\to\infty,
\label{eq:mh-log-P2-standalone}
\end{equation}
which is exactly (P2) of Assumption~\ref{ass:mh-admissible-profile}.

We first bound the growth of $g_{\log}'$ relative to $g_{\log}$. For $a\ge1$,
\begin{equation}
a\,g_{\log}'(a)
=
\frac{4a^2}{1+a^2}+\frac{4a^2}{(1+a^2)^2}
\le
4+4=8,
\label{eq:mh-log-g-growth-standalone}
\end{equation}
while $g_{\log}(a)\ge2\log2>0$ for $a\ge1$, so choosing $C_g\coloneqq 8/(2\log2)$ gives
\begin{equation}
a\,g_{\log}'(a)\le C_g\,g_{\log}(a)
\qquad\text{for every }a\ge1.
\label{eq:mh-log-g-growth-final-standalone}
\end{equation}
Consequently, for $a\ge1$ (where $g_{\log}(a)>0$),
\begin{equation}
[\phi_{\log}''(a)]_+
=
[g_{\log}(a)+a\,g_{\log}'(a)]_+
\le
(1+C_g)\,g_{\log}(a).
\label{eq:mh-log-curvature-bound-standalone}
\end{equation}
Since $\phi_{\log}'(a)^2=a^2g_{\log}(a)^2$, dividing \eqref{eq:mh-log-curvature-bound-standalone} through gives
\begin{equation}
0\le
\frac{[\phi_{\log}''(a)]_+}{\phi_{\log}'(a)^2}
\le
\frac{(1+C_g)\,g_{\log}(a)}{a^2g_{\log}(a)^2}
=
\frac{1+C_g}{a^2g_{\log}(a)}
\longrightarrow0
\qquad\text{as }a\to\infty,
\label{eq:mh-log-P3-standalone}
\end{equation}
using $g_{\log}(a)\to\infty$ from Step~B. This verifies (P3) of Assumption~\ref{ass:mh-admissible-profile}.

Having verified (P1)–(P3) Assumption~\ref{ass:mh-admissible-profile}, $\phi_{\log}$ satisfies Assumption~\ref{ass:mh-admissible-profile}.

By Theorem~\ref{thm:mh-general-profile-villani} applied with $\phi=\phi_{\log}$, the potential $\widetilde V_{H,\mathrm{ATT}}$ satisfies the Villani condition. Substituting $\phi_{\log}(\lVert\vT_j\rVert)=\lVert\vT_j\rVert^2\log(1+\lVert\vT_j\rVert^2)$ into \eqref{eq:mh-general-potential} and expanding $\lVert\vT_j\rVert^2=\lVert\mW_Q^{(j)}\rVert_F^2+\lVert\mW_K^{(j)}\rVert_F^2$ via \eqref{eq:mh-radii} for each head $j=1,\ldots,H$ gives exactly \eqref{eq:mh-log-recovers-standalone}.
\end{proof}

\begin{corollary}[{\bf Joint-Radius Head-Wise Polynomial Regularization}]
\label{cor:mh-joint-polynomial}
For every $\varepsilon>0$ and $\lambda>0$, let $\phi_\varepsilon(a)=a^{2+\varepsilon}$. Then the potential
\begin{equation}
\widetilde V_{\varepsilon,H,\mathrm{ATT}}(\vT)
\coloneqq
\widehat R_{A,H}(\vT) + \frac{\lambda}{2H} \sum_{j=1}^H \Bigl( \lVert\mW_Q^{(j)}\rVert_F^2 + \lVert\mW_K^{(j)}\rVert_F^2 \Bigr)^{1+\varepsilon/2}
\label{eq:mh-joint-polynomial-potential-cor}
\end{equation}
satisfies the Villani condition as in Definition~\ref{def:villani}.
\end{corollary}

\begin{proof}
    Let
    \begin{equation}
    \phi_\varepsilon(a)=a^{2+\varepsilon}.
    \label{eq:mh-polynomial-profile}
    \end{equation}
    Then
    \begin{equation}
    \phi_\varepsilon'(a)=(2+\varepsilon)a^{1+\varepsilon},
    \qquad
    \phi_\varepsilon''(a)=(2+\varepsilon)(1+\varepsilon)a^\varepsilon.
    \label{eq:mh-polynomial-derivatives}
    \end{equation}
    The function $\phi_\varepsilon$ belongs to $C^2([0,\infty))$ and satisfies
    $\phi_\varepsilon'(0)=0$. Furthermore,
    \begin{equation}
    \frac{\phi_\varepsilon'(a)}{a}
    =(2+\varepsilon)a^\varepsilon
    \longrightarrow+\infty,
    \label{eq:mh-polynomial-P2}
    \end{equation}
    and
    \begin{equation}
    \frac{[\phi_\varepsilon''(a)]_+}{\phi_\varepsilon'(a)^2}
    =
    \frac{1+\varepsilon}{2+\varepsilon}
    \,a^{-2-\varepsilon}
    \longrightarrow0.
    \label{eq:mh-polynomial-P3}
    \end{equation}
    Thus $\phi_\varepsilon$ is admissible. The corresponding regularizer is
    \begin{equation}
    R_{\varepsilon,H}^{\mathrm{rad}}(\vT)
    =
    \frac{\lambda}{2H}
    \sum_{j=1}^H
    \Bigl(
    \lVert\mW_Q^{(j)}\rVert_F^2
    +
    \lVert\mW_K^{(j)}\rVert_F^2
    \Bigr)^{1+\varepsilon/2}.
    \label{eq:mh-radial-polynomial-regularizer}
    \end{equation}
    By Theorem~\ref{thm:mh-general-profile-villani} applied with $\phi=\phi_{\varepsilon}$, the potential $\widetilde V_{H,\mathrm{ATT}}$ satisfies the Villani condition. 
    For this profile, Equation~\eqref{eq:mh-general-potential} becomes
    \begin{equation}
    \widetilde V_{\varepsilon,H,\mathrm{ATT}}(\vT)
    \coloneqq
    \widehat R_{A,H}(\vT) + \frac{\lambda}{2H} \sum_{j=1}^H \Bigl( \lVert\mW_Q^{(j)}\rVert_F^2 + \lVert\mW_K^{(j)}\rVert_F^2 \Bigr)^{1+\varepsilon/2}
    \end{equation}
\end{proof}

\clearpage
\section{Rank-Restricted Mean-Square Loss on Shallow Nets}
\label{app:lora-villani}

Throughout this section, we consider the LoRA loss in Defintion~\ref{def:lora-general-radial-potential}. We recall from Definition~\ref{def:shallow-lora-param}, $\mU$ and $\mV$ denote the trainable factor matrices. Here, we will define $\vT \coloneqq (\mU, \mV) \in \R^D$ where $D = (p+d)r$.The parameter norm is 
\begin{equation} 
    \norm{\vT} = \left( \norm{\mU}_F^2 + \norm{\mV}_F^2 \right)^{1/2}. \label{eq:shallow-lora-radius} 
\end{equation} 

\begin{lemma}[Gradient and Laplacian Bounds for the LoRA Data Term]
\label{lem:lora-data-bounds}
Suppose that, for every $i=1,\ldots,n$,
\[
\lVert\vx_i\rVert\le B_x,
\qquad
|y_i|\le B_y,
\]
and suppose that $\lVert\va\rVert\le B_a$, $\lVert\mW_0\rVert_F\le B_0$, and that the
activation $\sigma$ satisfies Assumption~\ref{as1}, i.e.\ $|\sigma(u)|\le B_\sigma$,
$|\sigma'(u)|\le L_\sigma$ and $|\sigma''(u)|\le L_\sigma'$ for all $u\in\R$. Set
\begin{equation}
\kappa\coloneqq\frac{|\alpha|B_x}{r},
\qquad
E\coloneqq B_y+B_aB_\sigma\sqrt p .
\label{eq:lora-data-constants}
\end{equation}
Then, for the empirical data term $\Ldata$ in \eqref{eq:lora-general-data-loss} and every
$\vT=(\mU,\mV)\in\R^D$,
\begin{equation}
0\le\Ldata(\vT)\le\tfrac12E^2,
\label{eq:lora-data-value-bound}
\end{equation}
\begin{equation}
\left\lVert\nabla_{\vT}\Ldata(\vT)\right\rVert
\le
C_{\nabla,\mathrm{LoRA}}\lVert\vT\rVert,
\qquad
C_{\nabla,\mathrm{LoRA}}\coloneqq\kappa B_aL_\sigma E,
\label{eq:lora-data-gradient-bound}
\end{equation}
and
\begin{equation}
\left|\Delta_{\vT}\Ldata(\vT)\right|
\le
C_{\Delta,\mathrm{LoRA}}\lVert\vT\rVert^2,
\qquad
C_{\Delta,\mathrm{LoRA}}\coloneqq
\kappa^2\bigl(B_a^2L_\sigma^2+E L_\sigma'B_a\sqrt p\bigr),
\label{eq:lora-data-laplacian-bound}
\end{equation}
where $\kappa$ and $E$ depend only on $B_x,B_y,B_a,B_\sigma,\alpha,r$ and the width $p$, and are
independent of $\vT$. In particular
$\lVert\nabla_{\vT}\Ldata\rVert\le C_{\nabla,\mathrm{LoRA}}(1+\lVert\vT\rVert)\lVert\vT\rVert$
and $|\Delta_{\vT}\Ldata|\le C_{\Delta,\mathrm{LoRA}}(1+\lVert\vT\rVert^2)$.
\end{lemma}

The above lemma is proved in Appendix~\ref{app:subsec:loralemm}.

\begin{proof}[Proof of Theorem~\ref{thm:lora-general-profile-villani}]
We verify the four requirements of Definition~\ref{def:villani} as follows. In Step~1, we
establish the first requirement. In Step~2, we verify the second and third requirements.
Finally, in Steps~3 to 6, we verify the fourth requirement.

\paragraph{\hspace{1.00em}\bf (Step 1: $C^2$-Regularity.)}
We recall from \eqref{eq:lora-general-radial-potential},
\[
\widetilde V_{\mathrm{LoRA}}(\vT)
=
\Ldata(\vT)
+
R_{\mathrm{LoRA}}(\vT),
\qquad
R_{\mathrm{LoRA}}(\vT) = \frac{\lambda}{2}\phi\!\left(\lVert\vT\rVert\right).
\]
The data term is twice continuously differentiable: for each $i$, the map
\begin{equation}
\vT=(\mU,\mV)
\longmapsto
\mW(\vT)\vx_i
=
\mW_0\vx_i+\frac{\alpha}{r}\mU\mV^\top\vx_i
\label{eq:lora-affine-preactivation}
\end{equation}
is a polynomial (in fact bilinear) map of the entries of $\vT$, plus the fixed shift
$\mW_0\vx_i$; since $\sigma\in C^2(\R)$ with $\sigma''$ bounded under Assumption~\ref{as1},
and since compositions, sums and inner products preserve $C^2$-regularity,
\begin{equation}
\Ldata\in C^2(\R^D).
\label{eq:lora-data-smooth}
\end{equation}
(If $\sigma$ is in addition $C^\infty$, as for the sigmoid, $\tanh$, $\arctan$, $\operatorname{erf}$
and the trigonometric activations of Table~\ref{tab:actlora}, then $\Ldata\in C^\infty(\R^D)$;
only $C^2$ is required here.)

For the regularizer, $\phi\in C^2([0,\infty))$ by (P1) of
Assumption~\ref{ass:lora-admissible-profile}, but $\lVert\vT\rVert$ is not differentiable at
the origin. Since there is now only a single radial block (the entire parameter vector
$\vT\in\R^D$, with $D=(p+d)r$, rather than $H$ separate head blocks), the argument of Step~1 in
the proof of Theorem~\ref{thm:mh-general-profile-villani} applies verbatim with $D_0$ replaced
by $D$ and $\vT_j$ replaced by $\vT$. That is, for $\vT\neq\vzero$,
\begin{equation}
\nabla_{\vT}\phi\!\left(\lVert\vT\rVert\right)
=
\frac{\phi'\!\left(\lVert\vT\rVert\right)}{\lVert\vT\rVert}\vT,
\qquad
\nabla_{\vT}^2\phi\!\left(\lVert\vT\rVert\right)
=
\frac{\phi'\!\left(\lVert\vT\rVert\right)}{\lVert\vT\rVert}\mI_D
+
\left(
\phi''\!\left(\lVert\vT\rVert\right)
-
\frac{\phi'\!\left(\lVert\vT\rVert\right)}{\lVert\vT\rVert}
\right)
\frac{\vT\vT^\top}{\lVert\vT\rVert^2},
\label{eq:lora-radial-hessian}
\end{equation}
and, exactly as in \eqref{eq:mh-phi-prime-origin-limit}--\eqref{eq:mh-radial-C2}, using
$\phi'(0)=0$ and continuity of $\phi''$ at the origin,
\begin{equation}
\phi\!\left(\lVert\vT\rVert\right)\in C^2(\R^D),
\qquad
\left.\nabla_{\vT}\phi\!\left(\lVert\vT\rVert\right)\right|_{\vT=\vzero}=\vzero,
\qquad
\left.\nabla_{\vT}^2\phi\!\left(\lVert\vT\rVert\right)\right|_{\vT=\vzero}=\phi''(0)\mI_D.
\label{eq:lora-radial-C2}
\end{equation}
Hence $R_{\mathrm{LoRA}}\in C^2(\R^D)$, and combined with \eqref{eq:lora-data-smooth},
\begin{equation}
\widetilde V_{\mathrm{LoRA}}\in C^2(\R^D).
\label{eq:lora-total-C2}
\end{equation}

\paragraph{\hspace{1.00em}\bf (Step 2: Confinement and Integrability.)}
Fix $M>0$. By (P2) of Assumption~\ref{ass:lora-admissible-profile}, there exists $R_M>0$ such
that
\begin{equation}
\phi'(a)\ge Ma
\qquad
\text{for every }a\ge R_M,
\label{eq:lora-slope-lower-bound}
\end{equation}
and integrating gives, exactly as in \eqref{eq:mh-global-phi-lower-bound},
\begin{equation}
\phi(a)\ge \frac{M}{2}a^2-C_M
\qquad
\text{for every }a\ge0,
\qquad
C_M\coloneqq\max_{0\le a\le R_M}\left(\frac M2 a^2-\phi(a)\right)<\infty.
\label{eq:lora-global-phi-lower-bound}
\end{equation}
Since $\Ldata\ge0$, Equations~\eqref{eq:lora-general-radial-potential} and
\eqref{eq:lora-global-phi-lower-bound} imply
\begin{equation}
\widetilde V_{\mathrm{LoRA}}(\vT)
\ge
\frac{\lambda}{2}\phi\!\left(\lVert\vT\rVert\right)
\ge
\frac{\lambda M}{4}\lVert\vT\rVert^2
-
\frac{\lambda}{2}C_M.
\label{eq:lora-confining-lower-bound}
\end{equation}
Hence
\begin{equation}
\lim_{\lVert\vT\rVert\to\infty}\widetilde V_{\mathrm{LoRA}}(\vT)=+\infty,
\label{eq:lora-confinement}
\end{equation}
and, for every $s>0$, \eqref{eq:lora-confining-lower-bound} yields a Gaussian tail bound so
that
\begin{equation}
\int_{\R^D}\exp\!\left(-\frac2s\widetilde V_{\mathrm{LoRA}}(\vT)\right)d\vT<\infty.
\label{eq:lora-integrability}
\end{equation}
We note in passing that, by \eqref{eq:lora-data-value-bound} of
Lemma~\ref{lem:lora-data-bounds}, the data term is also bounded \emph{above} by $\frac12E^2$,
so $\widetilde V_{\mathrm{LoRA}}$ and $\frac\lambda2\phi(\lVert\vT\rVert)$ differ by a bounded
quantity; only the lower bound is used here.

\paragraph{\hspace{1.00em}\bf (Step 3: Gradient and Laplacian of the Regularizer.)}
By \eqref{eq:lora-radial-hessian} and \eqref{eq:lora-radial-C2}, for every $\vT\in\R^D$,
\begin{equation}
\left\lVert\nabla_{\vT}R_{\mathrm{LoRA}}(\vT)\right\rVert^2
=
\frac{\lambda^2}{4}\phi'\!\left(\lVert\vT\rVert\right)^2
\eqqcolon
\frac{\lambda^2}{4}G(\vT).
\label{eq:lora-G-definition}
\end{equation}
Taking the trace of \eqref{eq:lora-radial-hessian} as in
\eqref{eq:mh-block-laplacian-nonzero}--\eqref{eq:mh-block-laplacian} (with $D_0$ replaced by
$D$) gives, for every $\vT\in\R^D$,
\begin{equation}
\Delta_{\vT}\phi\!\left(\lVert\vT\rVert\right)
=
\mathcal L_\phi\!\left(\lVert\vT\rVert\right)
\coloneqq
\phi''\!\left(\lVert\vT\rVert\right)
+
(D-1)\frac{\phi'\!\left(\lVert\vT\rVert\right)}{\lVert\vT\rVert},
\label{eq:lora-radial-laplacian}
\end{equation}
with the convention $\mathcal L_\phi(0)=D\phi''(0)$. Set
\begin{equation}
\kappa_\phi(a)\coloneqq[\mathcal L_\phi(a)]_+,
\qquad
K(\vT)\coloneqq\kappa_\phi\!\left(\lVert\vT\rVert\right).
\label{eq:lora-K-definition}
\end{equation}
Then
\begin{equation}
\Delta_{\vT}R_{\mathrm{LoRA}}(\vT)
\le
\frac{\lambda}{2}K(\vT).
\label{eq:lora-regularizer-laplacian-upper}
\end{equation}

\paragraph{\hspace{1.00em}\bf (Step 4: Asymptotic Domination.)}
We show
\begin{equation}
\lim_{\lVert\vT\rVert\to\infty}\frac{G(\vT)}{\lVert\vT\rVert^2}=+\infty,
\qquad
\lim_{\lVert\vT\rVert\to\infty}\frac{K(\vT)}{G(\vT)}=0.
\label{eq:lora-domination-goals}
\end{equation}
Since $G(\vT)=\phi'(\lVert\vT\rVert)^2$ is a function of $a=\lVert\vT\rVert$ alone, both limits
reduce to single-variable statements. By (P2) of
Assumption~\ref{ass:lora-admissible-profile},
\begin{equation}
\frac{G(\vT)}{\lVert\vT\rVert^2}
=
\left(\frac{\phi'(a)}{a}\right)^2
\longrightarrow+\infty
\qquad
\text{as }a=\lVert\vT\rVert\to\infty,
\label{eq:lora-G-dominates-rho}
\end{equation}
which gives the first limit and, in particular, $G(\vT)\to\infty$ as $\lVert\vT\rVert\to\infty$.
Equivalently, and in the form used in Step~6,
\begin{equation}
\frac{1+\lVert\vT\rVert^2}{G(\vT)}\longrightarrow0
\qquad
\text{as }\lVert\vT\rVert\to\infty.
\label{eq:lora-quadratic-over-G}
\end{equation}
For the second limit, by (P2), $\phi'(a)>0$ for all sufficiently large $a$, so
\begin{equation}
\frac{\kappa_\phi(a)}{\phi'(a)^2}
\le
\frac{[\phi''(a)]_+}{\phi'(a)^2}
+
(D-1)\frac{1}{a\phi'(a)},
\label{eq:lora-kappa-pointwise-bound}
\end{equation}
where the first term tends to zero by (P3) and the second tends to zero because
\begin{equation}
\frac{1}{a\phi'(a)}
=
\frac{1}{a^2}\cdot\frac{1}{\phi'(a)/a}
\longrightarrow0
\label{eq:lora-angular-limit}
\end{equation}
by (P2). Hence $\kappa_\phi(a)/\phi'(a)^2\to0$, i.e. $K(\vT)/G(\vT)\to0$ as
$\lVert\vT\rVert\to\infty$, proving the second limit in \eqref{eq:lora-domination-goals}.

\paragraph{\hspace{1.00em}\bf (Step 5: Data-Term Bounds.)}
Because $\sigma$ is bounded under Assumption~\ref{as1}, Lemma~\ref{lem:lora-data-bounds}
supplies the \emph{quadratic} bounds
\begin{equation}
\left\lVert\nabla_{\vT}\Ldata(\vT)\right\rVert
\le
C_{\nabla,\mathrm{LoRA}}\lVert\vT\rVert,
\qquad
\left|\Delta_{\vT}\Ldata(\vT)\right|
\le
C_{\Delta,\mathrm{LoRA}}\lVert\vT\rVert^2,
\label{eq:lora-data-bounds-prf}
\end{equation}
and in particular
\begin{equation}
\left\lVert\nabla_{\vT}\Ldata(\vT)\right\rVert^2
\le
C_{\nabla,\mathrm{LoRA}}^2\lVert\vT\rVert^2 .
\label{eq:lora-data-gradient-sq-bound}
\end{equation}
This is the decisive point of the argument: the data term contributes only $O(\lVert\vT\rVert^2)$
to both the gradient-squared and the Laplacian, matching the $\lVert\vT\rVert^2$ scale against
which (P2) already guarantees domination in \eqref{eq:lora-quadratic-over-G}.

\paragraph{\hspace{1.00em}\bf (Step 6: Villani Limit.)}
Fix $s>0$. For arbitrary vectors $\va,\vb$ we have
$\lVert\va+\vb\rVert^2\ge\frac12\lVert\vb\rVert^2-\lVert\va\rVert^2$. Applying this with
$\va=\nabla_{\vT}\Ldata$ and $\vb=\nabla_{\vT}R_{\mathrm{LoRA}}$, and using
\eqref{eq:lora-G-definition} and \eqref{eq:lora-data-gradient-sq-bound},
\begin{equation}
\left\lVert\nabla_{\vT}\widetilde V_{\mathrm{LoRA}}(\vT)\right\rVert^2
\ge
\frac{\lambda^2}{8}G(\vT)
-
C_{\nabla,\mathrm{LoRA}}^2\lVert\vT\rVert^2 .
\label{eq:lora-full-gradient-lower}
\end{equation}
Similarly, by \eqref{eq:lora-regularizer-laplacian-upper} and
\eqref{eq:lora-data-bounds-prf},
\begin{equation}
\Delta_{\vT}\widetilde V_{\mathrm{LoRA}}(\vT)
\le
C_{\Delta,\mathrm{LoRA}}\lVert\vT\rVert^2
+
\frac{\lambda}{2}K(\vT).
\label{eq:lora-full-laplacian-upper}
\end{equation}
Combining \eqref{eq:lora-full-gradient-lower} and \eqref{eq:lora-full-laplacian-upper} and
factoring out $G(\vT)$, which is strictly positive for all $\lVert\vT\rVert$ large by (P2),
\begin{align}
\frac1s\left\lVert\nabla_{\vT}\widetilde V_{\mathrm{LoRA}}(\vT)\right\rVert^2
-
\Delta_{\vT}\widetilde V_{\mathrm{LoRA}}(\vT)
&\ge
\frac{\lambda^2}{8s}G(\vT)
-
\left(\frac{C_{\nabla,\mathrm{LoRA}}^2}{s}+C_{\Delta,\mathrm{LoRA}}\right)\lVert\vT\rVert^2
-
\frac{\lambda}{2}K(\vT)
\notag\\
&=
G(\vT)
\left[
\frac{\lambda^2}{8s}
-
\left(\frac{C_{\nabla,\mathrm{LoRA}}^2}{s}+C_{\Delta,\mathrm{LoRA}}\right)
\frac{\lVert\vT\rVert^2}{G(\vT)}
-
\frac{\lambda}{2}\frac{K(\vT)}{G(\vT)}
\right].
\label{eq:lora-villani-factorization}
\end{align}
By \eqref{eq:lora-quadratic-over-G} the second term in the bracket tends to $0$, and by the
second limit in \eqref{eq:lora-domination-goals} so does the third. Hence the bracket converges
to
\begin{equation}
\frac{\lambda^2}{8s}>0,
\label{eq:lora-positive-bracket-limit}
\end{equation}
and is therefore bounded below by $\lambda^2/(16s)$ for all $\lVert\vT\rVert\ge\rho_0$, for some
$\rho_0<\infty$. Since $G(\vT)\to+\infty$ as $\lVert\vT\rVert\to\infty$ by Step~4, we conclude
\begin{equation}
\frac1s\left\lVert\nabla_{\vT}\widetilde V_{\mathrm{LoRA}}(\vT)\right\rVert^2
-
\Delta_{\vT}\widetilde V_{\mathrm{LoRA}}(\vT)
\ge
\frac{\lambda^2}{16s}G(\vT)
\xrightarrow[\ \lVert\vT\rVert\to\infty\ ]{}
+\infty .
\label{eq:lora-villani-limit}
\end{equation}
Since $s>0$ was arbitrary, all four requirements in Definition~\ref{def:villani} hold.
\end{proof}

\subsection{Proof of Intermediate Lemmas Needed for Theorem~\ref{thm:lora-general-profile-villani}}\label{app:subsec:loralemm}

\begin{proof}[Proof of Lemma~\ref{lem:lora-data-bounds}]
Write $\mW(\vT)=\mW_0+\frac\alpha r\mU\mV^\top$ and $\vz_i(\vT)\coloneqq\mW(\vT)\vx_i\in\R^p$.
It is convenient to set
\begin{equation}
\vv_i\coloneqq\frac\alpha r\mV^\top\vx_i\in\R^r,
\qquad\text{so that}\qquad
\vz_i(\vT)=\mW_0\vx_i+\mU\vv_i,
\qquad
\lVert\vv_i\rVert\le\kappa\lVert\mV\rVert_F .
\label{eq:lora-preactivation-factored}
\end{equation}

\emph{Value and residual.} Since $|\sigma(u)|\le B_\sigma$ pointwise,
$\lVert\sigma(\vz_i(\vT))\rVert\le B_\sigma\sqrt p$ for every $\vT$, whence by Cauchy--Schwarz
\begin{equation}
|\widehat y_i(\vT)|
=
\bigl|\va^\top\sigma(\vz_i(\vT))\bigr|
\le
B_aB_\sigma\sqrt p,
\qquad
|e_i(\vT)|\coloneqq|\widehat y_i(\vT)-y_i|\le E,
\label{eq:lora-residual-bound}
\end{equation}
uniformly in $\vT$; this gives \eqref{eq:lora-data-value-bound}. Note that, in contrast with the
linear-growth case, $E$ does \emph{not} depend on $\vT$, $B_0$ or $\kappa$.

\emph{Gradient.} Differentiating $\vz_i(\vT)=\mW_0\vx_i+\frac\alpha r\mU\mV^\top\vx_i$ in the
directions $\dd\mU,\dd\mV$,
\begin{equation}
\dd\vz_i
=
\frac{\alpha}{r}\bigl(\dd\mU\,\mV^\top\vx_i+\mU\,(\dd\mV)^\top\vx_i\bigr),
\label{eq:lora-preactivation-differential}
\end{equation}
so, using $\lVert\mV^\top\vx_i\rVert\le B_x\lVert\mV\rVert_F$ and
$\lVert\mU(\dd\mV)^\top\vx_i\rVert\le B_x\lVert\mU\rVert_F\lVert\dd\mV\rVert_F$,
\begin{equation}
\lVert\dd\vz_i\rVert
\le
\kappa\bigl(\lVert\mV\rVert_F\lVert\dd\mU\rVert_F+\lVert\mU\rVert_F\lVert\dd\mV\rVert_F\bigr)
\le
\kappa\lVert\vT\rVert\,\lVert(\dd\mU,\dd\mV)\rVert,
\label{eq:lora-preactivation-differential-bound}
\end{equation}
the last step by Cauchy--Schwarz together with
$\lVert\vT\rVert^2=\lVert\mU\rVert_F^2+\lVert\mV\rVert_F^2$. Setting
$\vg_i\coloneqq\operatorname{diag}(\sigma'(\vz_i))\va$, we have $\lVert\vg_i\rVert\le L_\sigma B_a$
and $\dd\widehat y_i=\vg_i^\top\dd\vz_i$, so $\dd\ell_i=-e_i\,\dd\widehat y_i$ gives
\begin{equation}
|\dd\ell_i|
\le
E\,B_aL_\sigma\,\kappa\lVert\vT\rVert\,\lVert(\dd\mU,\dd\mV)\rVert .
\label{eq:lora-loss-differential-bound}
\end{equation}
Hence $\lVert\nabla_{\vT}\ell_i(\vT)\rVert\le\kappa B_aL_\sigma E\lVert\vT\rVert$ for every $i$,
and averaging over $i$ with the triangle inequality yields
\eqref{eq:lora-data-gradient-bound}.

\emph{Laplacian.} For a scalar coordinate $\theta$ of $\mU$ or $\mV$,
\begin{equation}
\frac{\partial^2\ell_i}{\partial\theta^2}
=
\left(\frac{\partial\widehat y_i}{\partial\theta}\right)^2
-
e_i\,\frac{\partial^2\widehat y_i}{\partial\theta^2},
\label{eq:lora-second-derivative-decomposition}
\end{equation}
so $\Delta_{\vT}\ell_i=\lVert\nabla_{\vT}\widehat y_i\rVert^2-e_i\Delta_{\vT}\widehat y_i$, and we
bound the two terms separately.

For the first, \eqref{eq:lora-preactivation-factored} gives
$\partial\vz_i/\partial U_{ab}=(\vv_i)_b\,\ve_a$ and
$\partial\vz_i/\partial V_{cd}=\frac\alpha r(\vx_i)_c\,\mU_{:,d}$, whence
$\partial\widehat y_i/\partial U_{ab}=(\vg_i)_a(\vv_i)_b$ and
$\partial\widehat y_i/\partial V_{cd}=\frac\alpha r(\vx_i)_c\,(\mU^\top\vg_i)_d$. Summing squares
over all $D=(p+d)r$ coordinates,
\begin{equation}
\lVert\nabla_{\vT}\widehat y_i\rVert^2
=
\lVert\vg_i\rVert^2\lVert\vv_i\rVert^2
+
\frac{\alpha^2}{r^2}\lVert\vx_i\rVert^2\lVert\mU^\top\vg_i\rVert^2
\le
\kappa^2\lVert\vg_i\rVert^2\bigl(\lVert\mV\rVert_F^2+\lVert\mU\rVert_F^2\bigr)
\le
\kappa^2B_a^2L_\sigma^2\lVert\vT\rVert^2 .
\label{eq:lora-output-gradient-sum}
\end{equation}

For the second, $\vz_i$ is \emph{bilinear} in $(\mU,\mV)$, so
$\partial^2\vz_i/\partial\theta^2=0$ for every single coordinate $\theta$ (the nonzero second
partials of $\vz_i$ are the mixed ones $\partial^2\vz_i/\partial U_{ab}\partial V_{cd}$, which do
not enter the Laplacian). The chain rule therefore leaves only the $\sigma''$ term,
\begin{equation}
\Delta_{\vT}\widehat y_i
=
\sum_{\theta}\sum_{j=1}^p a_j\,\sigma''\bigl((\vz_i)_j\bigr)
\left(\frac{\partial(\vz_i)_j}{\partial\theta}\right)^{2},
\label{eq:lora-output-laplacian}
\end{equation}
and, coordinate by coordinate,
$\sum_\theta(\partial(\vz_i)_j/\partial\theta)^2
=\lVert\vv_i\rVert^2+\frac{\alpha^2}{r^2}\lVert\vx_i\rVert^2\lVert\mU_{j,:}\rVert^2
\le\kappa^2\bigl(\lVert\mV\rVert_F^2+\lVert\mU_{j,:}\rVert^2\bigr)$.
Using $|\sigma''|\le L_\sigma'$, $\lVert\va\rVert_1\le\sqrt p\,B_a$ and
$\lVert\va\rVert_\infty\le B_a$,
\begin{equation}
\bigl|\Delta_{\vT}\widehat y_i\bigr|
\le
L_\sigma'\kappa^2
\Bigl(\lVert\va\rVert_1\lVert\mV\rVert_F^2+\lVert\va\rVert_\infty\lVert\mU\rVert_F^2\Bigr)
\le
L_\sigma'\kappa^2B_a\sqrt p\,\lVert\vT\rVert^2 .
\label{eq:lora-output-laplacian-bound}
\end{equation}

Combining \eqref{eq:lora-output-gradient-sum}, \eqref{eq:lora-output-laplacian-bound} and
$|e_i|\le E$ from \eqref{eq:lora-residual-bound},
\begin{equation}
\bigl|\Delta_{\vT}\ell_i(\vT)\bigr|
\le
\kappa^2\bigl(B_a^2L_\sigma^2+EL_\sigma'B_a\sqrt p\bigr)\lVert\vT\rVert^2 ,
\label{eq:lora-sample-laplacian-bound}
\end{equation}
and averaging over $i$ proves \eqref{eq:lora-data-laplacian-bound} with
$C_{\Delta,\mathrm{LoRA}}$ as stated.
\end{proof}

\begin{remark}
    As in Section~\ref{subsec:mh-generated-profiles}, the LoRA loss as in Definition~\ref{def:lora-general-radial-potential} satisfies Assumption~\ref{ass:mh-admissible-profile} for $\phi_{\log}(a) = a^2\log(1+a^2)$ and $\phi_\varepsilon(a)=a^{2+\varepsilon}$ for any $\varepsilon > 0$ as shown in Corollaries~\ref{cor:mh-log-profile-standalone} and \ref{cor:mh-joint-polynomial}.
\end{remark}

\subsection{Activations satisfying Assumption~\ref{as1}}\label{app:subsec:actlora}
~\\
\begin{table}[h]
\begin{center}
\begin{tabular}{lccc}
\toprule
$\sigma(u)$ & $B_\sigma=\sup|\sigma|$ & $L_\sigma=\sup|\sigma'|$ & $L_\sigma'=\sup|\sigma''|$ \\
\midrule
Sigmoid $\dfrac{1}{1+e^{-u}}$ & $1$ & $\tfrac14$ & $\dfrac{1}{6\sqrt3}\approx0.0962$ \\[1.6ex]
Tanh $\dfrac{e^u-e^{-u}}{e^u+e^{-u}}$ & $1$ & $1$ & $\dfrac{4}{3\sqrt3}\approx0.7698$ \\[1.6ex]
$\arctan(u)$ & $\dfrac{\pi}{2}\approx1.5708$ & $1$ & $\dfrac{3\sqrt3}{8}\approx0.6495$ \\[1.6ex]
$\sin(u)$, $\cos(u)$ & $1$ & $1$ & $1$ \\
\bottomrule
\end{tabular}
\end{center}
\caption{Activation functions satisfying Assumption~\ref{as1}, with explicit constants.}
\label{tab:actlora}
\end{table}

\section{A Schematic of the Unregularized Multi-Headed Loss}\label{sec:diag}

For readability of the diagrams we drop the bracketing of the head-index superscripts of the various matrices, that is used in the main text.  

\providecommand{\alert}[1]{\textcolor{red}{#1}}
\newcommand{\Wv}{\bm{W}_v}
\newcommand{\WQ}[1]{\bm{W}_Q^{#1}}
\newcommand{\WK}[1]{\bm{W}_K^{#1}}
\newcommand{\Th}{\bm{T}_H}
\newcommand{\Yh}{\widehat{\bm{Y}}}

% ---- glossy diagram styling ----
% Palette chosen for GREYSCALE separation. Approx. luminance
% (0.299R+0.587G+0.114B) is spread across the range so the figures survive
% black-and-white printing:
%   cOut 207 (lightest) > cQK 178 > cOp 152 > cIn 124 > cAt 103 > cTgt 95 > cTr 79
% Redundant, non-colour cues are also used:
%   trainable = THICK border,  ground truth = DASHED border.
\definecolor{cIn} {RGB}{122,126,134}   % input / data        L~124
\definecolor{cTr} {RGB}{ 35, 84,165}   % trainable           L~ 79
\definecolor{cQK} {RGB}{120,200,215}   % projected Q, K      L~178
\definecolor{cOp} {RGB}{215,140, 50}   % arithmetic ops      L~152
\definecolor{cAt} {RGB}{128, 78,168}   % attention / softmax L~103
\definecolor{cOut}{RGB}{245,210, 95}   % model output        L~207
\definecolor{cTgt}{RGB}{ 40,125, 85}   % ground truth        L~ 95
\definecolor{cLo} {RGB}{190, 55, 55}   % loss                L~101
 
\tikzset
{
  glossy/.style={
    draw=#1!78!black, line width=0.5pt, rounded corners=3pt,
    top color=white, middle color=#1!16, bottom color=#1!45,
    blur shadow={shadow blur steps=5, shadow xshift=0.5pt,
                 shadow yshift=-0.9pt, shadow opacity=45},
    font=\sffamily\sansmath\scriptsize, align=center, inner sep=3.5pt,
    minimum height=8.5mm, minimum width=13mm},
  gin/.style   ={glossy=cIn},
  gtrain/.style={glossy=cTr, line width=1.2pt},          % thick = trained
  gqk/.style   ={glossy=cQK},
  gop/.style   ={glossy=cOp},
  gsm/.style   ={glossy=cAt},
  gout/.style  ={glossy=cOut},
  gtar/.style  ={glossy=cTgt, line width=0.9pt,
  dash pattern=on 1.7pt off 1.3pt},        % dashed = ground truth
  gloss/.style ={glossy=cLo, circle, minimum size=8mm, minimum width=8mm},
  gcirc/.style ={glossy=cOp, circle, minimum size=8.5mm, minimum width=8.5mm},
  flow/.style  ={-{Stealth[length=2.4mm,width=1.8mm]}, draw=gray!55!black, line width=0.7pt},
  dlbl/.style  ={font=\sffamily\small, text=gray!55!black}
  }

\begin{figure}[ht]
  \centering
 \scalebox{0.86}{%
  \begin{tikzpicture}
    % --- nodes -----------------------------------------------------------
    \node[gin,minimum width=18mm,minimum height=18mm]    (X)  at (0,0)        {$\bm{X}$\\[-1pt]$t\times d$};
    \node[gtrain] (wq) at (2.15, 1.45) {$\bm{W}_Q^{j}$\\[-1pt]\tiny$d\times r$};
    \node[gtrain] (wk) at (2.15,-1.45) {$\bm{W}_K^{j}$\\[-1pt]\tiny$d\times r$};
    \node[gqk]    (Q)  at (4.25, 1.45) {$\bm{Q}^{j}$\\[-1pt]\tiny$t\times r$};
    \node[gqk]    (K)  at (4.25,-1.45) {$\bm{K}^{j}$\\[-1pt]\tiny$t\times r$};
 
    \node[gop, minimum width=17mm] (Z)  at (6.75,0)
        {$\bm{M}^{j}=\dfrac{\bm{Q}^{j}(\bm{K}^{j})^{\!\top}}{\sqrt d}$};
    \node[gop, minimum width=16mm] (SM) at (9.5,0) {$\mathrm{RowSoftMax}_\beta$};
    \node[gsm]                     (S)  at (11.75,0) {$\bm{S}^{j}$\\[-1pt]\tiny$t\times t$};
 
    \node[gtrain] (V)   at (11.75,-2.75) {$\bm{X}\bm{W}_v$\\[-1pt]\tiny$t\times d$};
    \node[gcirc]  (mul) at (13.6,0) {$\times$};
    \node[gout, minimum width=18mm,minimum height=18mm] (out) at (15.9,0)
        {$\widehat{\bm{Y}}^{j}(\bm{X})$\\[-1pt]$t\times d$};
 
    % --- flow ------------------------------------------------------------
    \draw[flow] (X.north) |- (wq.west);
    \draw[flow] (X.south) |- (wk.west);
    \draw[flow] (wq) -- (Q);
    \draw[flow] (wk) -- (K);
    \draw[flow] (Q.east) -- (Z.north west);
    \draw[flow] (K.east) -- (Z.south west);
    \draw[flow] (Z) -- (SM);
    \draw[flow] (SM) -- (S);
    \draw[flow] (S) -- (mul);
    \draw[flow] (V.east) -| (mul.south);
    \draw[flow] (mul) -- (out);
    \draw[flow] (X.south) |- (0,-2.75) -- (V.west);
 
    % --- annotations -----------------------------------------------------
    \node[dlbl, below=0.8mm of Z]  {scores / logits};
    \node[dlbl, above=0.8mm of S]  {row-stochastic};
    \node[dlbl, below=0.8mm of V]  {$\bm{W}_v$ --- (not trained)};
  \end{tikzpicture}}
\caption{Single Key--Query Attention Head.}
  \label{fig:head}
\end{figure}

\begin{figure}[ht]
  \centering 
\scalebox{0.85}{%
  \begin{tikzpicture}
    % --- heads -----------------------------------------------------------
    \node[gin,minimum width=18mm,minimum height=18mm] (X) at (0,0) {$\bm{X}$\\[-1pt]$t\times d$};
 
    \node[gsm, minimum width=26mm] (h1) at (3.5, 2.05) {Head $1$:\ \ $\bm{S}^{1}(\bm{X})$};
    \node[gsm, minimum width=26mm] (h2) at (3.5, 0.75) {Head $2$:\ \ $\bm{S}^{2}(\bm{X})$};
    \node[font=\sffamily\small, text=gray!60!black] at (3.5,-0.35) {$\vdots$};
    \node[gsm, minimum width=26mm] (hH) at (3.5,-1.55) {Head $H$:\ \ $\bm{S}^{H}(\bm{X})$};
 
    \node[gcirc] (sum)  at (6.4,0.25) {$+$};
    \node[gsm, minimum width=15mm] (Ssum) at (8.7,0.25)
        {$\displaystyle\sum_{j=1}^{H}\bm{S}^{j}$\\[-1pt]\tiny$t\times t$};
 
    \node[gtrain] (V)  at (8.7,-2.5) {$\bm{X}\bm{W}_v$\\[-1pt]\tiny$t\times d$};
    \node[gcirc]  (mm) at (10.9,0.25) {$\times$};
    \node[gout, minimum width=18mm,minimum height=18mm] (Yh) at (13.2,0.25)
        {$\hat{\bm{Y}}(\bm{X})$\\[-1pt]$t\times d$};
    \node[gtar,minimum width=18mm,minimum height=18mm] (Y)    at (13.2, 2.35) {$\bm{Y}$\\[-1pt]$t\times d$};
    \node[gloss] (loss) at (16,1.3) {$\ell$};
 
    % --- flow ------------------------------------------------------------
    \draw[flow] (X.east) -- (h1.west);
    \draw[flow] (X.east) -- (h2.west);
    \draw[flow] (X.east) -- (hH.west);
    \draw[flow] (h1) -| (sum); \draw[flow] (h2) -- (sum); \draw[flow] (hH) -| (sum);
    \draw[flow] (sum) -- (Ssum);
    \draw[flow] (Ssum) -- (mm);
    \draw[flow] (V.east) -| (mm.south);
    \draw[flow] (mm) -- (Yh);
    \draw[flow] (X.south) |- (0,-2.5) -- (V.west);
    \draw[flow] (Yh) -- (loss);
    \draw[flow] (Y)  -- (loss);
 
    \node[dlbl, below=1.2mm of loss, align=center]
        {$\tfrac12\lVert\bm{Y}-\hat{\bm{Y}}(\bm{X})\rVert_F^{2}$};
    \node[dlbl, below=0.8mm of V] {$\bm{W}_v$ --- (not trained), shared by all heads};
  \end{tikzpicture}}
\caption{Multi-Head Key--Query Attention Loss.}
  \label{fig:mhead}
\end{figure}

% \vspace{0.4em}
% The $H$ heads share the input and the value stream; they differ \alert{only} in how they weight tokens.

\end{document}